\PassOptionsToPackage{table}{xcolor}
\documentclass{valence} 

\usepackage{ulem}
\usepackage{url}
\usepackage{enumitem}
\usepackage[T1]{fontenc}
\usepackage{wrapfig}
\usepackage{xcolor}
\usepackage{booktabs}
\usepackage{amsthm}
\usepackage{tcolorbox}
\tcbuselibrary{most}
\usepackage[ruled,vlined]{algorithm2e}
\usepackage{thmtools}
\usepackage{thm-restate}
\usepackage{cleveref}  

\crefname{figure}{Fig.}{Figs.}
\crefname{table}{Tab.}{Tabs.}
\crefname{section}{Sec.}{Secs.}
\crefname{equation}{Eq.}{Eqs.}
\crefname{algocf}{Alg.}{Algs.}

\tcbset{
  thmbox/.style={
    enhanced,
    breakable,
    colback=cyan!10!white,
    colframe=cyan!10!white,
    boxrule=0pt,
    arc=2pt,
    left=8pt, right=8pt, top=6pt, bottom=6pt,
    fonttitle=\bfseries,
  }
}

\newcounter{boxmaster}[section]
\renewcommand{\theboxmaster}{\thesection.\arabic{boxmaster}}

\renewcommand{\title}[1]{\newcommand{\titlelist}{{\LARGE\sffamily\bfseries #1}}}

\newcommand{\defbox}[5]{%
  \ifcsname c@#1Inner\endcsname\else\newcounter{#1Inner}\fi
  \expandafter\def\csname the#1Inner\endcsname{\theboxmaster}%
  \expandafter\let\csname #1Inner\endcsname\undefined
  \expandafter\let\csname end#1Inner\endcsname\undefined
  \newenvironment{#1Inner}[1][]{%
    \stepcounter{boxmaster}\refstepcounter{#1Inner}%
    {\bfseries #2~\theboxmaster}%
    \ifx\relax##1\relax\else\ (##1)\fi
    {\bfseries.}\ #3%
  }{}%
  \expandafter\let\csname #1\endcsname\undefined
  \expandafter\let\csname end#1\endcsname\undefined
  \newenvironment{#1}[1][]{%
    \begin{tcolorbox}[thmbox]\begin{#1Inner}[##1]%
  }{\end{#1Inner}\end{tcolorbox}}%
  \crefname{#1Inner}{#4}{#5}%
}

\defbox{theorem}{Theorem}{\itshape}{Thm.}{Thms.}
\defbox{lemma}{Lemma}{\itshape}{Lem.}{Lems.}
\defbox{proposition}{Proposition}{\itshape}{Prop.}{Props.}
\defbox{corollary}{Corollary}{\itshape}{Cor.}{Cors.}
\defbox{remark}{Remark}{\relax}{Rem.}{Rems.}
\defbox{definition}{Definition}{\relax}{Def.}{Defs.}

\newenvironment{restated}[1]
  {\par\medskip\noindent\textbf{#1} (restated).\ \itshape}
  {\par\medskip}

\usepackage{amsmath,amsfonts,bm}

\def\eqref#1{equation~\ref{#1}}

\def\1{\bm{1}}

\DeclareMathAlphabet{\mathsfit}{\encodingdefault}{\sfdefault}{m}{sl}
\SetMathAlphabet{\mathsfit}{bold}{\encodingdefault}{\sfdefault}{bx}{n}

\DeclareMathOperator*{\argmin}{arg\,min}

\title{Trajectory Replay: Continuous-Time Mixture Transport for Lifelong Generative Models}

\author[1,2,3\dagger]{Alessandro~Palma}
\author[4,5]{Charles~Jones}
\author[6\dagger]{Tristan~Cinquin}
\author[1,2,3]{Fabian~J.~Theis}
\author[4,5\star]{Frederik~Wenkel}

\affiliation[1]{Helmholtz Munich}

\affiliation[2]{Technical University of Munich}

\affiliation[3]{Munich Center for Machine Learning}

\affiliation[4]{Valence Labs}

\affiliation[5]{Recursion}

\affiliation[6]{University of Tübingen}

\contribution[\dagger]{Work done during an internship at Valence Labs \\
  $^{\star}$Corresponding author: \texttt{frederik@valencelabs.com}}

\abstract{
Adapting pre-trained continuous-time generative models to sequentially incoming data without forgetting earlier tasks is a central challenge in lifelong learning. A leading strategy, generative replay, uses a generative model trained on historical tasks to augment its own continual update with synthetic samples from experience. However, it requires full trajectory simulations at every training step, which accumulate discretization errors and can degrade the replayed historical distribution. Here, we propose \textit{trajectory replay}, which bypasses sample generation by training an updated velocity field to match both a frozen historical velocity along its trajectories and the incoming task's velocity field, weighted by a stability-plasticity coefficient. Through the lens of continuous-time transport, we show the resulting optimum averages the two task velocities by their respective densities and generates the mixture of their probability paths. Combined with mean velocity distillation, trajectory replay accesses historical trajectories using only a few network evaluations per continual update. Empirically, it outperforms generative replay and distillation baselines on class-incremental image generation. In two biological case studies, trajectory replay extends generative perturbation models to new contexts and experimental batches across two data modalities while preserving treatment-effect prediction. Code is available at \url{https://github.com/valence-labs/trajectory_replay}.
}

\begin{document}

\maketitle

\section{Introduction}
Continuous-time generative models, such as diffusion and flow matching \citep{song2020score, lipman2023flow}, have gained traction across image generation \citep{dhariwal2021diffusion, esser2024scaling}, language modeling \citep{li2022diffusionlm}, and scientific applications \citep{morehead2026flow, sriram2024flowllm}. Trained on large corpora of data to capture complex, multi-modal distributions, they yield reusable generative priors for a range of downstream tasks \citep{rombach2022high}. As novel data accumulates, however, effective strategies are needed to update pre-trained models with incoming knowledge while preserving generative capabilities over the historical distribution. Joint retraining on the union of historical and incoming data grows more expensive with every update, and is infeasible when the original data is no longer accessible, for instance due to privacy or licensing constraints.

Incorporating new data without full retraining motivates continual learning (CL) strategies \citep{wang2024comprehensive}, whereby a seed model is updated on incoming data streams while mitigating forgetting of previous tasks. In generative modeling, each \textit{task} corresponds to a new data distribution, absent from the historical dataset, that the model must learn to sample from \citep{guo2025continual}. CL approaches applied to generative models include architectural evolution \citep{smith2024continual}, regularization \citep{song2024towards}, and replay \citep{NEURIPS2019_fa7cdfad}. Among them, generative replay \citep{shin2017continual} is one of the most established \citep{van2022three}: the objective over incoming tasks is augmented with samples generated by the model itself over the historical distribution, alleviating catastrophic forgetting without storing or revisiting historical data.

Applying standard generative replay to continuous-time frameworks is expensive, as each replay sample requires iterative simulation of the underlying dynamics. Approximation errors compound along these trajectories, degrading replay fidelity and, in turn, subsequent CL updates \citep{gao2023ddgr}. One line of work stabilizes replay by renoising generated samples and regressing the frozen historical model's prediction on the renoised states, rather than using the samples directly as training data \citep{masip2025continual}, but generation is still required. Cheaper solutions that bypass data generation remain underexplored. Beyond this bottleneck, most work on CL for continuous-time models remains empirical and lacks a formal understanding of the distributional properties induced by sequential updates.

To address these shortcomings, we introduce \textit{trajectory replay} (\cref{fig: concept}), a flow-based generative replay framework for continuous-time models. Given a batch of incoming data and a vector field trained to generate the historical task distribution, our method splits the CL-update objective for a joint-task velocity field into two regression terms. The first is a standard conditional flow matching loss \citep{lipman2023flow} over the incoming data. The second regresses the joint-task velocity toward the frozen historical field along its own generative trajectory at intermediate times, so replay never runs through to fully generated samples. This halves the simulation cost in expectation and makes replay more robust to discretization error. A stability-plasticity coefficient weights the two terms and controls how much mass the optimum assigns to the historical and the incoming density. 

We validate trajectory replay in controlled simulations and class-incremental image-generation benchmarks, where it improves sample fidelity and retention of earlier tasks across sequential updates at a fraction of the network evaluations of existing CL strategies. Since continual adaptation is particularly important in biological discovery, where measurements arrive sequentially across cell lines, donors, and experimental batches, and where historical screens may be inaccessible, we evaluate two drug-discovery applications of generative perturbation prediction in single-cell gene expression and microscopy. Across these settings, trajectory replay offers the best stability-plasticity balance among methods that store no historical data. In summary, our contributions are as follows:

\begin{figure}[t]
\centering
\includegraphics[width=1\textwidth]{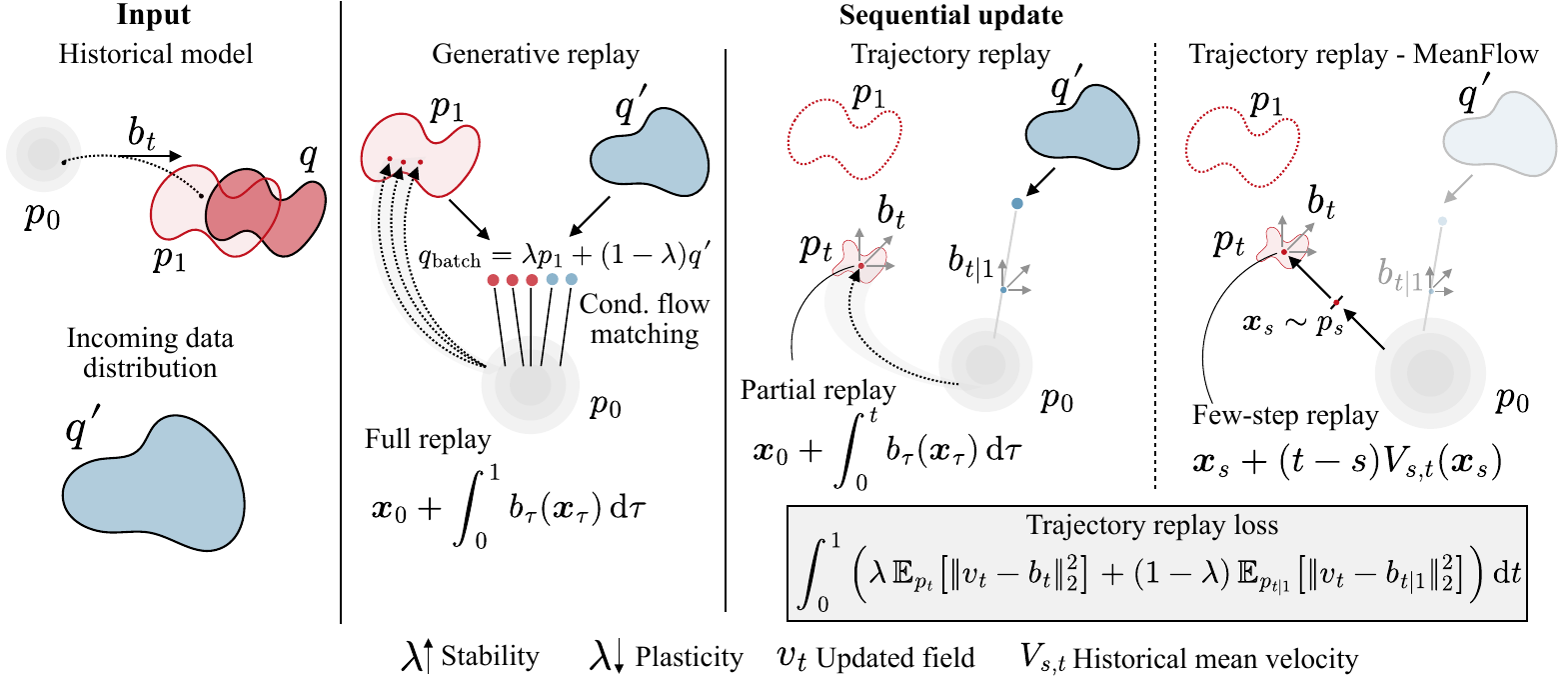}
\caption{\textbf{Trajectory replay for generative CL.} Given a historical flow model and an incoming task (left), generative replay simulates full trajectories to draw replay samples and trains on the mixture $q_{\mathrm{batch}}$. Trajectory replay instead regresses a single field against the historical field $b_t$ and the incoming task's data-conditioned velocity $b_{t|1}$ along their own paths, stopping the historical trajectory at an intermediate time. The MeanFlow variant optimizes the same objective with intermediate states drawn through the historical mean velocity $V_{s,t}$. Increasing $\lambda$ favors stability and decreasing it favors plasticity.}
\label{fig: concept}
\end{figure} 

\begin{itemize}[leftmargin=*, itemindent=0pt]
\item \textbf{Trajectory-level replay objective.} We formulate generative replay as velocity-field regression along the historical model's generative trajectories, and derive a matching continual update rule for its mean velocity \citep{geng2026mean}, reducing each update to fewer network evaluations.
\item \textbf{Optimality and stability-plasticity control.} We show the population optimum of our trajectory objective is a density-weighted mixture of the historical and incoming probability paths, with the mixture weight controlling the stability-plasticity trade-off, and that the objective recovers a weighted forward Kullback-Leibler (KL) divergence up to time reweighting.
\item \textbf{Gains across scales and domains.} Trajectory replay improves the fit-retention trade-off over baselines on simulations, image benchmarks, and two biological case studies.
\end{itemize}
\section{Background and preliminaries}
\subsection{Continuous-time generative flows}\label{sec: ct_generative_flows}
Let $\mathcal{X} \subseteq \mathbb{R}^d$ denote a continuous data space of dimensionality $d$, and let $q \in \mathcal{P}(\mathcal{X})$ denote the data density, where $\mathcal{P}(\mathcal{X})$ is the space of probability density functions supported on $\mathcal{X}$. Continuous normalizing flows sample from a probability path $\{p_t\}_{t \in [0, 1]}$, where $p_t \in \mathcal{P}(\mathcal{X})$, transitioning from a tractable prior $p_0:=\mathcal{N}(\bm{0}_d, \mathbb{I}_d)$ to a terminal density $p_1$ approximating $q$ \citep{chen2018neural}. Here, $p_t$ is generated by a \textit{marginal velocity} $b_t: \mathcal{X} \to \mathbb{R}^d$, for $t\in[0,1]$, according to the continuity equation $\partial_t p_t = -\nabla \cdot (p_t b_t)$, so that learning the generative process reduces to regressing $b_t$ with a parameterized function $b_t^{\theta}$.\looseness=-1

\textbf{Flow matching.} A prominent framework for learning these paths is flow matching \citep{lipman2023flow}, which optimizes a vector field $b_t^{\theta}$ to approximate $b_t$ across time via the following objective:
\begin{equation}\label{eq: flow_matching_obj}
    \mathcal{L}_{\text{FM}}(\theta) := \int_0^1 \mathbb{E}_{\bm{x} \sim p_t}\left[ \|b_t^{\theta}(\bm{x}) - b_t(\bm{x})\|_2^2 \right] \mathrm{d}t \, .
\end{equation}
As $b_t(\bm{x})$ is intractable, \citet{lipman2023flow} instead minimize its tractable conditional counterpart:
\begin{equation}\label{eq: c_flow_matching_obj}
    \mathcal{L}_{\text{CFM}}(\theta) := \mathbb{E}_{t\sim\mathcal{U}[0,1], \, \bm{x}_1\sim q, \, \bm{x} \sim p_{t|1}(\cdot\mid \bm{x}_1)}\left[ \|b_t^{\theta}(\bm{x}) - b_{t|1}(\bm{x} \mid \bm{x}_1)\|_2^2 \right] \, ,
\end{equation}
where $p_{t|1}(\cdot \mid \bm{x}_1):= \mathcal{N}(\alpha_t \bm{x}_1, \sigma_t^2 \mathbb{I}_d)$ is a Gaussian conditional path with scheduling parameters $\alpha_t,\sigma_t$, and $b_{t|1}$ its associated conditional velocity, equal to $\bm{x}_1-\bm{x}_0$ at $\bm{x}=t\bm{x}_1+(1-t)\bm{x}_0$, with $\bm{x}_0\sim p_0$, under the linear schedule $\alpha_t=t$, $\sigma_t=1-t$. Crucially, \cref{eq: flow_matching_obj} and \cref{eq: c_flow_matching_obj} differ only by a term independent of $\theta$, and therefore have identical gradients with respect to $\theta$ (details in \cref{sec: flow_matching_basics}).

\textbf{Mean velocity distillation.} Denote with $\bm{x}_t$ the integral solution of an ODE with marginal velocity $b_t$:\looseness=-1
\begin{equation}\label{eq: ode_solution}
    \frac{\mathrm{d}}{\mathrm{d}t}\bm{x}_t = b_t(\bm{x}_t), \qquad \bm{x}_t = \bm{x}_0 + \int_0^t b_{\tau}(\bm{x}_{\tau}) \, \mathrm{d}\tau \, ,
\end{equation}
where $\bm{x}_0 \sim p_0$, so $\bm{x}_t \sim p_t$ for all $t \in [0,1]$ (\cref{sec: flow_matching_basics}). As solving \cref{eq: ode_solution} requires many discretization steps and is therefore expensive, recent work learns a \textit{mean velocity} \citep{geng2026mean,geng2026improved}, $V_{s,t}(\bm{x}_s):=\frac{1}{t-s}\int_s^t b_{\tau}(\bm{x}_{\tau})\mathrm{d}\tau$, averaging $b_\tau$ over the interval $[s,t]$ for $0 \le s < t \le 1$ and $\bm{x}_s\sim p_s$. Differentiating this definition with respect to $s$ yields the \textit{MeanFlow Identity}, whose residual
\begin{equation}\label{eq: consistency_residual}
    \mathcal{R}_{s,t}(V,b)(\bm{x}) := V_{s,t}(\bm{x}) - \mathrm{sg}\Big[b_s(\bm{x}) + (t-s)\big(\partial_s V_{s,t}(\bm{x}) + \nabla_{\bm{x}} V_{s,t}(\bm{x})\, b_s(\bm{x})\big)\Big]
\end{equation}
vanishes exactly at the mean velocity of $b$, meaning $\mathcal{R}_{s,t}(V,b)=\bm{0}$ whenever $V=V_{s,t}$ as defined above. Here $\mathrm{sg}[\cdot]$ is the stop-gradient, so differentiating only outside it yields a \textit{semigradient} rather than a true gradient (\cref{sec: semigradient}). Minimizing $\|\mathcal{R}_{s,t}(V^\theta,b)\|_2^2$ thus gives a tractable objective for a parameterized mean field $V^{\theta}_{s,t}$, trained end-to-end or by distillation from a teacher field $b$ (\cref{sec: mean_flows_basics}). $V_{s,t}$ also defines a \textit{flow map} $\bm{X}_{s,t}(\bm{x}_s) := \bm{x}_s + (t-s)V_{s,t}(\bm{x}_s)$, which maps a state at $s$ directly to its corresponding state at $t$, satisfying $p_t = (\bm{X}_{s,t})_\sharp p_s$ \citep{boffi2024flow}. Substituting the learned $V^{\theta}_{s,t}$ for $V_{s,t}$ in $\bm{X}_{s,t}$ therefore allows sampling in a few discrete steps, without simulating \cref{eq: ode_solution}.

\subsection{Continual learning and replay for generative models}\label{sec: cl_and_replay_for_gen}
\textbf{Problem statement.} We consider a CL paradigm where a generative model trained on an initial data distribution is extended to sequentially arriving data streams. We illustrate our setting across a single update step, though our framework extends to multiple iterations. Denote $q\in\mathcal{P}(\mathcal{X})$ as the historical density, the ground-truth distribution of the initial task dataset before sequential updates, and $q'\in\mathcal{P}(\mathcal{X})$ as the incoming data's ground-truth density for the subsequent CL step. The goal is to approximate a \textit{target density} $\tilde{q}\in\mathcal{P}(\mathcal{X})$ that preserves the ability to sample from both, mitigating catastrophic forgetting on the old task while ensuring accurate adaptation to incoming information.

\textbf{Generative replay.} Consider a flow model $\{p_t\}_{t\in[0,1]}$ approximating $q$ at $t=1$ (\cref{sec: ct_generative_flows}). \textit{Naive fine-tuning} on samples from $q'$ leads to forgetting, as the velocity field is retrained to transport noise onto $q'$ alone, overwriting the transport onto $q$ in the regions the two paths share. Generative replay \citep{shin2017continual} mitigates this by augmenting training batches with samples generated from $p_1$ at rate $\lambda \in (0,1)$, modulating the \textit{stability-plasticity} trade-off that governs how far the model may deviate from the historical task. However, generating each replay sample requires simulating the full generative trajectory defined in \cref{eq: ode_solution}, making this augmentation expensive and susceptible to discretization errors.\looseness=-1

\section{Trajectory replay for continual learning}
First, we characterize the terminal distribution a model generates after a replay update. Training on replay-augmented batches with the objective in \cref{eq: c_flow_matching_obj} optimizes a vector field whose updated terminal marginal approximates $\tilde{q}$, the mixture of historical and incoming densities (details in \cref{sec: replay_as_mixture_modeling}):
\begin{equation}\label{eq: mixture}
    \tilde{q}(\bm{x}):=\lambda q(\bm{x})+(1-\lambda)q'(\bm{x}).
\end{equation}
This mixture behavior admits a variational justification: modeling CL updates as a $\lambda$-weighted mixture is equivalent to minimizing a forward KL divergence objective. Specifically, one can show that the $\tilde{q}$ in \cref{eq: mixture} is the unique global minimizer of the following optimization problem:
\begin{equation}\label{eq: kl_optimization_problem}
    \argmin_{\rho\in\mathcal{P}(\mathcal{X})}\ \lambda D_{\mathrm{KL}}(q\parallel\rho) +(1-\lambda)D_{\mathrm{KL}}(q'\parallel\rho) \, .
\end{equation}
The formalization and proof are in \cref{sec: static_kl_as_mixture_minimizer}. Since the forward $D_{\mathrm{KL}}(q\parallel\rho)$ penalizes low mass where $q$ has support, its minimizer is mode-covering, discouraging collapse onto either task and thus retaining history while adapting to incoming data, as desired for continual updates. In this work, we develop a more efficient continuous-time solution for such mixture-based updates at the trajectory level.

\subsection{A continuous-time transport framework for generative CL}\label{sec: ct_framework}
\textbf{Task.} We consider a historical model whose velocity field $b_t$ generates the path $p_t$, and an incoming one whose velocity field $b'_t$ generates $p'_t$ (\cref{sec: cl_and_replay_for_gen}). We aim to find a path $\rho_t^\star$ mixing these two generative trajectories and recovering \cref{eq: mixture} at $t=1$. To this end, we pose an optimization problem over a time-dependent velocity field $v:=\{v_t\}_{t\in[0,1]}$, with $v_t:\mathbb{R}^d\rightarrow\mathbb{R}^d$, tasked with generating $\rho_t^\star$. We write $b:=\{b_t\}_{t\in[0,1]}$, $b':=\{b'_t\}_{t\in[0,1]}$ for the respective families of historical and incoming task velocities.

\textit{Standing assumptions.} Throughout this section, $\{p_t\}_{t\in[0,1]}$ and $\{p'_t\}_{t\in[0,1]}$, with $p_t,p'_t:\mathbb{R}^d\rightarrow\mathbb{R}_{\geq0}$, are marginal paths induced by a shared Gaussian conditional kernel $p_{t|1}$ (\cref{sec: ct_generative_flows}), with scheduling parameters $\alpha_t,\sigma_t$ satisfying $\alpha_t,\sigma_t>0$ and $\dot{\alpha}_t\sigma_t-\alpha_t\dot{\sigma}_t>0$ on $(0,1)$. They are generated through the continuity equation of \cref{sec: ct_generative_flows} by velocity fields $b,b'\in C^0([0,1],(C^1(\mathbb{R}^d))^d)$, starting from the shared prior $p_0=p'_0:=\mathcal{N}(\bm{0}_d,\mathbb{I}_d)$ and ending at the task densities $p_1:=q$ and $p'_1:=q'$. We fix $\lambda\in(0,1)$ and assume $\lambda p_t(\bm{x})+(1-\lambda)p'_t(\bm{x})>0$ for all $t\in[0,1]$ and $\bm{x}\in\mathcal{X}$.

\textbf{Trajectory replay target.} Generative replay preserves the historical task through its terminal samples. We instead propose to preserve it through its velocity field along the whole path: the updated field should agree with $b_t$ where the historical path carries mass, and with $b'_t$ where the incoming path does. Under these assumptions, and taking $v$ in the same regularity class as $b$ and $b'$, we therefore define the \textit{trajectory replay target} $v^\star$ as the minimizer of a two-term regression objective
\begin{equation}\label{eq: mixture_objective}
    \mathcal{J}(v; b, b') := \int_0^1 \Big(\lambda\,\mathbb{E}_{p_t}\big[\|v_t(\bm{x})-b_t(\bm{x})\|_2^2\big] + (1-\lambda)\,\mathbb{E}_{p'_t}\big[\|v_t(\bm{x})-b'_t(\bm{x})\|_2^2\big]\Big)\,\mathrm{d}t \, ,
\end{equation}
so that $v^\star = \argmin_{v}\ \mathcal{J}(v; b, b')$. Each term regresses $v_t$ onto one task velocity under that task's marginal, with $\lambda$ weighing the old and new contributions. The objective admits a unique minimizer.

\begin{restatable}{proposition}{propobjmin}\label{prop: objective_minimizer}
Under the standing assumptions, $\mathcal{J}(\cdot\,;b,b')$ in \cref{eq: mixture_objective} admits
\begin{equation}\label{eq: dynamic_kl_minimizer}
    v_t^{\star}(\bm{x}) := \frac{\lambda\, p_t(\bm{x})\, b_t(\bm{x}) + (1-\lambda)\, p'_t(\bm{x})\, b'_t(\bm{x})}{\lambda\, p_t(\bm{x}) + (1-\lambda)\, p'_t(\bm{x})} \, ,
\end{equation}
as unique minimizer, which generates the mixture path $\rho_t^\star := \lambda p_t+(1-\lambda)p'_t$, satisfying the equation $\partial_t\rho_t^\star+\nabla_{\bm{x}}\cdot(v_t^\star\rho_t^\star)=0$ with $\rho_0^\star=\mathcal{N}(\bm{0}_d,\mathbb{I}_d)$ and $\rho_1^\star=\tilde{q}$ (\cref{eq: mixture}).\looseness=-1
\end{restatable}
The proof is in \cref{sec: proof_unique_minimizer}. The optimum is a pointwise average of the two task velocities, weighted by their densities and by $\lambda$, generating a path $\rho_t^\star$ that interpolates the two trajectories and reduces to $\tilde{q}$ at $t=1$. Where one task's density dominates, $v^\star_t$ approaches that task's velocity, and it blends the two where their densities are comparable, which near $t=0$ is everywhere, since both start from the same prior. Two insights follow from \cref{prop: objective_minimizer}. (\textbf{I}) $\lambda$ fixes how much mass the updated model retains on the historical task relative to the incoming one. (\textbf{II}) Since both task paths share the Gaussian kernel $p_{t|1}(\bm{x}\mid\bm{x}_1)=\mathcal{N}(\bm{x};\alpha_t\bm{x}_1,\sigma_t^2\mathbb{I}_d)$ by the standing assumptions, $\rho_t^{\star}(\bm{x}) = \int p_{t|1}(\bm{x}\mid\bm{x}_1)\,\tilde{q}(\bm{x}_1)\,\mathrm{d}\bm{x}_1$, so $\rho_t^\star$ is itself a marginal path of the same kernel, over $\tilde{q}$ (see \cref{sec: gaussianity_marginal_mixture_path}).

\textbf{KL-optimality.} The trajectory objective in \cref{eq: mixture_objective} connects back to the static KL problem in \cref{eq: kl_optimization_problem}. Evaluating the latter at the endpoints of the generative paths gives the weighted terminal divergence $\lambda D_{\mathrm{KL}}(p_1\parallel\rho_1)+(1-\lambda)D_{\mathrm{KL}}(p'_1\parallel\rho_1)$, measuring how well a candidate path $\rho_t$ reproduces the two tasks at $t=1$, which admits an exact trajectory-level expression.
\begin{theorem}\label{thm: rho_minim_kl}
Under the standing assumptions, let $\rho_t$ be a candidate marginal path induced by the same kernel, with velocity $v_t$ and shared prior $\rho_0=p_0$, and define $\kappa_t := \alpha_t\big/\big(\sigma_t(\dot{\alpha}_t\sigma_t-\alpha_t\dot{\sigma}_t)\big) > 0$ on $(0,1)$. The weighted terminal divergence then satisfies
\begin{equation}\label{eq: kl_trajectory_equivalence}
\begin{aligned}
    &\lambda D_{\mathrm{KL}}(p_1\parallel\rho_1) + (1-\lambda)D_{\mathrm{KL}}(p'_1\parallel\rho_1) \\
    &\qquad = \int_0^1 \kappa_t \Big(\lambda\,\mathbb{E}_{p_t}\big[\|v_t(\bm{x})-b_t(\bm{x})\|_2^2\big] + (1-\lambda)\,\mathbb{E}_{p'_t}\big[\|v_t(\bm{x})-b'_t(\bm{x})\|_2^2\big]\Big)\mathrm{d}t \, .
\end{aligned}
\end{equation}
\end{theorem}
The proof is in \cref{sec: proof_rho_minim_kl} and follows \citet{albergo2025stochastic}. The right-hand side of \cref{eq: kl_trajectory_equivalence} is the trajectory objective of \cref{eq: mixture_objective} with each time weighted by $\kappa_t$. Since $\kappa_t>0$ on $(0,1)$, this reweighting changes how much each time contributes to the objective but not where its pointwise minimizer lies, so the unweighted objective of \cref{eq: mixture_objective} and its $\kappa_t$-weighted counterpart in \cref{eq: kl_trajectory_equivalence} share the same minimizer $v^\star$. Moreover, the path generated by $v^\star$ is induced by the same Gaussian kernel as the two task paths (\cref{sec: gaussianity_marginal_mixture_path}), so it is an admissible candidate in \cref{thm: rho_minim_kl}. The same construction extends to \textit{stochastic generative models}, where it recovers the same density-weighted mixture (\cref{sec: stochastic_formulation}).

\vspace{-0.1cm}

\textbf{Mean velocity continual update.} Evaluating \cref{eq: mixture_objective} requires samples from the marginal paths $p_t$ and $p'_t$, which in general means integrating \cref{eq: ode_solution} under their respective velocity fields. At each continual update, we can therefore additionally learn a mean velocity $V^\star_{s,t}$ alongside the mixture velocity $v^\star_t$, averaging $v^\star$ over $[s,t]$ for $0\le s<t\le1$ and $\bm{x}_s\sim\rho^\star_s$:
\begin{equation}\label{eq: mixture_mean_velocity}
    V_{s,t}^{\star}(\bm{x}_s) := \frac{1}{t-s}\int_s^t v^{\star}_{\tau}(\bm{x}^\star_{\tau})\,\mathrm{d}\tau \qquad \text{with} \qquad \bm{x}^\star_s = \bm{x}_s \, ,
\end{equation}
where $\bm{x}^\star_\tau$ solves the ODE in \cref{eq: ode_solution} with velocity $v^\star$ from \cref{prop: objective_minimizer}, started from $\bm{x}_s$ at time $s$. The role of $V^\star_{s,t}$ appears at the next sequential update, where the mixture path $\rho^\star_t$ becomes the historical path and must be sampled at intermediate times to evaluate the historical term of \cref{eq: mixture_objective}. Instead of integrating $v^\star$ from $t=0$ for each such sample, the flow map $\bm{X}^\star_{s,t}$ induced by $V^\star_{s,t}$ jumps directly between times, reaching the required states in a few network evaluations (\cref{sec: theory_in_practice}). Writing $V:=\{V_{s,t}\}_{0\le s\le t\le 1}$ for a candidate mean velocity, the counterpart of \cref{eq: mixture_objective} at the level of mean velocities weights the consistency residuals of \cref{eq: consistency_residual} along the two task paths:
\begin{equation}\label{eq: mean_flow_mixture_objective}
    \mathcal{K}(V; b, b') := \int_0^1\!\!\int_s^1 \Big(\lambda\,\mathbb{E}_{p_s}\big[\|\mathcal{R}_{s,t}(V,b)(\bm{x})\|_2^2\big] + (1-\lambda)\,\mathbb{E}_{p'_s}\big[\|\mathcal{R}_{s,t}(V,b')(\bm{x})\|_2^2\big]\Big)\,\mathrm{d}t\,\mathrm{d}s \, .
\end{equation}

\begin{proposition}\label{prop: mean_flow_minimizer}
Under the standing assumptions, let $V$ be continuously differentiable in $s$ and $\bm{x}$ on $0\le s\le t\le 1$. Then the semigradient of $\mathcal{K}(\cdot\,;b,b')$ vanishes if and only if $V$ satisfies the MeanFlow identity for $v^\star$, that is $\mathcal{R}_{s,t}(V,v^\star)=\bm{0}$ for all $0\le s<t\le 1$, with $v^\star$ from \cref{eq: dynamic_kl_minimizer}, whose unique solution is the mixture mean velocity $V^\star_{s,t}$ of \cref{eq: mixture_mean_velocity}.
\end{proposition}
The proof is in \cref{sec: proof_mean_flow_minimizer}. \Cref{prop: mean_flow_minimizer} mirrors \cref{prop: objective_minimizer} one level up: regressing a single mean velocity against the two task residuals recovers the mean velocity of the mixture path, with the density weighting emerging rather than imposed. Since \cref{eq: mean_flow_mixture_objective} involves only the instantaneous fields $b_s,b'_s$, the mixture mean velocity can be learned jointly with $v^\star$, without ever forming $v^\star$ explicitly.

\subsection{Trajectory replay in practice}\label{sec: theory_in_practice}
\textbf{Sequential updates.} Our framework extends to a stream of tasks $k\in\{1,\dots,K\}$ with densities $q_k$, where update $k$ targets a mixture between the $k$ tasks accumulated so far and the incoming task $k+1$. At each update we have only a model for the historical tasks, $b_t^{\theta_{\le k}}$, and only data for the incoming one. The historical term therefore regresses onto the previous model, while the incoming term uses the \textit{data-conditioned velocity}, which yields the same gradient as the incoming marginal velocity (\cref{sec: ct_generative_flows}). Each update trains a joint field $v_t^{\theta_{\le k+1}}$, which becomes $b_t^{\theta_{\le k+1}}$ at the next update, with the first task trained by flow matching alone (see \cref{alg: trajectory_replay_ode}):
\begin{equation}\label{eq: empirical_mixture_objective}
    \mathcal{J}^{k+1} = \lambda_k\, \mathbb{E}\big[\|v^{\theta_{\le k+1}}_t(\bm{x}_t) - b^{\theta_{\le k}}_t(\bm{x}_t)\|_2^2\big] + (1-\lambda_k)\, \mathbb{E}\big[\|v^{\theta_{\le k+1}}_t(\bm{x}) - b_{t|1}(\bm{x}\mid\bm{x}_1)\|_2^2\big] \, ,
\end{equation}
with $t\sim\mathcal{U}(0,1]$, $\bm{x}_t$ obtained by integrating \cref{eq: ode_solution} from $\bm{x}_0\sim p_0$ under the frozen $b^{\theta_{\le k}}$ only as far as $t$, and $\bm{x}_1\sim q_{k+1}$, $\bm{x}\sim p_{t|1}(\cdot\mid\bm{x}_1)$. In expectation, stopping early costs \textit{half the evaluations} that generative replay needs, plus one for the target at the anchor. It also changes where discretization error lands: the frozen model supplies the velocity at whatever state the solver reaches, so the error only shifts which states the student sees rather than entering the loss as corrupted data (\cref{sec: partial_replay_analysis}).\looseness=-1

\textbf{Efficient sampling along the trajectory.} \Cref{eq: empirical_mixture_objective} still integrates \cref{eq: ode_solution} at every training step, at a cost set by the solver. The \textit{MeanFlow backend} removes that dependence by training, alongside $v^{\theta_{\le k+1}}$, a mean velocity $V^{\theta_{\le k+1}}$ under the empirical counterpart $\mathcal{K}^{k+1}$ of \cref{eq: mean_flow_mixture_objective}. Its two residuals use the same teachers as \cref{eq: empirical_mixture_objective}: the historical residual takes the frozen field $b^{\theta_{\le k}}$ as teacher at states drawn from the historical path, and the incoming residual takes the data-conditioned velocity $b_{s|1}(\bm{x}\mid\bm{x}_1)$ at $\bm{x}\sim p_{s|1}(\cdot\mid\bm{x}_1)$, with $\bm{x}_1\sim q_{k+1}$, which recovers the same update as the incoming marginal velocity (\cref{sec: semigradient}). Each update produces two fields that are passed forward together: $v^{\theta_{\le k+1}}$, which becomes the next update's target, and $V^{\theta_{\le k+1}}$, whose flow map $\bm{X}^{\theta_{\le k+1}}_{s,t}$ becomes the sampler for the next update's historical states (\cref{alg:trajectory_replay_meanflow}). Before the first update, the mean velocity is distilled from the first task's field.\looseness=-1 


\textbf{Choosing $\lambda_k$.} Each value of $\lambda_k$ yields a different minimizer in \cref{eq: dynamic_kl_minimizer}, striking a different balance between fitting the incoming task and forgetting earlier ones, so the best choice \textit{depends on the metric at hand}. Holding $\lambda_k$ constant is rarely appropriate, as task weights then decay geometrically with age rather than balancing (\cref{sec: selection_of_lambda}). We consider two schedules. \textit{Equal-task} weighting sets $\lambda_k=k/(k+1)$, which gives the $k$ accumulated tasks and the incoming one equal mass at every update and requires no tuning. \textit{Adaptive} weighting instead sweeps $\lambda$ at each update and scores every candidate on old-task and new-task validation sets, picking the best point on the retention-adaptation trade-off. \looseness=-1
\section{Experiments}\label{sec: experiments}
\subsection{Controlled mixture recovery}
\label{sec:results-mixture-recovery}

\begin{wrapfigure}[15]{r}{0.5\textwidth}
\vspace{-33pt}
\centering
\includegraphics[width=0.48\textwidth]{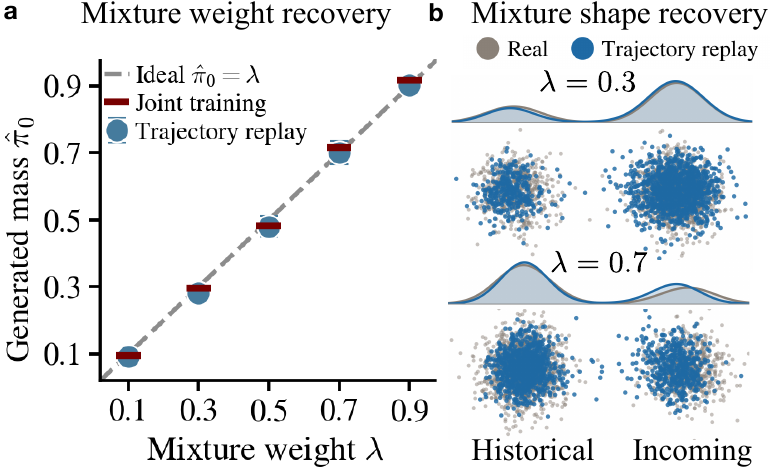}
\caption{\textbf{(a)} Generated historical-mode mass versus mixture weight $\lambda$ for trajectory replay and joint training (mean ± SD, three seeds).
\textbf{(b)} Target and generated samples at $\lambda\in\{0.3,0.7\}$. Historical and incoming modes are on the left and right.}
\label{fig:simulation-mixture}
\end{wrapfigure}
\textbf{Task.} Before testing trajectory replay at scale, we verify in a simulation study that the updated field recovers the terminal task mixture $\tilde{q}=\lambda q+(1-\lambda)q'$ predicted by \cref{prop: objective_minimizer}. A bimodal Gaussian mixture lets us evaluate this recovery directly, since the ground-truth mixture weight and mode geometry are known exactly.

\textbf{Setup.} In $d=2$, a historical flow model pre-trained on $q:=\mathcal{N}(\bm{\mu}_0,\sigma^2\mathbb{I}_2)$ is updated with samples from $q':=\mathcal{N}(\bm{\mu}_1,\sigma^2\mathbb{I}_2)$, where $\bm{\mu}_0:=(-1.5,0)$, $\bm{\mu}_1:=(1.5,0)$, and $\sigma:=0.35$, using $10{,}000$ samples per task and three seeds throughout. Across the grid $\lambda\in\{0.1,0.3,0.5,0.7,0.9\}$, we measure how closely the generated mass, computed by assigning samples to their nearest mode mean, matches $\lambda$, and how well mode shape is preserved, with joint training on the true mixture as an oracle reference. We additionally repeat the update at $d\in\{128,256\}$ with $\lambda=0.5$ to compare the teacher cost of the replay methods (see \cref{app:results-mixture-recovery-exp}).

\textbf{Results.} Trajectory replay recovers the prescribed historical mass across all $\lambda$, matching joint training and the ideal calibration line, and preserves mode location and geometry while shifting relative mass as prescribed (\cref{fig:simulation-mixture}a,b). In higher dimensions, it also reaches a given fit at a lower teacher cost than generative replay, with the gap widening as $d$ grows (\cref{fig: app_fig2}). Naive fine-tuning instead collapses onto the incoming mode, recovering no historical mass (\cref{fig: app_fig1}b).

\subsection{Sequential updates on a ring-structured Gaussian mixture}\label{sec:results-sequential-rings}
\textbf{Task.} We extend to a five-task stream of separated Gaussian modes with means arranged on a ring, testing whether trajectory replay preserves a balanced distribution across tasks without storing historical samples. Here, losses of historical mass or geometry are directly visible in the generated task proportions and in the distance to the uniform five-task target.

\vspace{-0.1cm}
\begin{wrapfigure}[17]{r}{0.5\linewidth}
\vspace{-10pt}
\centering
\includegraphics[width=\linewidth]{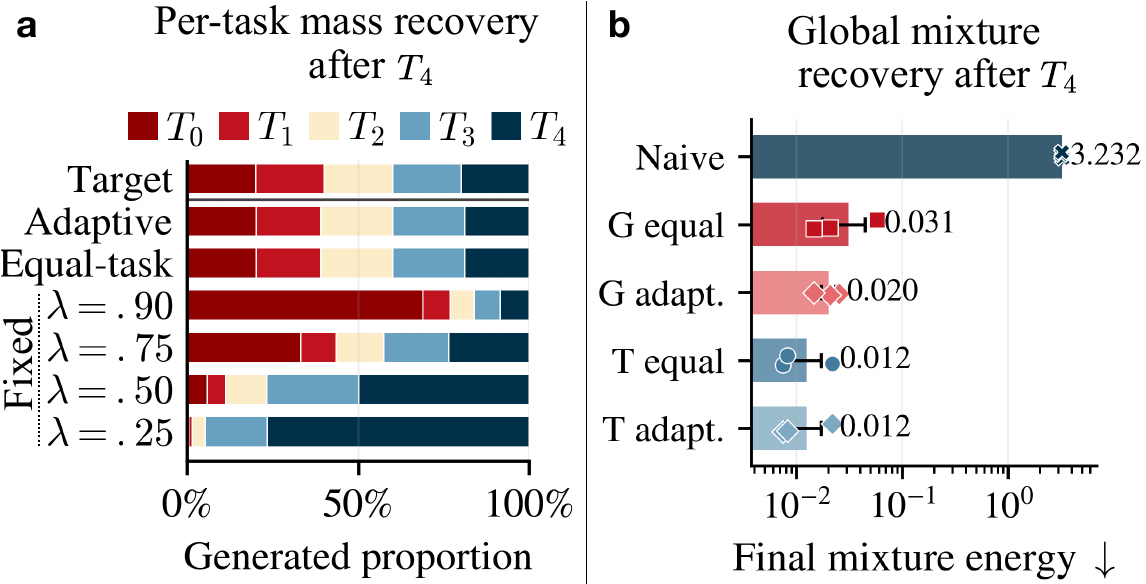}
\caption{\textbf{(a)} Generated task mass after the final update for trajectory replay with fixed-$\lambda$, equal-task, and adaptive policies. \textbf{(b)} Energy distance to the balanced held-out mixture for naive fine-tuning and equal-task/adaptive generative (G) and trajectory (T) replay (mean {\scriptsize$\pm$} SE, three seeds).}
\label{fig: sequential_gaussian_simulation}
\end{wrapfigure}

\textbf{Data \& Evaluation.} We define five Gaussian tasks $T_0,\ldots,T_4$ arriving sequentially in $d=2$, with means equally spaced on a circle of radius $3$ and isotropic standard deviation parameter $\sigma:=0.35$. Each task provides $10{,}000$/$2{,}000$/$5{,}000$ training/validation/test points, and the final target places $20\%$ probability mass on each task (see \cref{fig: app_fig3}a). We estimate per-task mass by labeling generated samples with their nearest mode mean and measure global recovery via energy distance to a balanced joint-task held-out mixture (details in \cref{app:results-sequential-rings-details}).

\vspace{-0.1cm}

\textbf{Baselines.} We compare naive fine-tuning, which updates the model on incoming data alone, with generative and trajectory replay under equal-task ($\lambda_k:=k/(k+1)$) and adaptive schedules (\cref{sec: theory_in_practice}), the latter sweeping $\lambda$ at each update and selecting by validation energy distance to the uniform mixture seen so far. Since that target is uniform, the equal-task rule \textit{is already optimal here} and the adaptive selection should recover it. Trajectory replay is additionally ablated at constant $\lambda$.

\vspace{-0.1cm}

\textbf{Results.} Naive fine-tuning collapses almost entirely onto $T_4$ (\cref{fig: app_fig3}b), whereas equal-task and adaptive trajectory replay recover a near-uniform distribution across all five tasks (\cref{fig: sequential_gaussian_simulation}a). Fixing $\lambda$ trades retention against fit, with high values holding earlier tasks and low values favoring $T_4$. Trajectory replay reaches the lowest energy distance to the balanced held-out mixture under both schedules, ahead of generative replay and with increased stability across tasks (\cref{fig: sequential_gaussian_simulation}b, \cref{fig: app_fig5}). Expectedly, adaptive replay recovers the equal-task rule at every update, so the two tie (\cref{fig: app_fig4}b).

\subsection{Class-incremental image generation}
\label{sec:results-image-cl}
\textbf{Task.} We evaluate class-incremental image generation, where disjoint classes arrive sequentially as CL tasks, following the same protocol as \citet{masip2025continual} (see \cref{sec:exp-image-cl-details} for full details).

\begin{table*}[t]
\centering
\definecolor{trajectoryrow}{HTML}{E5EEF3}
\definecolor{upperboundrow}{HTML}{F0F0F0}
\setlength{\tabcolsep}{4.5pt}
\small
\caption{Final-stream performance after five tasks on Fashion-MNIST and CIFAR-10 and three tasks on ImageNet-100 (mean {\scriptsize$\pm$} SE, three seeds). FID compares generated samples with pooled held-out data, KLD measures class-mass recovery, and MF-$N$ uses $N$ flow-map steps. Blue rows use trajectory replay, gray rows joint training with all data. Bold marks the best non-joint column mean.}
\label{tab:image_cl_all}
\resizebox{1\textwidth}{!}{%
\begin{tabular}{@{}lccccc@{}}
\toprule
Method & \multicolumn{2}{c}{Fashion-MNIST} & \multicolumn{2}{c}{CIFAR-10} & ImageNet-100 \\
\cmidrule(lr){2-3}\cmidrule(lr){4-5}\cmidrule(l){6-6}
& $\downarrow$ FID & $\downarrow$ KLD & $\downarrow$ FID & $\downarrow$ KLD & $\downarrow$ FID \\
\midrule
\rowcolor{upperboundrow} \textbf{Joint training} (all data; upper bound) & $12.5$ {\scriptsize $(\pm\,0.5)$} & $0.040$ {\scriptsize $(\pm\,0.010)$} & $25.9$ {\scriptsize $(\pm\,0.1)$} & $0.100$ {\scriptsize $(\pm\,0.020)$} & $50.7$ {\scriptsize $(\pm\,0.7)$} \\
Diffusion naive fine-tuning & $67.7$ {\scriptsize $(\pm\,6.6)$} & $4.140$ {\scriptsize $(\pm\,1.400)$} & $53.6$ {\scriptsize $(\pm\,4.4)$} & $1.100$ {\scriptsize $(\pm\,0.140)$} & $105.8$ {\scriptsize $(\pm\,3.1)$} \\
Generative replay (diffusion) \citep{masip2025continual} & $184.9$ {\scriptsize $(\pm\,12.6)$} & $3.680$ {\scriptsize $(\pm\,1.320)$} & $84.3$ {\scriptsize $(\pm\,9.1)$} & $0.590$ {\scriptsize $(\pm\,0.110)$} & $106.3$ {\scriptsize $(\pm\,2.5)$} \\
Generative distillation (diffusion) \citep{masip2025continual} & $25.5$ {\scriptsize $(\pm\,1.1)$} & $0.290$ {\scriptsize $(\pm\,0.040)$} & $38.5$ {\scriptsize $(\pm\,4.4)$} & $0.230$ {\scriptsize $(\pm\,0.080)$} & $85.5$ {\scriptsize $(\pm\,3.4)$} \\
\midrule
Flow matching naive fine-tuning & $64.6$ {\scriptsize $(\pm\,11.4)$} & $5.427$ {\scriptsize $(\pm\,2.446)$} & $51.2$ {\scriptsize $(\pm\,4.0)$} & $1.189$ {\scriptsize $(\pm\,0.172)$} & $103.4$ {\scriptsize $(\pm\,2.6)$} \\
\midrule
Generative replay (ODE, adaptive $\lambda$) & $35.0$ {\scriptsize $(\pm\,2.4)$} & $0.493$ {\scriptsize $(\pm\,0.131)$} & $43.3$ {\scriptsize $(\pm\,2.3)$} & $0.501$ {\scriptsize $(\pm\,0.046)$} & $88.4$ {\scriptsize $(\pm\,1.3)$} \\
Generative replay (ODE, equal-task $\lambda$) & $63.9$ {\scriptsize $(\pm\,3.5)$} & $0.526$ {\scriptsize $(\pm\,0.161)$} & $92.4$ {\scriptsize $(\pm\,15.0)$} & $0.737$ {\scriptsize $(\pm\,0.203)$} & $91.3$ {\scriptsize $(\pm\,4.5)$} \\
Generative replay (MF-1, adaptive $\lambda$) & $52.2$ {\scriptsize $(\pm\,2.7)$} & $0.748$ {\scriptsize $(\pm\,0.280)$} & $55.8$ {\scriptsize $(\pm\,7.7)$} & $0.531$ {\scriptsize $(\pm\,0.081)$} & $151.0$ {\scriptsize $(\pm\,25.7)$} \\
Generative replay (MF-1, equal-task $\lambda$) & $88.9$ {\scriptsize $(\pm\,7.1)$} & $0.277$ {\scriptsize $(\pm\,0.030)$} & $180.2$ {\scriptsize $(\pm\,37.4)$} & $0.700$ {\scriptsize $(\pm\,0.254)$} & $129.0$ {\scriptsize $(\pm\,6.4)$} \\
Generative replay (MF-4, adaptive $\lambda$) & $37.8$ {\scriptsize $(\pm\,3.4)$} & $0.504$ {\scriptsize $(\pm\,0.160)$} & $45.2$ {\scriptsize $(\pm\,1.4)$} & $0.508$ {\scriptsize $(\pm\,0.044)$} & $140.6$ {\scriptsize $(\pm\,31.5)$} \\
Generative replay (MF-4, equal-task $\lambda$) & $53.8$ {\scriptsize $(\pm\,4.3)$} & $0.202$ {\scriptsize $(\pm\,0.029)$} & $113.1$ {\scriptsize $(\pm\,14.4)$} & $0.265$ {\scriptsize $(\pm\,0.066)$} & $117.7$ {\scriptsize $(\pm\,4.1)$} \\
\midrule
\rowcolor{trajectoryrow} Trajectory replay (ODE, adaptive $\lambda$) & $21.1$ {\scriptsize $(\pm\,3.6)$} & $0.190$ {\scriptsize $(\pm\,0.090)$} & $\mathbf{31.4}$ {\scriptsize $(\pm\,0.9)$} & $0.272$ {\scriptsize $(\pm\,0.055)$} & $\mathbf{70.0}$ {\scriptsize $(\pm\,5.8)$} \\
\rowcolor{trajectoryrow} Trajectory replay (ODE, equal-task $\lambda$) & $22.1$ {\scriptsize $(\pm\,3.2)$} & $0.068$ {\scriptsize $(\pm\,0.039)$} & $38.6$ {\scriptsize $(\pm\,3.5)$} & $0.237$ {\scriptsize $(\pm\,0.066)$} & $70.4$ {\scriptsize $(\pm\,9.2)$} \\
\rowcolor{trajectoryrow} Trajectory replay (MF-1, adaptive $\lambda$) & $20.2$ {\scriptsize $(\pm\,2.2)$} & $0.140$ {\scriptsize $(\pm\,0.008)$} & $33.3$ {\scriptsize $(\pm\,0.5)$} & $0.332$ {\scriptsize $(\pm\,0.066)$} & $70.7$ {\scriptsize $(\pm\,5.3)$} \\
\rowcolor{trajectoryrow} Trajectory replay (MF-1, equal-task $\lambda$) & $24.2$ {\scriptsize $(\pm\,0.8)$} & $\mathbf{0.065}$ {\scriptsize $(\pm\,0.027)$} & $36.8$ {\scriptsize $(\pm\,1.2)$} & $0.237$ {\scriptsize $(\pm\,0.089)$} & $83.0$ {\scriptsize $(\pm\,9.3)$} \\
\rowcolor{trajectoryrow} Trajectory replay (MF-4, adaptive $\lambda$) & $18.4$ {\scriptsize $(\pm\,1.5)$} & $0.120$ {\scriptsize $(\pm\,0.023)$} & $33.5$ {\scriptsize $(\pm\,1.7)$} & $0.349$ {\scriptsize $(\pm\,0.056)$} & $71.1$ {\scriptsize $(\pm\,2.7)$} \\
\rowcolor{trajectoryrow} Trajectory replay (MF-4, equal-task $\lambda$) & $\mathbf{18.0}$ {\scriptsize $(\pm\,1.7)$} & $0.090$ {\scriptsize $(\pm\,0.030)$} & $35.1$ {\scriptsize $(\pm\,3.1)$} & $\mathbf{0.211}$ {\scriptsize $(\pm\,0.023)$} & $75.3$ {\scriptsize $(\pm\,6.8)$} \\
\bottomrule
\end{tabular}%
}
\end{table*}

\vspace{-0.1cm}

\textbf{Data \& Evaluation.} We use five-task streams on Fashion-MNIST (FMNIST) \citep{xiao2017fashion} and CIFAR-10 \citep{krizhevsky2009learning} with two disjoint classes per task, plus a three-task ImageNet-100 stream split $34/33/33$ \citep{deng2009imagenet}. FMNIST and CIFAR-10 are modeled unconditionally in pixel space. To ensure quality, ImageNet-100 is modeled class-conditionally in frozen Stable-Diffusion \citep{rombach2022high} latent space and decoded before evaluation. After the final update, we measure fit and retention jointly as Fréchet Inception Distance (FID) from $10{,}000$ generated samples against pooled held-out images from all seen classes. For FMNIST and CIFAR-10, a fixed classifier estimates generated class proportions, reported as $\mathrm{KL}(\pi_{\mathrm{data}}\Vert \pi_{\mathrm{gen}})$ (KLD), which scores retention alone. This is omitted for ImageNet-100, as class proportions are prescribed at sampling.

\vspace{-0.1cm}

\textbf{Baselines.} Joint training on all tasks gives an upper bound, and naive sequential fine-tuning measures forgetting. We run fine-tuning separately for diffusion and flow matching, ruling out the generative objective as the source of any difference. We compare against diffusion-based generative replay and generative distillation \citep{masip2025continual}, which regresses the frozen teacher's noise prediction on renoised replay samples. We reproduce both under our protocol for ImageNet-100, since they were not evaluated there originally. For flow-based CL, generative and trajectory replay are each compared using numerical integration of the teacher flow or a MeanFlow approximation (\cref{sec: theory_in_practice}) with one or four flow-map steps, under both equal-task and adaptive weighting (selecting $\lambda$ by minimizing validation FID across seen tasks). In \cref{tab:image_cl_lora_gft}, we add LoRA and gradual fine-tuning as adapter and regularization baselines.\looseness=-1

\vspace{-0.1cm}

\textbf{Results.} Trajectory replay is the strongest method family, achieving the best non-joint FID on all three datasets and the best KLD on both unconditional benchmarks, outperforming generative replay and diffusion-based baselines (\cref{tab:image_cl_all}). For trajectory replay, adaptive weighting generally yields the lowest FID, whereas equal-task weighting remains competitive and consistently improves class-mass recovery. On FMNIST and CIFAR-10, the weights selected for trajectory replay track the equal-task rule, while those for generative replay settle below it, consistent with the criterion discounting its lower-fidelity replayed samples (\cref{fig: app_fig8}). Trajectory replay also remains robust when ODE integration is replaced by one or four flow-map steps, while generative replay often degrades under the same approximation (\cref{fig: app_fig9}). The runtime advantage of MeanFlow distillation emerges only as teacher steps grow (\cref{fig: app_fig14}). Trajectory replay is cheaper end-to-end too, by $1.22\times$ to $1.65\times$, with the margin growing with generation cost (\cref{fig: app_fig24}).\looseness=-1

\subsection{Application: perturbation modeling in single-cell transcriptomics}
\label{sec:results-transcriptomics}

\textbf{Task.} We apply trajectory replay to perturbation modeling, where a generative model of cellular responses to treatment is updated as screens arrive in new cellular contexts. In each screen, data arrive as tuples $(\bm{x}, y, c)$ of a high-dimensional gene expression vector $\bm{x}$, a perturbation label $y$ (a molecule or a gene suppression), and a context $c$ (a cell line or a donor), and a generative model approximates the treatment response via $p(\bm{x} \mid c,y)$ \citep{palla2026scalable}. The initial task trains on a subset of contexts, and updates bring the same perturbations to novel ones. Success requires (\textbf{I}) preserving perturbation modeling on earlier contexts, (\textbf{II}) fitting the incoming ones, and (\textbf{III}) generalizing to unseen context-perturbation combinations from the incoming data.

\begin{figure*}[t]
    \centering    
    \includegraphics[width=0.90\textwidth]{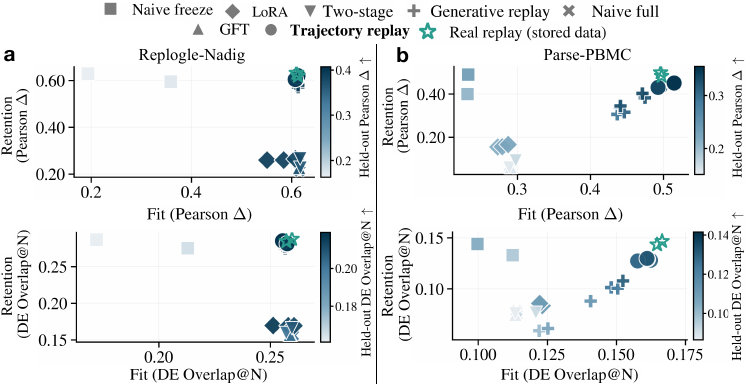}
    \caption{Final-checkpoint fit versus retention on (\textbf{a}) Replogle-Nadig and (\textbf{b}) Parse-PBMC, for Pearson $\Delta$ (top) and DE Overlap@$N$ (bottom). Marker shape denotes the method family and each point one configuration within it, as listed in \cref{tab:replogle_cl_full,tab:pbmc_cl_full}. Color denotes held-out performance.}
    \label{fig: tx_figure}
\end{figure*}

\textbf{Data.} We use two datasets, each organized by context into a task stream (\cref{app: perturbation_tx_details}): Replogle-Nadig \citep{Replogle2022, Nadig2025}, CRISPRi knockdowns of $2{,}023$ essential genes, split by cell line into two tasks ($T_0$: Jurkat/K562/RPE1, $T_1$: HepG2), and Parse-PBMC \citep{Oesinghaus2025.12.12.693897}, $90$ cytokine stimulations of peripheral blood mononuclear cells, split by donor into five ($T_0$: eight donors, $T_1$--$T_4$: D1, D4, D9, D12). Models operate on principal component analysis (PCA) cell states fitted on $T_0$ ($50$ and $100$ components, respectively), mapped back to gene space by frozen per-task affine decoders that invert the PCA and recalibrate each gene's scale for that task.

\textbf{Evaluation.} We evaluate per stratum, a single context-perturbation combination. At the final checkpoint, we report retention on previously learned strata, fit on newly introduced ones, and generalization on held-out perturbations, following the few-shot split of \citet{Adduri2026}, where the held-out set covers $60\%$ of the incoming context's perturbations, all seen in every other context, and is never used for training or $\lambda$ selection. We score $500$ generated cells per stratum with two metrics: macro-averaged \textit{Pearson $\Delta$} (correlation between predicted and observed control-subtracted effects) and \textit{DE Overlap@$N$} (recovery of the $N$ observed significant differentially expressed genes relative to control among the top-$N$ predicted, $N$ set per perturbation).

\textbf{Baselines.} We compare naive fine-tuning, full or with the velocity field or embeddings frozen in combination, gradual fine-tuning (GFT) \citep{thorkelsdottir2026gradual}, LoRA \citep{hu2022lora}, two-stage adaptation, which first fits the new context embedding alone and then unfreezes the full model, and generative and trajectory replay. Real replay stores past tasks as a privileged reference.

\textbf{Results.} We jointly examine fit, retention, and held-out performance (\cref{fig: tx_figure}; full tables in \cref{tab:replogle_cl_full,tab:pbmc_cl_full}). Replay is the only family reaching both axes: freezing retains without fitting, while LoRA, GFT, two-stage, and full fine-tuning either only fit or neither. Freezing can still top the retention column because the historical task lives in the velocity field rather than in the conditioning embeddings, so a frozen field preserves it by refusing to adapt (\cref{tab:replogle_cl_full,tab:pbmc_cl_full}). Trajectory replay edges ahead of generative replay on the single-update Replogle-Nadig stream and clearly surpasses it across the five sequential tasks of Parse-PBMC, approaching real replay without storing cells. Point color confirms that held-out performance follows the upper-right frontier rather than either axis alone, since an unseen combination needs both the perturbation response carried from old contexts and the basal structure of the new one. As in the image streams, replacing ODE integration with one or four flow-map steps leaves trajectory replay at parity while generative replay degrades (\cref{fig: app_fig17}).

\subsection{Application: Continual phenotypic modeling in microscopy}
\label{sec:results-microscopy}

\textbf{Task.} We extend \cref{sec:results-transcriptomics} to microscopy, where $\bm{x}$ is an image and the model captures the morphological phenotype a perturbation induces. Image-based profiling is a standard readout in drug discovery, but experiments acquired at different times differ in staining and imaging conditions, producing batch effects that complicate comparisons across experiments. The RxRx1 dataset \citep{Sypetkowski_2023_CVPR} images U2OS cells under gene knockdowns across experimental batches acquired weeks apart, and we treat each batch as one task. Since the perturbation vocabulary is identical across tasks, the shift is purely technical, so any loss of performance on earlier batches reflects forgetting of their acquisition conditions.\looseness=-1

\vspace{-0.1cm}

\textbf{Setup} (details in \cref{app: perturbation_px_details}). We stream three U2OS experiments (01 $\rightarrow$ 02 $\rightarrow$ 03) over 1,071 perturbations, conditioning an ablated diffusion model (ADM) U-Net \citep{dhariwal2021diffusion} on perturbation and experiment identity, with six-channel $96\times96$ image crops modeled natively. Baselines follow \cref{sec:results-image-cl}, adding real replay as a privileged reference that stores past images. The MeanFlow backends are run for trajectory replay under equal-task weighting only, since the Jacobian-vector product in the MeanFlow residual becomes costly at this resolution. With only a few dozen test images per perturbation, Fréchet distances cannot be estimated per condition, so we generate one image per test condition, matching the perturbation composition of generated and reference sets, and report FID against pooled test images of all seen experiments after each update \citep{zhang2025cellflux}. We also report a Fréchet distance in DINOv2 \citep{oquab2024dinov} features (FDD), with full results in \cref{app:results-microscopy} and \cref{tab:px_cl_final}.

\vspace{-0.1cm}

\begin{wrapfigure}[22]{r}{0.60\linewidth}
\vspace{-13pt}
\centering
\includegraphics[width=\linewidth]{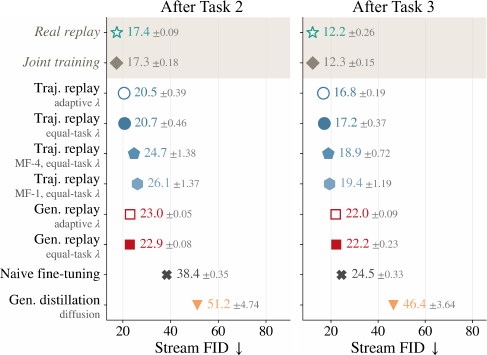}
\caption{FID after the first update, once U2OS-02 has been learned (After Task 2), and after the second, once all three experiments have been learned (After Task 3), computed on features pooled over all seen experiments (mean $\pm$ SE, three seeds). Gray rows are upper bounds on the performance.}
\label{fig: px_figure}
\end{wrapfigure}

\textbf{Results} (\cref{fig: px_figure}). Trajectory replay with ODE integration improves over generative replay at each update, reaching $17.2$ against $22.2$ pooled FID at the final checkpoint ($16.8$ with adaptive weighting), within a few points of real replay and joint training without storing images. The gain is concentrated on the oldest batch, where generative replay degrades to about $71$ FID while trajectory replay stays near $40$ (\cref{fig: app_fig21}), with the same ordering in DINOv2 features. Generated crops also reproduce the staining and morphology of their own batch (\cref{fig: app_fig23}). The MeanFlow variants rank just behind the ODE backend on the final stream, still ahead of generative replay. Generative diffusion distillation \citep{masip2025continual}, the strongest published no-storage baseline on natural-image streams, is the weakest replay method here, worse than naive fine-tuning.
\section{Related works}
\textbf{Lifelong learning for generative models.} CL methods for generative models fall into three broad families: architectural updates, which allocate task-specific parameters \citep{smith2024continual}, regularization, which constrains how far parameters may drift from those learned on earlier tasks \citep{dong2024how, kirkpatrick2017overcoming}, and replay, which rehearses past tasks during each update \citep{Sun2024create, wang2024comprehensive, guo2025continual}. Replay has been used to curb forgetting both in sequential classification \citep{shin2017continual, wu2018memory, ye2020learning} and in continuous-time generative frameworks \citep{gao2023ddgr, skier2025joint, liao2025continual}. It has also been moved from the input space to intermediate representations, where generating features is easier than generating pixels \citep{liu2020generative}, and we discuss the relation of this line of work to our setting in \cref{sec: relation_feature_replay}. As in our setting, several works treat generation not only as a source of replay samples but as the task to be preserved itself \citep{simone2025continualflow, zajkac2023exploring, liu2026ccd}. Closest to our work is \citet{masip2025continual}, who use a diffusion model trained on historical data as a teacher for a sequential noise predictor, supervised on noise-augmented replay samples. Their replay still generates complete samples before renoising them, whereas trajectory replay stops the historical trajectory at an intermediate time and never forms a finished sample.
\vspace{-0.1cm}

\textbf{Flow matching and density mixtures.} Trajectory replay builds on flow matching \citep{lipman2023flow, tong2024improving} and on its distillation into few-step samplers \citep{geng2026mean, boffi2024flow}. It also relates to work on controlling likelihoods in continuous time \citep{albergo2025stochastic, su2025flow, romer2026uncertainty} and on sampling from mixtures of generative models \citep{liu2022rectifiedflowmarginalpreserving}, both of which have mostly been studied as inference-time control over frozen models \citep{skreta2025the}, whereas we use these ideas to train a model across sequential updates. To our knowledge, mixtures of probability paths have not previously been used to formulate continual updates of continuous-time generative models. Most related within this line is \citet{thorkelsdottir2026gradual}, who anneal the historical model toward the incoming task, while trajectory replay preserves old tasks while integrating new ones.\looseness=-1
\section{Discussion}
\textbf{Conclusion.} We introduced trajectory replay, which replaces sample-level generative replay with velocity-field regression along the historical model's own trajectories. Its optimum is the density-weighted average of the two task velocities, generating the mixture of their probability paths, so the stability-plasticity coefficient sets what the update retains rather than only trading it off, and the objective recovers a weighted forward KL divergence up to time reweighting. A matching rule for the mean velocity carries this optimum to the flow-map level. Because replay stops at intermediate times, it halves the expected simulation cost and makes replay more robust to discretization error. Empirically, trajectory replay outperforms replay and distillation baselines on class-incremental image generation, stays robust under few-step replay, and extends perturbation models to incremental experimental measurements.\looseness=-1

\vspace{-0.1cm}

\textbf{Limitations.} Trajectory replay still involves partial simulation under a historical model, and its adaptive version, though favorable on image generation, requires a training run per candidate weight at every update. The mean velocity aids sampling but requires a Jacobian-vector product whose memory cost grows with the state, which restricted its use in our microscopy experiments. Partial replay limits how solver error enters each update, but not the approximation error of the historical model itself, which every update inherits and which may compound over long task sequences, echoing forgetting bounds for continually fine-tuned Large Language Models (LLMs), while our streams span at most five tasks. Our reliance on flow matching also leaves extension to autoregressive models open. Finally, future work will explore CL updates of more efficient distillation models.

\newpage
\section*{Acknowledgments}
The authors acknowledge Giulia Lanzillotta for the fruitful discussions and feedback.

\bibliography{valence_bib}
\bibliographystyle{valence}

\newpage
\appendix

\section{Theoretical integration}\label{sec: theoretical_integration}
\subsection{Theoretical integration on flow matching}\label{sec: flow_matching_basics}
\paragraph{Flows and the continuity equation.} Flow matching \citep{lipman2023flow, tong2024improving} regresses a marginal vector field $\{b_t(\bm{x})\}_{t\in[0,1]}$ that generates a probability path $\{p_t(\bm{x})\}_{t\in[0,1]}$ through the continuity equation: 
\begin{equation}\label{eq: app_continuity}
    \frac{\partial}{\partial t}p_t(\bm{x}) = -\nabla \cdot (p_t(\bm{x})b_t(\bm{x})) \, .
\end{equation}
For sufficiently regular $b_t$, the solution map of \cref{eq: ode_solution} that sends an initial state to its state at time $t$ is a diffeomorphism, called the \textit{flow}, that maps noise vectors to samples from $p_t$. Let $\phi_t: \mathcal{X}\rightarrow \mathcal{X}$, for $t\in[0,1]$, denote this map, with $\phi_0 := \mathrm{id}$ and $\phi_t(\bm{x}_0) := \bm{x}_t$. The mass conservation expressed by \cref{eq: app_continuity} then makes the flow push the prior forward along the path,
\begin{equation}\label{eq: app_pushforward}
    p_t = (\phi_t)_\sharp\, p_0 \, ,
\end{equation}
that is, $\bm{x}_0 \sim p_0$ implies $\bm{x}_t \sim p_t$ for all $t\in[0,1]$. 

\paragraph{Gaussian conditional flows.} It is common practice to use conditional Gaussian paths to evaluate the conditional flow matching objective in \cref{eq: c_flow_matching_obj}. Recall the definition of Gaussian paths:
\begin{equation}\label{eq: app_gaussian_path}
    p_{t|1}(\bm{x} \mid \bm{x}_1) := \mathcal{N}(\bm{x}; \alpha_t \bm{x}_1, \sigma_t^2 \mathbb{I}_d)\, ,
\end{equation}
with $\bm{x}_1\sim q$ and, for the linear schedule, $\alpha_t:=t$ and $\sigma_t:=1-t$, though other variants exist \citep{albergo2025stochastic, tong2024simulation}. Writing $\bm{x}=\alpha_t\bm{x}_1+\sigma_t\bm{x}_0$ with $\bm{x}_0\sim p_0$, differentiating in $t$ gives the associated conditional velocity
\begin{equation}\label{eq: app_conditional_velocity}
    b_{t|1}(\bm{x}\mid\bm{x}_1) := \dot{\alpha}_t\bm{x}_1 + \dot{\sigma}_t\,\frac{\bm{x}-\alpha_t\bm{x}_1}{\sigma_t} \, ,
\end{equation}
which reduces to $\bm{x}_1-\bm{x}_0$ under the linear schedule \citep{tong2024improving}. Conditional flow matching hinges on the fact that, in expectation, regressing $b_{t|1}$ along $p_{t|1}$ recovers the velocity field of the marginal path
\begin{equation}\label{eq: marginal_path}
    p_t(\bm{x}) = \int p_{t|1}(\bm{x}\mid\bm{x}_1)\,q(\bm{x}_1)\, \mathrm{d}\bm{x}_1 \, ,
\end{equation}
which is the posterior expectation of the conditional one \citep{lipman2023flow},
\begin{equation}
    b_t(\bm{x}) = \int b_{t|1}(\bm{x}\mid\bm{x}_1)\,p_{1|t}(\bm{x}_1\mid\bm{x})\,\mathrm{d}\bm{x}_1 \, ,
\end{equation}
where $p_{1|t}(\bm{x}_1\mid\bm{x})$ is the posterior over clean data given the noisy state.

\textbf{Relationship with the density score.} A crucial property of flow matching with Gaussian conditional paths is that the marginal vector field is affine in the marginal probability path's score, with coefficients dependent on the scheduling parameters \citep{zheng2023guided}: 
\begin{equation}\label{eq: score_vf_relationship}
    b_t(\bm{x}) = a_t\,\bm{x} + c_t\, \nabla_{\bm{x}}\log p_t(\bm{x}) \, ,
\end{equation}
with $a_t := \dot{\alpha}_t/\alpha_t$ and $c_t := \frac{\sigma_t}{\alpha_t}\big(\dot{\alpha}_t\sigma_t - \alpha_t\dot{\sigma}_t\big)$. Wherever $\alpha_t,\sigma_t>0$ and $\dot{\alpha}_t\sigma_t-\alpha_t\dot{\sigma}_t\neq0$, so that $c_t\neq0$, the relation is invertible, and learning the marginal velocity field and learning the marginal score are equivalent tasks. We use this equivalence in the proof of \cref{thm: rho_minim_kl}, whose assumptions include that condition on $(0,1)$.

\subsection{Mean flow basics}\label{sec: mean_flows_basics}
\paragraph{From instantaneous to average velocity.} Recall the definition of mean velocity in \cref{sec: ct_generative_flows}: 
\begin{equation}\label{eq: meanflow_app}
    V_{s,t}(\bm{x}_s) := \frac{1}{t-s}\int_s^t b_{\tau}(\bm{x}_{\tau})\mathrm{d}\tau \, , 
\end{equation}
with $\bm{x}_s \sim p_s$ and $0 \le s < t \le 1$. To distinguish our notation from \citet{geng2026mean}, we set noise at $t=0$ and data at $t=1$, so $s$ is the lower limit of integration. Multiplying \cref{eq: meanflow_app} by $(t-s)$ and differentiating both sides with respect to $s$, with $t$ held fixed and along the trajectory through $\bm{x}_s$, yields the MeanFlow identity:
\begin{align}
    \frac{\mathrm{d}}{\mathrm{d}s}\Big[(t-s)\,V_{s,t}(\bm{x}_s)\Big]
    &= \frac{\mathrm{d}}{\mathrm{d}s}\int_s^t b_{\tau}(\bm{x}_{\tau})\,\mathrm{d}\tau \nonumber\\
    -V_{s,t}(\bm{x}_s) + (t-s)\frac{\mathrm{d}}{\mathrm{d}s}V_{s,t}(\bm{x}_s)
    &= -b_s(\bm{x}_s) \nonumber\\
    V_{s,t}(\bm{x}_s) &= b_s(\bm{x}_s) + (t-s)\frac{\mathrm{d}}{\mathrm{d}s}V_{s,t}(\bm{x}_s) \, ,
    \label{eq: app_mean_flow_identity}
\end{align}
where the left-hand side follows from the product rule and the right-hand side from the fundamental theorem of calculus. Note that the sign in \cref{eq: app_mean_flow_identity} is flipped compared to \citet{geng2026mean}, as we differentiate with respect to the earlier time, which is the lower limit of integration. Since $\bm{x}_s$ varies with $s$ along the trajectory, the derivative is total and expands as
\begin{equation}\label{eq: app_total_derivative}
    \frac{\mathrm{d}}{\mathrm{d}s}V_{s,t}(\bm{x}_s) = \partial_s V_{s,t}(\bm{x}_s) + \nabla_{\bm{x}_s} V_{s,t}(\bm{x}_s)\, b_s(\bm{x}_s) \, .
\end{equation}

\paragraph{Boundary behavior and flow maps.} As stated, \cref{eq: meanflow_app} is undefined on the diagonal, but for continuous $b$ the average converges to the instantaneous velocity as the interval closes \citep{geng2026mean}, so we extend $V$ to $s=t$ by
\begin{equation}\label{eq: boundary_condition}
    V_{t,t}(\bm{x}_t) := \lim_{s\uparrow t} V_{s,t}(\bm{x}_s) = b_t(\bm{x}_t) \, .
\end{equation}
A parameterized field $V^\theta$ is defined on all of $0\le s\le t\le1$ and does not satisfy this automatically, so it is usually enforced with heuristics such as oversampling the time diagonal during distillation.

As mentioned in \cref{sec: ct_generative_flows}, the mean velocity $V_{s,t}$ defines a flow map: 
\begin{equation}\label{eq: app_flow_map}
    \bm{X}_{s,t}(\bm{x}_s) := \bm{x}_s + (t-s)V_{s,t}(\bm{x}_s) \, ,
\end{equation}
which satisfies $\bm{X}_{t,t}=\mathrm{id}$ by construction, since the factor $(t-s)$ vanishes on the diagonal. Under the boundary condition of \cref{eq: boundary_condition}, and for sufficiently regular $b$, the flow map is characterized equivalently by the Lagrangian condition, the Eulerian condition, or the semigroup property \citep{boffi2024flow}, so any one of them can serve as a training target. We use the semigroup property, $\bm{X}_{s,t} = \bm{X}_{r,t}\circ\bm{X}_{s,r}$ for any $s<r<t$, which is what licenses few-step sampling in \cref{sec: theory_in_practice}, and the Eulerian condition, which recovers the MeanFlow identity as follows.

Following \citet{dieleman2026flowmaps}, the Eulerian condition states that the flow map is invariant along the trajectory through its starting point,
\begin{equation}\label{eq: app_eulerian}
    \frac{\partial}{\partial s}\bm{X}_{s,t}(\bm{x}_s) + \nabla_{\bm{x}_s}\bm{X}_{s,t}(\bm{x}_s)\, b_s(\bm{x}_s) = \bm{0} \, .
\end{equation}
Substituting \cref{eq: app_flow_map} into \cref{eq: app_eulerian} gives
\begin{align}
    \bm{0} &= -V_{s,t}(\bm{x}_s) + (t-s)\,\partial_s V_{s,t}(\bm{x}_s) + \Big(\mathbb{I}_d + (t-s)\nabla_{\bm{x}_s} V_{s,t}(\bm{x}_s)\Big) b_s(\bm{x}_s) \nonumber\\
    &= -V_{s,t}(\bm{x}_s) + b_s(\bm{x}_s) + (t-s)\Big(\partial_s V_{s,t}(\bm{x}_s) + \nabla_{\bm{x}_s} V_{s,t}(\bm{x}_s)\, b_s(\bm{x}_s)\Big) \, ,
\end{align}
which rearranges to \cref{eq: app_mean_flow_identity}. The two formulations are therefore equivalent: the MeanFlow identity is the Eulerian condition written in terms of the average velocity rather than the map itself.

\subsection{Stop-gradient training and marginal-from-conditional learning}\label{sec: semigradient}
\paragraph{Semigradients.} Recall the residual formula used to learn the mean velocity in \cref{eq: consistency_residual}. The gradient flows through the parameters $\theta$ of the neural mean velocity approximation only through the first term, treating the target inside the stop-gradient as a constant and avoiding higher-order derivatives. Following \citet{dieleman2026flowmaps}, writing $\mathcal{R}$ for the residual and $\mathcal{L}(\theta) := \mathbb{E}[\|\mathcal{R}\|_2^2]$, the direction followed is
\begin{equation}\label{eq: app_semigradient}
    \widetilde{G}_\theta = 2\,\mathbb{E}\Big[\big(\partial_\theta V^\theta_{s,t}(\bm{x}_s)\big)^{\!\top}\mathcal{R}_{s,t}(V^\theta, b)(\bm{x}_s)\Big] \, ,
\end{equation}
rather than $\nabla_\theta\mathcal{L} = 2\,\mathbb{E}[(\partial_\theta\mathcal{R})^{\!\top}\mathcal{R}]$, since $V^\theta_{s,t}(\bm{x}_s)$ is the only $\theta$-dependent term outside the stop-gradient. We refer to $\widetilde{G}_\theta$ as a \textit{semigradient}: it is not, in general, the gradient of any objective, so the field it converges to is characterized as a \textit{fixed point}, at which $\widetilde{G}_\theta$ vanishes, rather than as a minimizer. This is the sense in which \cref{prop: mean_flow_minimizer} is stated.

\paragraph{Marginal from conditional.} Write the target of \cref{eq: consistency_residual} as $T_{s,t}(V,\bm{u})(\bm{x}) := \bm{c}(\bm{x}) + M(\bm{x})\,\bm{u}$, with $\bm{c} := (t-s)\,\partial_s V_{s,t}$ and $M := \mathbb{I}_d + (t-s)\nabla_{\bm{x}} V_{s,t}$. The residual, and hence the expression inside the expectation in \cref{eq: app_semigradient}, is therefore affine in the velocity slot $\bm{u}$. Conditioning on $\bm{x}_s$ and applying the tower property to $b_s(\bm{x}_s) = \mathbb{E}[b_{s|1}(\bm{x}_s\mid\bm{x}_1)\mid\bm{x}_s]$ then leaves $\widetilde{G}_\theta$ unchanged under substitution of the conditional velocity for the marginal one, so the mean velocity can be trained simulation-free from data \citep{geng2026mean, dieleman2026flowmaps}. The argument relies on two ingredients. The first is the stop-gradient, without which the expression is quadratic in $\bm{u}$ and the conditional variance of $b_{s|1}$ given $\bm{x}_s$ would shift the fixed point rather than only adding variance. The second is evaluating \cref{eq: app_mean_flow_identity} at $\bm{x}_s$, so that $\bm{c}$, $M$ and $\partial_\theta V^\theta_{s,t}$ depend on $\bm{x}_s$ alone and pass through the conditional expectation.
\subsection{Replay as mixture modeling}\label{sec: replay_as_mixture_modeling}
Generative replay can be seen as modeling a mixture between historical and new data densities, where the mixture coefficient $\lambda$ represents the rate at which samples are drawn from a generative model over previous tasks. 

Given a flow model over historical data $\{p_t^{\theta}\}_{t \in [0,1]}$ and an incoming task distribution $q'$, generative replay samples batches from the following distribution: 
\begin{equation}\label{eq: q_batch}
    q_{\text{batch}} := \lambda p_1^{\theta} + (1-\lambda) q' \, .
\end{equation}
Note that $p_1^{\theta}$ approximates the historical density $q$, hence $q_{\text{batch}}$ is an approximation of \cref{eq: mixture}. The resulting mixed batches are then used to train a generative model with the conditional flow matching loss in \cref{eq: c_flow_matching_obj}. There, the data density enters only as the law of $\bm{x}_1$, so training on replay-augmented batches amounts to minimizing the same objective with $q$ replaced by $q_{\text{batch}}$. By the marginal-from-conditional construction in \cref{eq: marginal_path}, its minimizer is the marginal velocity generating
\begin{equation}
    \rho_t(\bm{x}) := \int p_{t|1}(\bm{x}\mid\bm{x}_1)\, q_{\text{batch}}(\bm{x}_1)\, \mathrm{d}\bm{x}_1 \, ,
\end{equation}
whose terminal density is $\rho_1 = q_{\text{batch}}$, since the schedule satisfies $\alpha_1=1$ and $\sigma_1=0$, so that $p_{1|1}(\cdot\mid\bm{x}_1)$ is a point mass at $\bm{x}_1$.

Hence, up to an approximation error of $q$ by $p_1^{\theta}$, generative replay optimizes a flow matching model over \cref{eq: mixture}.
\subsection{Static KL as mixture minimizer}\label{sec: static_kl_as_mixture_minimizer}
We restate the formal optimization objective over nonnegative integrable functions subject to normalization:
\begin{equation}\label{eq: kl_minimizer}
\begin{aligned}
    \tilde{q} = \argmin_{\rho \geq 0} & \left[ \lambda D_{\mathrm{KL}}(q \parallel \rho) + (1-\lambda)D_{\mathrm{KL}}(q' \parallel \rho)  \right] \, , \\
    \text{s.t.} \quad & \int \rho(\bm{x}) \mathrm{d}\bm{x} = 1 \, .
\end{aligned}
\end{equation}
Expanding the Kullback-Leibler divergence definitions yields the following objective functional:
\begin{align}
    F(\rho) &= \lambda \int q(\bm{x})\log \frac{q(\bm{x})}{\rho(\bm{x})}\mathrm{d}\bm{x} + (1-\lambda)\int q'(\bm{x})\log \frac{q'(\bm{x})}{\rho(\bm{x})}\mathrm{d}\bm{x} \nonumber \\
    &= \lambda\int q(\bm{x})\log q(\bm{x})\mathrm{d}\bm{x} - \lambda\int q(\bm{x})\log \rho(\bm{x})\mathrm{d}\bm{x} \nonumber \\ 
    & \qquad + (1-\lambda)\int q'(\bm{x})\log q'(\bm{x})\mathrm{d}\bm{x} - (1-\lambda)\int q'(\bm{x})\log \rho(\bm{x})\mathrm{d}\bm{x} \, .
\end{align}
Incorporating the density normalization constraint via a Lagrange multiplier $\gamma \in \mathbb{R}$, we construct the Lagrangian: 
\begin{equation}\label{eq: lagrangian}
    \Lambda(\rho, \gamma) = F(\rho) + \gamma \left[ \int \rho(\bm{x})\mathrm{d}\bm{x} - 1 \right] \, .
\end{equation}
To find the optimal density, we compute the functional derivative of \cref{eq: lagrangian} with respect to $\rho(\bm{x})$ and set it to zero:
\begin{equation}\label{eq: functional_derivative}
    \frac{\delta \Lambda(\rho, \gamma)}{\delta \rho(\bm{x})} = -\lambda \frac{q(\bm{x})}{\rho(\bm{x})} - (1-\lambda) \frac{q'(\bm{x})}{\rho(\bm{x})} + \gamma = 0 \, .
\end{equation}
Solving for $\rho(\bm{x})$ yields:
\begin{equation}\label{eq: rho_solved}
     \rho(\bm{x}) = \frac{\lambda q(\bm{x}) + (1-\lambda) q'(\bm{x})}{\gamma} \, .
\end{equation}
Integrating both sides over the domain ensures that the normalization constraint is satisfied, which requires:
\begin{align}
    \gamma &= \int \left[\lambda q(\bm{x}) + (1-\lambda)q'(\bm{x})\right]\mathrm{d}\bm{x} \nonumber \\
    &= \lambda \int q(\bm{x})\mathrm{d}\bm{x} + (1-\lambda)\int q'(\bm{x})\mathrm{d}\bm{x} \nonumber \\
    &= 1 \, ,
\end{align}
since $q(\bm{x})$ and $q'(\bm{x})$ are valid probability density functions that integrate to unity. Substituting $\gamma = 1$ back into \cref{eq: rho_solved} gives the mixture $\tilde{q}(\bm{x}) = \lambda q(\bm{x}) + (1-\lambda) q'(\bm{x})$, which is nonnegative and integrates to unity, so it is feasible.

This critical point is the unique global minimizer: discarding terms independent of the optimized $\rho$, $F(\rho) = -\int \tilde{q}(\bm{x})\log\rho(\bm{x})\,\mathrm{d}\bm{x} + \mathrm{const} = D_{\mathrm{KL}}(\tilde{q}\parallel\rho) + \mathrm{const}'$, which is nonnegative and vanishes only at $\rho=\tilde{q}$ almost everywhere, completing the proof.

\subsection{\texorpdfstring{Proof of \cref{prop: objective_minimizer}}{Proof of Proposition 3.1}}\label{sec: proof_unique_minimizer}

\begin{restated}{Proposition~\ref{prop: objective_minimizer}}
Let $\{p_t\}_{t\in[0,1]}$ and $\{p'_t\}_{t\in[0,1]}$, with $p_t,p'_t:\mathbb{R}^d\rightarrow\mathbb{R}_{\geq0}$, be density paths generated through the continuity equation in \cref{sec: ct_generative_flows} by velocity fields $b,b'\in C^0([0,1],(C^1(\mathbb{R}^d))^d)$, with shared prior $p_0=p'_0:=\mathcal{N}(\bm{0}_d,\mathbb{I}_d)$ and terminal densities $p_1:=q$, $p'_1:=q'$. Fix $\lambda\in(0,1)$ and assume $\lambda p_t(\bm{x})+(1-\lambda)p'_t(\bm{x})>0$ for all $t\in[0,1]$ and $\bm{x}\in\mathcal{X}$. Then $\mathcal{J}(\cdot\,;b,b')$ in \cref{eq: mixture_objective} admits 
\begin{equation*}
    v_t^{\star}(\bm{x}) := \frac{\lambda\, p_t(\bm{x})\, b_t(\bm{x}) + (1-\lambda)\, p'_t(\bm{x})\, b'_t(\bm{x})}{\lambda\, p_t(\bm{x}) + (1-\lambda)\, p'_t(\bm{x})} \, ,
\end{equation*}
as unique minimizer, which generates the mixture path $\rho_t^\star := \lambda p_t+(1-\lambda)p'_t$, satisfying $\partial_t\rho_t^\star+\nabla_{\bm{x}}\cdot(v_t^\star\rho_t^\star)=0$ with $\rho_0^\star=\mathcal{N}(\bm{0}_d,\mathbb{I}_d)$ and $\rho_1^\star=\tilde{q}$ (\cref{eq: mixture}).
\end{restated}

The shared Gaussian conditional kernel required by the standing assumptions of \cref{sec: ct_framework} is not used here: the minimizer and the continuity argument need only positivity and the stated regularity, so the density-weighted mixture holds for arbitrary probability paths. A shared kernel becomes necessary only in \cref{thm: rho_minim_kl}, where the score-to-velocity inversion relies on it.

\begin{proof}
\textbf{Minimizer.} As a first step, we rewrite the objective in \cref{eq: mixture_objective} by expanding the expectation and grouping the integrals. 
\begin{equation}\label{eq: app_grouped_objective}
    \mathcal{J}(v; b, b') = \int_0^1 \int \Big(\lambda\,\|v_t(\bm{x})-b_t(\bm{x})\|_2^2\, p_t(\bm{x}) + (1-\lambda)\,\|v_t(\bm{x})-b'_t(\bm{x})\|_2^2 \, p_t'(\bm{x}) \Big)\,\mathrm{d}\bm{x} \, \mathrm{d}t \, .
\end{equation}
By the calculus of variations, the optimization task in \cref{eq: mixture_objective} becomes a pointwise minimization, as the objective does not contain any derivative of the minimizer $v$. It therefore suffices to minimize the integrand of \cref{eq: app_grouped_objective} over $\bm{u}\in\mathbb{R}^d$ at each fixed $(t,\bm{x})$. Setting its gradient to zero,
\begin{equation}\label{eq: app_pointwise_stationarity}
    2\lambda\, p_t(\bm{x})\big(\bm{u}-b_t(\bm{x})\big) + 2(1-\lambda)\, p'_t(\bm{x})\big(\bm{u}-b'_t(\bm{x})\big) = \bm{0} \, ,
\end{equation}
and collecting the terms in $\bm{u}$ gives $\big(\lambda p_t + (1-\lambda)p'_t\big)\bm{u} = \lambda p_t b_t + (1-\lambda) p'_t b'_t$, where all quantities are evaluated at $(t,\bm{x})$. The bracketed factor is strictly positive by assumption, so dividing through recovers \cref{eq: dynamic_kl_minimizer}. The integrand is a positively weighted sum of quadratics, with Hessian $2(\lambda p_t + (1-\lambda)p'_t)\mathbb{I}_d \succ 0$, hence strictly convex in $\bm{u}$, so this critical point is its unique global minimizer. Since the pointwise minimizer is unique at every $(t,\bm{x})$, any other field attaining the same value of $\mathcal{J}$ must agree with $v^\star$ almost everywhere on $[0,1]\times\mathcal{X}$.

\textbf{Generating field.} We prove that \cref{eq: dynamic_kl_minimizer} is the generating field for $\rho_t^{\star}:=\lambda p_t + (1-\lambda) p_t'$. Differentiating $\rho^\star_t$ in time and applying the continuity equation to each task path gives
\begin{align}
    \frac{\partial}{\partial t}\rho_t^{\star}
    &= \lambda \frac{\partial}{\partial t} p_t + (1-\lambda)\frac{\partial}{\partial t} p_t' \nonumber\\
    &= -\lambda \nabla \cdot (p_t b_t) - (1-\lambda) \nabla \cdot (p_t' b_t') \nonumber\\
    &= -\nabla \cdot \big[ \lambda p_t b_t + (1-\lambda)p_t'b_t' \big] \nonumber\\
    &= -\nabla \cdot \left[ \rho^\star_t \,\frac{\lambda p_t b_t + (1-\lambda)p_t'b_t'}{\lambda p_t + (1-\lambda)p_t'} \right] \nonumber\\
    &= -\nabla \cdot \big( \rho^\star_t\, v^\star_t \big) \, ,
\end{align}
where the fourth equality multiplies and divides by $\rho^\star_t = \lambda p_t + (1-\lambda)p'_t$, which is allowed by the positivity assumption, and the last uses \cref{eq: dynamic_kl_minimizer}. Hence $\rho^\star_t$ satisfies the continuity equation with velocity $v^\star_t$. The boundary conditions follow from those of the task paths: $\rho^\star_0 = \lambda p_0 + (1-\lambda)p'_0 = \mathcal{N}(\bm{0}_d,\mathbb{I}_d)$ since the two priors coincide, and $\rho^\star_1 = \lambda q + (1-\lambda)q' = \tilde{q}$.
\end{proof}
\subsection{Closure of the marginal mixture path}\label{sec: gaussianity_marginal_mixture_path}
Let $p_t$ and $p'_t$ be marginal paths induced by the same fixed conditional Gaussian kernel $p_{t|1}(\bm{x}\mid\bm{x}_1) = \mathcal{N}(\bm{x};\alpha_t\bm{x}_1,\sigma_t^2\mathbb{I}_d)$ over the task densities $q$ and $q'$ respectively, as in \cref{eq: marginal_path}. By linearity of the integral, the mixture path of \cref{prop: objective_minimizer} satisfies
\begin{align}
    \rho^\star_t(\bm{x}) &= \lambda\, p_t(\bm{x}) + (1-\lambda)\, p'_t(\bm{x}) \nonumber\\
    &= \lambda \int p_{t|1}(\bm{x}\mid\bm{x}_1)\,q(\bm{x}_1)\,\mathrm{d}\bm{x}_1 + (1-\lambda)\int p_{t|1}(\bm{x}\mid\bm{x}_1)\,q'(\bm{x}_1)\,\mathrm{d}\bm{x}_1 \nonumber\\
    &= \int p_{t|1}(\bm{x}\mid\bm{x}_1)\,\big[\lambda\, q(\bm{x}_1) + (1-\lambda)\,q'(\bm{x}_1)\big]\,\mathrm{d}\bm{x}_1 \nonumber\\
    &= \int p_{t|1}(\bm{x}\mid\bm{x}_1)\,\tilde{q}(\bm{x}_1)\,\mathrm{d}\bm{x}_1 \, .
\end{align}
The mixture path is therefore itself a marginal path induced by the same kernel, with data distribution $\tilde{q}$ in place of $q$ or $q'$. Combined with \cref{prop: objective_minimizer}, which identifies $v^\star_t$ as its generating field, the pair $(\rho^\star_t, v^\star_t)$ is therefore an admissible candidate in \cref{thm: rho_minim_kl}.

Note that the shared kernel is essential: if the two task paths used different scheduling parameters, the two integrals could not be combined, and the mixture would in general not be a marginal path of either kernel, for any data distribution.

\subsection{\texorpdfstring{Proof of \cref{thm: rho_minim_kl}}{Proof of Theorem 3.2}}\label{sec: proof_rho_minim_kl}
\begin{restated}{Theorem~\ref{thm: rho_minim_kl}}
Under the standing assumptions, let $\rho_t$ be a candidate marginal path induced by the same kernel, with velocity $v_t$ and shared prior $\rho_0=p_0$, and define $\kappa_t := \alpha_t\big/\big(\sigma_t(\dot{\alpha}_t\sigma_t-\alpha_t\dot{\sigma}_t)\big) > 0$ on $(0,1)$. The weighted terminal divergence then satisfies
\begin{equation*}
\begin{aligned}
    &\lambda D_{\mathrm{KL}}(p_1\parallel\rho_1) + (1-\lambda)D_{\mathrm{KL}}(p'_1\parallel\rho_1) \\
    &\qquad = \int_0^1 \kappa_t \Big(\lambda\,\mathbb{E}_{p_t}\big[\|v_t(\bm{x})-b_t(\bm{x})\|_2^2\big] + (1-\lambda)\,\mathbb{E}_{p'_t}\big[\|v_t(\bm{x})-b'_t(\bm{x})\|_2^2\big]\Big)\mathrm{d}t \, .
\end{aligned}
\end{equation*}
\end{restated}

\begin{proof}
\textbf{Dynamic formulation of terminal} $D_{\text{KL}}$. \citet{albergo2025stochastic} show that the terminal KL objective can be rewritten as a dynamic product between score and vector field differences (Lemma 21). Exemplifying this on the first term in \cref{eq: kl_trajectory_equivalence} yields:
\begin{equation}\label{eq: app_kl_time_derivative}
    D_{\mathrm{KL}}(p_1 \parallel \rho_1) = \int_0^1 \mathbb{E}_{p_t}\Big[\big(\nabla_{\bm{x}}\log p_t(\bm{x}) - \nabla_{\bm{x}}\log \rho_t(\bm{x})\big) \cdot \big(b_t(\bm{x}) - v_t(\bm{x})\big)\Big]\,\mathrm{d}t \, ,
\end{equation}
where both continuity equations have been used, and the initial divergence drops out since the shared prior gives $D_{\mathrm{KL}}(p_0\parallel\rho_0)=0$ (see \citet{albergo2025stochastic} for the derivation).

\textbf{From scores to velocities.} As, by assumption, $\rho_t$, $p_t$ and $p_t'$ are marginal paths of the same Gaussian conditional kernel, the score can be rewritten in terms of the vector field and the scheduling parameters $\alpha_t$ and $\sigma_t$ \citep{zheng2023guided}. Rearranging terms in \cref{eq: score_vf_relationship} and, again, exemplifying the historical branch as in \cref{eq: app_kl_time_derivative} we obtain: 
\begin{equation}\label{eq: app_score_inversion}
    \nabla_{\bm{x}} \log p_t(\bm{x}) = \frac{\alpha_t b_t(\bm{x}) - \dot{\alpha}_t \bm{x}}{\sigma_t (\dot{\alpha}_t\sigma_t-\alpha_t \dot{\sigma}_t)} \, .
\end{equation}
Since $\rho_t$ is induced by the same kernel, the identical relation holds with $v_t$ in place of $b_t$. Subtracting the two, the affine terms cancel, and the score difference collapses to a velocity difference:
\begin{equation}\label{eq: app_score_to_velocity}
    \nabla_{\bm{x}}\log p_t(\bm{x}) - \nabla_{\bm{x}}\log \rho_t(\bm{x}) = \frac{\alpha_t}{\sigma_t(\dot{\alpha}_t\sigma_t-\alpha_t\dot{\sigma}_t)}\big(b_t(\bm{x}) - v_t(\bm{x})\big) = \kappa_t\big(b_t(\bm{x}) - v_t(\bm{x})\big) \, .
\end{equation}
This cancellation is where the restriction to a shared kernel is used: were the two paths governed by different schedules, the drift terms would not agree, and no such proportionality would hold.

\textbf{Weighted combination.} Substituting \cref{eq: app_score_to_velocity} into \cref{eq: app_kl_time_derivative} turns the inner product into a squared norm,
\begin{equation}\label{eq: app_kl_single_branch}
    D_{\mathrm{KL}}(p_1 \parallel \rho_1) = \int_0^1 \kappa_t\, \mathbb{E}_{p_t}\big[\|v_t(\bm{x}) - b_t(\bm{x})\|_2^2\big]\,\mathrm{d}t \, ,
\end{equation}
and the same argument applied to the incoming branch, with $p'_t$ and $b'_t$ in place of $p_t$ and $b_t$, gives the analogous identity with the same $\kappa_t$, since both branches share $\alpha_t$ and $\sigma_t$. Taking the $\lambda$-weighted combination of the two and collecting the integrals yields \cref{eq: kl_trajectory_equivalence}. For the linear schedule of \cref{sec: flow_matching_basics}, $\alpha_t=t$ and $\sigma_t=1-t$ give $\dot{\alpha}_t\sigma_t-\alpha_t\dot{\sigma}_t = (1-t)+t = 1$, so that $\kappa_t = t/(1-t)$.
\end{proof}

\subsection{\texorpdfstring{Proof of \cref{prop: mean_flow_minimizer}}{Proof of Proposition 3.3}}\label{sec: proof_mean_flow_minimizer}
\begin{restated}{Proposition~\ref{prop: mean_flow_minimizer}}
Under the standing assumptions, let $V$ be continuously differentiable in $s$ and $\bm{x}$ on $0\le s\le t\le 1$. Then the semigradient of $\mathcal{K}(\cdot\,;b,b')$ vanishes if and only if $V$ satisfies the MeanFlow identity for $v^\star$, that is $\mathcal{R}_{s,t}(V,v^\star)=\bm{0}$ for all $0\le s<t\le 1$, with $v^\star$ from \cref{eq: dynamic_kl_minimizer}, whose unique solution is the mixture mean velocity $V^\star_{s,t}$ of \cref{eq: mixture_mean_velocity}.
\end{restated}
\begin{proof}
Recall the extended formula for the mean flow continual update objective:
\begin{equation}\label{eq: app_mean_flow_mixture_objective}
    \mathcal{K}(V; b, b') := \int_0^1\!\!\int_s^1 \Big(\lambda\,\mathbb{E}_{\bm{x}\sim p_s}\big[\|\mathcal{R}_{s,t}(V,b)(\bm{x})\|_2^2\big] + (1-\lambda)\,\mathbb{E}_{\bm{x}\sim p'_s}\big[\|\mathcal{R}_{s,t}(V,b')(\bm{x})\|_2^2\big]\Big)\,\mathrm{d}t\,\mathrm{d}s \, ,
\end{equation}
where, by \cref{eq: consistency_residual}, the two residuals read
\begin{align}
    \mathcal{R}_{s,t}(V,b)(\bm{x}) &= V_{s,t}(\bm{x}) - \mathrm{sg}\Big[b_s(\bm{x}) + (t-s)\big(\partial_s V_{s,t}(\bm{x}) + \nabla_{\bm{x}} V_{s,t}(\bm{x})\, b_s(\bm{x})\big)\Big] \, , \label{eq: app_residual_hist}\\
    \mathcal{R}_{s,t}(V,b')(\bm{x}) &= V_{s,t}(\bm{x}) - \mathrm{sg}\Big[b'_s(\bm{x}) + (t-s)\big(\partial_s V_{s,t}(\bm{x}) + \nabla_{\bm{x}} V_{s,t}(\bm{x})\, b'_s(\bm{x})\big)\Big] \, . \label{eq: app_residual_inc}
\end{align}
As illustrated in \cref{sec: semigradient} and following \citet{dieleman2026flowmaps}, each residual is affine in its velocity slot, and the stop-gradient keeps the target fixed under differentiation. Denote $\bm{c} := (t-s)\,\partial_s V_{s,t}$ and $M := \mathbb{I}_d + (t-s)\nabla_{\bm{x}} V_{s,t}$, suppressing their dependence on $(s,t,\bm{x})$. We rewrite \cref{eq: app_residual_hist} and \cref{eq: app_residual_inc} as follows:
\begin{align}
    \mathcal{R}_{s,t}(V,b)(\bm{x}) &= V_{s,t}(\bm{x}) - \mathrm{sg}\big[\bm{c}(\bm{x}) + M(\bm{x})\,b_s(\bm{x})\big] \, , \label{eq: app_residual_hist_aff}\\
    \mathcal{R}_{s,t}(V,b')(\bm{x}) &= V_{s,t}(\bm{x}) - \mathrm{sg}\big[\bm{c}(\bm{x}) + M(\bm{x})\,b'_s(\bm{x})\big] \, . \label{eq: app_residual_inc_aff}
\end{align}
The two share $\bm{c}$ and $M$, which depend on $V$ but not on the velocity, and differ only in which field occupies the affine slot.

\textbf{Stationarity.} Write the expectations in \cref{eq: app_mean_flow_mixture_objective} as integrals over $\bm{x}$ and perturb $V \to V + \varepsilon\,\eta$ for an arbitrary direction $\eta$. Because the stop-gradient holds $\bm{c}$ and $M$ fixed, each residual moves by exactly $\varepsilon\,\eta$, so expanding the squares and collecting first-order terms gives
\begin{equation}\label{eq: app_mean_flow_semigradient}
    \frac{\mathrm{d}}{\mathrm{d}\varepsilon}\Big|_{\varepsilon=0}\mathcal{K}(V+\varepsilon\eta; b,b') = \int_0^1\!\!\int_s^1\!\!\int \widetilde{G}(s,t,\bm{x})\cdot\eta(s,t,\bm{x})\,\mathrm{d}\bm{x}\,\mathrm{d}t\,\mathrm{d}s \, ,
\end{equation}
with semigradient density
\begin{equation}\label{eq: app_mean_flow_semigradient_density}
    \widetilde{G}(s,t,\bm{x}) := 2\Big(\lambda\, p_s(\bm{x})\,\mathcal{R}_{s,t}(V,b)(\bm{x}) + (1-\lambda)\, p'_s(\bm{x})\,\mathcal{R}_{s,t}(V,b')(\bm{x})\Big) \, ,
\end{equation}
where the two densities appear as pointwise weights. This vanishes for every $\eta$ if and only if $\widetilde{G}$ vanishes at every $(s,t,\bm{x})$. Substituting \cref{eq: app_residual_hist_aff} and \cref{eq: app_residual_inc_aff}, the common term $V_{s,t}-\bm{c}$ is weighted by $\lambda p_s + (1-\lambda)p'_s$ while the velocity slot is weighted by $\lambda p_s b_s + (1-\lambda)p'_s b'_s$, so dividing by the total weight leaves
\begin{equation}\label{eq: app_mean_flow_fixed_point}
    V_{s,t}(\bm{x}) = \bm{c}(\bm{x}) + M(\bm{x})\,\frac{\lambda p_s(\bm{x}) b_s(\bm{x}) + (1-\lambda)p'_s(\bm{x}) b'_s(\bm{x})}{\lambda p_s(\bm{x}) + (1-\lambda)p'_s(\bm{x})} = \bm{c}(\bm{x}) + M(\bm{x})\, v^\star_s(\bm{x}) \, ,
\end{equation}
where the division is admitted by the positivity assumption and the ratio is the mixture velocity of \cref{eq: dynamic_kl_minimizer}. Hence $\mathcal{R}_{s,t}(V,v^\star)=\bm{0}$ for all $0\le s<t\le 1$: the density weighting is inherited rather than imposed, since the stop-gradient leaves each residual affine in its velocity, and a weighted sum of affine conditions is the condition at the weighted velocity.

\textbf{Uniqueness.} The stationarity condition is a differential equation in $V$, so on its own it fixes $V$ only up to a constant of integration. Following the sufficiency argument of \citet[Sec.~B.3]{geng2026mean}, collecting $-V_{s,t} + (t-s)\tfrac{\mathrm{d}}{\mathrm{d}s}V_{s,t}$ into the form $\tfrac{\mathrm{d}}{\mathrm{d}s}\big[(t-s)V_{s,t}\big]$ by the product rule, $\mathcal{R}_{s,t}(V,v^\star)=\bm{0}$ reads
\begin{equation}\label{eq: app_displacement_ode}
    \frac{\mathrm{d}}{\mathrm{d}s}\Big[(t-s)\,V_{s,t}(\bm{x}^\star_s)\Big] = -\,v^\star_s(\bm{x}^\star_s) \, ,
\end{equation}
along the trajectory $\bm{x}^\star_\tau$ solving \cref{eq: ode_solution} with velocity $v^\star$. Integrating \cref{eq: app_displacement_ode} in $s$ recovers the two sides only up to constants,
\begin{equation}\label{eq: app_uniqueness_constants}
    (t-s)\,V_{s,t}(\bm{x}^\star_s) + C_1 = \int_s^t v^\star_\tau(\bm{x}^\star_\tau)\,\mathrm{d}\tau + C_2 \, ,
\end{equation}
which coincide since both sides vanish at $s=t$: the left by the factor $(t-s)$, for any $V$ finite on the diagonal, and the right because the integration limits meet. Hence
\begin{equation}\label{eq: app_uniqueness_integration}
    (t-s)\,V_{s,t}(\bm{x}^\star_s) = \int_s^t v^\star_\tau(\bm{x}^\star_\tau)\,\mathrm{d}\tau \, ,
\end{equation}
which is \cref{eq: mixture_mean_velocity}: the stationary point of $\mathcal{K}(\cdot\,;b,b')$ is the mean velocity of the mixture path. As \citet{geng2026mean} note, this relies on parameterizing the average velocity rather than the displacement: had $S_{s,t}:=(t-s)V_{s,t}$ been modeled directly, $S_{t,t}=0$ would have to be imposed as a separate boundary condition.
\end{proof}

\subsection{A stochastic view on the mixture optimization problem}\label{sec: stochastic_formulation}
Consider the diffusion counterpart of \cref{sec: flow_matching_basics}, where sampling from $q$ and $q'$ is governed by
\begin{equation}\label{eq: stochastic_version}
\begin{aligned}
    \mathbb{P}: \quad &\mathrm{d}X_t = \mu_t(X_t)\,\mathrm{d}t + \epsilon_t\,\mathrm{d}B_t \, , \qquad X_0 \sim \mathcal{N}(\bm{0}_d, \mathbb{I}_d), \; X_1 \sim q \, , \\
    \mathbb{P}': \quad &\mathrm{d}X_t = \mu_t'(X_t)\,\mathrm{d}t + \epsilon_t\,\mathrm{d}B_t \, , \qquad X_0 \sim \mathcal{N}(\bm{0}_d, \mathbb{I}_d), \; X_1 \sim q' \, ,
\end{aligned}
\end{equation}
with $B_t$ a standard Brownian motion and $\epsilon_t \in \mathbb{R}_+$ a time-dependent diffusion coefficient. The marginals $p_t,p'_t$ at time $t$ evolve under the corresponding Fokker--Planck equations,
\begin{equation}\label{eq: app_fokker_planck}
    \partial_t p_t = -\nabla\cdot(\mu_t p_t) + \tfrac{1}{2}\epsilon_t^2\Delta p_t \, , \qquad
    \partial_t p'_t = -\nabla\cdot(\mu'_t p'_t) + \tfrac{1}{2}\epsilon_t^2\Delta p'_t \, ,
\end{equation}
and the drifts decompose through the probability-flow identity as
\begin{equation}\label{eq: drift_decomposition}
    \mu_t = b_t + \tfrac{1}{2}\epsilon_t^2\nabla_{\bm{x}}\log p_t \, , \qquad
    \mu_t' = b_t' + \tfrac{1}{2}\epsilon_t^2\nabla_{\bm{x}}\log p_t' \, ,
\end{equation}
where $b_t,b_t'$ are the velocity fields of the probability-flow ODEs, which generate the same marginals $p_t,p'_t$ through the continuity equation. Formulating the mixture problem over path measures rather than densities, Girsanov's theorem reduces the divergence between two such diffusions to an integrated quadratic drift mismatch, weighted by the inverse squared diffusion coefficient. Similar to \citet{thorkelsdottir2026gradual} but applied to the forward KL between path measures, we obtain the following proposition.

\setcounter{theoremInner}{0}
\begin{proposition}\label{prop: stochastic_path_mixture}
Let $\mathbb{P},\mathbb{P}'$ be the path measures of \cref{eq: stochastic_version}, with marginals $p_t,p_t'$, and let $\tilde{\mathbb{P}}$ range over diffusions with the same $\epsilon_t>0$ and initial condition. Fix $\lambda\in(0,1)$ and assume $\lambda p_t(\bm{x})+(1-\lambda)p'_t(\bm{x})>0$ for all $t\in[0,1]$ and $\bm{x}\in\mathcal{X}$. Then
\begin{equation}
    \argmin_{\tilde{\mathbb{P}}}\ \lambda D_{\mathrm{KL}}(\mathbb{P} \,\|\, \tilde{\mathbb{P}}) + (1-\lambda) D_{\mathrm{KL}}(\mathbb{P}' \,\|\, \tilde{\mathbb{P}})
\end{equation}
admits the unique minimizer $\tilde{\mathbb{P}}^\star:\ \mathrm{d}X_t = \mu_t^\star(X_t)\,\mathrm{d}t + \epsilon_t\,\mathrm{d}B_t$, with
\begin{equation}\label{eq: optimal_drift}
    \mu_t^\star(\bm{x}) := \frac{\lambda\, p_t(\bm{x})\,\mu_t(\bm{x}) + (1-\lambda)\, p_t'(\bm{x})\,\mu_t'(\bm{x})}{\lambda\, p_t(\bm{x}) + (1-\lambda)\, p_t'(\bm{x})} \, ,
\end{equation}
whose marginal law is the mixture path $\rho^\star_t = \lambda p_t + (1-\lambda)p_t'$ of \cref{prop: objective_minimizer}.
\end{proposition}
\begin{proof}
\textbf{Minimizer.} Since $\mathbb{P}$, $\mathbb{P}'$ and $\tilde{\mathbb{P}}$ share the diffusion coefficient $\epsilon_t$ and the initial condition, Girsanov's theorem gives
\begin{equation}\label{eq: app_girsanov}
    D_{\mathrm{KL}}(\mathbb{P}\,\|\,\tilde{\mathbb{P}}) = \int_0^1 \frac{1}{2\epsilon_t^2}\,\mathbb{E}_{p_t}\big[\|\mu_t(\bm{x}) - \tilde\mu_t(\bm{x})\|_2^2\big]\,\mathrm{d}t \, ,
\end{equation}
and likewise for $\mathbb{P}'$ with $p'_t$ and $\mu'_t$. The weighted objective is therefore the functional of \cref{eq: app_grouped_objective} with $\mu_t,\mu'_t$ in place of $b_t,b'_t$ and an extra positive weight $1/(2\epsilon_t^2)$, which does not affect the pointwise minimization. The argument of \cref{sec: proof_unique_minimizer} applies and yields \cref{eq: optimal_drift} as the unique minimizer.

\textbf{Marginal law.} Taking the $\lambda$-weighted combination of the two equations in \cref{eq: app_fokker_planck} and using $\mu^\star_t\rho^\star_t = \lambda p_t\mu_t + (1-\lambda)p'_t\mu'_t$, which is \cref{eq: optimal_drift} cleared of its denominator, gives
\begin{equation}
    \partial_t\rho^\star_t = -\nabla\cdot\big(\mu^\star_t\rho^\star_t\big) + \tfrac{1}{2}\epsilon_t^2\Delta\rho^\star_t \, ,
\end{equation}
so $\rho^\star_t$ is the marginal law of $\tilde{\mathbb{P}}^\star$. The boundary conditions carry over as in \cref{sec: proof_unique_minimizer}.
\end{proof}

Substituting \cref{eq: drift_decomposition} into \cref{eq: optimal_drift} and using $p\,\nabla_{\bm{x}}\log p = \nabla_{\bm{x}} p$, the score terms recombine into the score of $\rho^\star_t$, leaving
\begin{equation}\label{eq: optimal_velocity_stoch}
    \mu_t^\star = v_t^\star + \tfrac{1}{2}\epsilon_t^2\nabla_{\bm{x}}\log \rho_t^\star \, ,
\end{equation}
with $v_t^\star$ exactly the velocity field of \cref{eq: dynamic_kl_minimizer}. The deterministic and stochastic formulations therefore select the same mixture: the optimal drift is the optimal velocity plus the score correction its own marginal requires.

\subsection{Additional details on the sequential mean flow update}\label{sec: details_mean_flow}
\paragraph{Parameterized objective.}
Following \cref{sec: mean_flows_basics} and mirroring \cref{eq: empirical_mixture_objective} for the instantaneous field, we derive a counterpart of \cref{eq: mean_flow_mixture_objective} across multiple sequential updates. Specifically, at step $k$:
\begin{equation}\label{eq: app_empirical_mean_flow_objective}
\begin{aligned}
    \mathcal{K}^{k+1} = \int_0^1\!\!\int_s^1 \Big( \ &\lambda_k\, \mathbb{E}_{\bm{x}_s\sim p_s^{\theta_{\le k}}}\big[\|\mathcal{R}_{s,t}(V^{\theta_{\le k+1}}, b^{\theta_{\le k}})(\bm{x}_s)\|_2^2\big] \\
    + \ &(1-\lambda_k)\, \mathbb{E}_{\bm{x}_1\sim q_{k+1},\,\bm{x}\sim p_{s|1}(\cdot\mid\bm{x}_1)}\big[\|\mathcal{R}_{s,t}(V^{\theta_{\le k+1}}, b_{\cdot|1}(\cdot\mid\bm{x}_1))(\bm{x})\|_2^2\big] \Big)\,\mathrm{d}t\,\mathrm{d}s \, .
\end{aligned}
\end{equation}
The two backends differ in how the historical samples $\bm{x}_s$ are obtained.
\begin{itemize}[leftmargin=*, itemindent=0pt]
  \item \emph{ODE backend.} No mean velocity is learned. Historical states come from integrating $b^{\theta_{\le k}}$ along \cref{eq: ode_solution}, at a cost that grows with the solver budget, and only $v^{\theta_{\le k+1}}$ is carried forward.
  \item \emph{MeanFlow backend.} Each update additionally trains a mean velocity network $V^{\theta_{\le k+1}}$ under \cref{eq: app_empirical_mean_flow_objective}, whose flow map supplies the historical states at the next update in a few evaluations. Both fields are carried forward, and the first task, having no history, trains them on its own data alone, by flow matching and single-teacher distillation, respectively.
\end{itemize}
$V^{\theta_{\le k+1}}$ is therefore auxiliary: it exists only to reduce the number of network evaluations needed to train the joint-task field at the following update. Note that the incoming branch is trained in a data-conditional way, as in \citet{geng2026mean} and explained in \cref{sec: semigradient}: substituting the conditional velocity for the marginal one leaves the semigradient unchanged in expectation.

\paragraph{Departure from the population objective.} \Cref{prop: mean_flow_minimizer} characterizes the fixed point of $\mathcal{K}$ under the true marginal fields and exact samples from $p_s,p'_s$, neither of which is available at training time. First, the historical field is replaced by the previous model $b^{\theta_{\le k}}_s$, itself approximate. Second, the incoming field is never formed: the conditional substitution is exact in expectation but introduces variance at finite batch size. Third, historical states are drawn through the learned flow map, so their sampling measure is $(\bm{X}^{\theta_{\le k}}_{0,s})_\sharp p_0$ rather than $p_s$. \Cref{eq: app_empirical_mean_flow_objective} therefore differs from $\mathcal{K}$, and the fixed-point characterization transfers only up to these three approximations.

\subsection{Integration on the selection of \texorpdfstring{$\lambda$}{lambda}}\label{sec: selection_of_lambda}
The coefficient $\lambda_k$ governs not only a single update but the weight each task retains across the whole stream. Unrolling \cref{eq: mixture} over the stream expresses the terminal density as a mixture over all tasks seen, with weights determined by the schedule $\{\lambda_k\}_{k=1}^{K-1}$. Holding $\lambda_k=\lambda$ constant gives the incoming task weight $1-\lambda$, the previous one $\lambda(1-\lambda)$, and any task introduced $m$ updates earlier $\lambda^{m}(1-\lambda)$, with the initial task retaining $\lambda^{K-1}$ since it never entered as an incoming task: the weight \textit{decays geometrically} in task age, so a constant rate discounts older tasks rather than balancing them. 

Uniform weight instead requires $\lambda_k=k/(k+1)$, since at step $k$ the historical model carries $k$ tasks at weight $1/k$ each, and weighting it by $k/(k+1)$ against $1/(k+1)$ for the incoming task leaves all $k+1$ tasks at $1/(k+1)$. Selecting $\lambda_k$ is therefore a choice of how the stream is weighted overall, not only of the current stability-plasticity trade-off, and the two schedules above correspond to the settings compared in \cref{sec: experiments}. The adaptive policy admits no such closed form, since the weights depend on validation outcomes at each update, so the terminal mixture it induces is known only after the fact.

\subsection{Efficiency and stability gains}\label{sec: partial_replay_analysis}
\paragraph{Evaluation count.} Fix a solver step size $h$ and write $b$ for the frozen historical field. Generative replay integrates every replay sample from $0$ to $1$, costing $\lceil 1/h\rceil$ evaluations. Trajectory replay integrates only to $t\sim\mathcal{U}(0,1]$ and then evaluates $b_t$ once at the anchor, costing $\lceil t/h\rceil+1$, so that in expectation
\begin{equation}\label{eq: app_nfe_ratio}
    \mathbb{E}_t\big[\lceil t/h\rceil+1\big] \;\approx\; \frac{1}{2h}+1
    \qquad\text{against}\qquad
    \frac{1}{h} \, ,
\end{equation}
approaching a factor of two as $h\to0$. At $h=0.05$, for instance, a full generation costs $20$ evaluations while trajectory replay averages $11.5$, a saving of $42.5\%$ that grows as the solver is refined.

\paragraph{Accumulated solver error.} Intuitively, simulating less results in lower discretization error. Let $\bm{x}_\tau$ solve \cref{eq: ode_solution} exactly under the frozen field $b$, and let $\tilde{\bm{x}}_\tau$ denote its Euler approximation at step size $h$ from the same initial state. Assume, in addition to the standing assumptions, that $b$ is Lipschitz in $\bm{x}$ uniformly in time, with constant $L>0$, and that the exact path has bounded curvature, $\sup_{\tau}\|\ddot{\bm{x}}_\tau\| \le M_2$. The classical global error bound for Euler's method \citep{suli2003introduction} then gives, with exact initial data,
\begin{equation}\label{eq: app_euler_error}
    \|\tilde{\bm{x}}_t-\bm{x}_t\| \;\leq\; \frac{M_2 h}{2L}\Big(e^{Lt}-1\Big) \, ,
\end{equation}
which is increasing in $t$: the local error of each step is $O(h^2)$, and the flow amplifies errors already committed by at most $(1+hL)$ per step. Generative replay always pays \cref{eq: app_euler_error} at $t=1$, whereas trajectory replay pays it at a uniformly drawn anchor, so the two stand in the ratio
\begin{equation}\label{eq: app_error_ratio}
    \frac{\mathbb{E}_t\big[e^{Lt}-1\big]}{e^{L}-1}
    = \frac{(e^{L}-1)/L - 1}{e^{L}-1} \;<\; \frac{1}{2} \, ,
\end{equation}
which approaches $1/2$ from below as $L\to0$ and decays like $1/L$ for large $L$, since $e^{Lt}$ concentrates near the endpoint and a uniformly drawn anchor mostly avoids it. Stopping early therefore reduces the worst-case error in the replayed state, and by more when the field varies rapidly.

\paragraph{Where the error lands.} The two objectives also differ in what that error corrupts. Generative replay treats the simulated endpoint $\tilde{\bm{x}}_1$ as if it were a real historical sample, so the student is fit to a corrupted dataset: by \cref{sec: replay_as_mixture_modeling} its optimum reproduces the law of $\tilde{\bm{x}}_1$ rather than the historical density, and training longer or harder cannot recover the difference. Trajectory replay instead asks the frozen model what its velocity is at $\tilde{\bm{x}}_t$, and the answer is correct whatever $\tilde{\bm{x}}_t$ happens to be, since it is the \textit{teacher's own value} at that point. Solver error then only shifts where along the trajectory the student is supervised, never what it is taught there, so the teacher is still reproduced exactly on the states that are visited.

\paragraph{Caveats with trajectory replay.} The argument only covers the states the solver actually reaches. A solver that systematically avoids a region leaves the student unconstrained there, and with finite capacity the shift in which states are visited is not innocuous either, since supervising the wrong points trades approximation error between regions. \Cref{fig: app_fig2} and the matched-backend comparisons of \cref{sec:exp-image-cl} measure the combined effect of both mechanisms.

\subsection{Relation to feature replay}\label{sec: relation_feature_replay}
Trajectory replay is loosely connected with feature replay approaches, sharing a similar motivating diagnosis: replaying generated images is expensive, and errors in those images enter training as if they were data. \citet{liu2020generative} address this by splitting the network at a layer, freezing and distilling the feature extractor, and replaying intermediate features rather than images. A generator produces features for past classes, which are paired with their old labels and fed to the classifier. The gain is that feature distributions are lower-dimensional and far easier to model than pixels, so the generator can be accurate where an image generator would not be.

\textbf{A continuous-depth view.} The same structure appears in our setting. Under the neural ODE correspondence \citep{chen2018neural}, integration time is depth and a flow is a \textit{weight-tied residual network} \citep{he2016deep} whose block at depth $t$ is
\begin{equation}\label{eq: app_resnet_block}
    \bm{x}_{t+\Delta t} = \bm{x}_t + \Delta t\, v_t(\bm{x}_t) \, ,
\end{equation}
so $v_t$ is the residual branch and $\bm{x}_t$ an intermediate representation. Generative replay runs every block to obtain the output $\bm{x}_1$ and treats it as a sample. Trajectory replay runs a prefix and supervises the residual branch at the depth it reached. Two consequences follow. The expected prefix is half the depth, which is the accounting of \cref{sec: partial_replay_analysis}. And because the blocks share parameters, a single sampled $t$ trains the whole network, so no layer-wise schedule is needed.

\textbf{Where the two differ.} The resemblance stops at what is replayed. Feature replay generates a \textit{substitute for a real activation} and pairs it with an old label, so the generator's fidelity decides whether the training pair is correct, exactly as in standard generative replay. The requirement is moved to an easier space rather than removed. Trajectory replay generates no substitute for anything. The state $\bm{x}_t$ is a probe location, and the target is the frozen model's own velocity there, which is a valid input-output pair wherever the solver landed (\cref{sec: partial_replay_analysis}). The split also differs: feature replay's is architectural and fixed. Ours is temporal and resampled every step, so no layer is privileged.

\section{Hardware details}\label{app: hardware_details}
All experiments run on a Slurm cluster, one NVIDIA H100 80GB GPU per job.

\clearpage
\section{Algorithm}\label{sec: algorithms}
\begin{algorithm}[!ht]
\caption{Trajectory replay continual update at step $k$, ODE backend}
\label{alg: trajectory_replay_ode}
\KwIn{Frozen historical field $b^{\theta_{\le k}}$; incoming dataset $\mathcal{D}_{k+1}\sim q_{k+1}$; weight $\lambda_k$; teacher step size $h$; time cutoff $\varepsilon$; optimizer steps $T$}
\KwOut{The updated joint-task field $v^{\theta_{\le k+1}}$}
$\theta_{\le k+1}\leftarrow\texttt{copy}(\theta_{\le k})$ \tcp*{warm start; $b^{\theta_{\le k}}$ stays frozen}
\For{$1,\dots,T$}{
  $t\sim\mathcal{U}(\varepsilon,1]$\;
  \tcc{Historical branch: partial replay}
  $\bm{x}_0\sim p_0$\;
  $\bm{x}_t\leftarrow\textsc{ODESolve}(b^{\theta_{\le k}},\bm{x}_0,0\!\to\!t,h)$ \tcp*{$\lceil t/h\rceil$ evals}
  $\mathcal{L}_{\mathrm{hist}}\leftarrow\|v^{\theta_{\le k+1}}_t(\bm{x}_t)-b^{\theta_{\le k}}_t(\bm{x}_t)\|_2^2$ \tcp*{one more eval}
  \tcc{Incoming branch: conditional flow matching}
  $\bm{x}_1\sim\mathcal{D}_{k+1}$, $\bm{x}\sim p_{t|1}(\cdot\mid\bm{x}_1)$\;
  $\mathcal{L}_{\mathrm{inc}}\leftarrow\|v^{\theta_{\le k+1}}_t(\bm{x})-b_{t|1}(\bm{x}\mid\bm{x}_1)\|_2^2$\;
  $\mathcal{J}^{k+1}\leftarrow\lambda_k\mathcal{L}_{\mathrm{hist}}+(1-\lambda_k)\mathcal{L}_{\mathrm{inc}}$ \tcp*{\cref{eq: empirical_mixture_objective}}
  $\theta_{\le k+1}\leftarrow\textsc{OptStep}(\theta_{\le k+1},\nabla_{\theta_{\le k+1}}\mathcal{J}^{k+1})$\;
}
\Return $v^{\theta_{\le k+1}}$\;
\end{algorithm}

\begin{algorithm}[!ht]
\caption{Trajectory replay continual update at step $k$, MeanFlow backend}
\label{alg:trajectory_replay_meanflow}
\KwIn{Frozen field $b^{\theta_{\le k}}$ and frozen mean velocity $V^{\theta_{\le k}}$; incoming dataset $\mathcal{D}_{k+1}\sim q_{k+1}$; weight $\lambda_k$; replay jumps $N$; diagonal rate $p_{\mathrm{diag}}$; time law $\mathcal{T}$; time cutoff $\varepsilon$; optimizer steps $T$}
\KwOut{The updated joint-task field $v^{\theta_{\le k+1}}$ and mean velocity $V^{\theta_{\le k+1}}$}
$\theta_{\le k+1}=(\theta^{v},\theta^{V})\leftarrow\texttt{copy}(\theta_{\le k})$ \tcp*{warm start; teachers stay frozen}
\For{$1,\dots,T$}{
  $t\sim\mathcal{U}(\varepsilon,1]$ \tcp*{one shared time per step}
  \tcc{Few-step replay through the frozen flow map}
  $\bm{x}\sim p_0$, \quad $u_0\leftarrow 0$\;
  \For{$n=1,\dots,N$}{
    $u_n\leftarrow nt/N$\;
    $\bm{x}\leftarrow\bm{x}+(u_n-u_{n-1})\,V^{\theta_{\le k}}_{u_{n-1},u_n}(\bm{x})$\;
  }
  $\bm{x}_t\leftarrow\bm{x}$ \tcp*{$N$ evals, independent of $t$}
  \tcc{Velocity field}
  $\bm{x}_1\sim\mathcal{D}_{k+1}$, \quad $\bm{x}'\sim p_{t|1}(\cdot\mid\bm{x}_1)$\;
  $\mathcal{J}^{k+1}\leftarrow\lambda_k\|v^{\theta_{\le k+1}}_t(\bm{x}_t)-b^{\theta_{\le k}}_t(\bm{x}_t)\|_2^2+(1-\lambda_k)\|v^{\theta_{\le k+1}}_t(\bm{x}')-b_{t|1}(\bm{x}'\mid\bm{x}_1)\|_2^2$\;
  $\theta^{v}\leftarrow\textsc{OptStep}(\theta^{v},\nabla_{\theta^{v}}\mathcal{J}^{k+1})$\;
  \BlankLine
  \tcc{Mean velocity, carried to the next update}
  $s\leftarrow t$ \tcp*{reuses the shared time and the rollout $\bm{x}_t$}
  $\tau_{\mathrm{o}},\tau_{\mathrm{n}}\leftarrow s+(1-s)u$ with $u\sim\mathcal{T}$, each set to $s$ w.p.\ $p_{\mathrm{diag}}$\;
  $\bm{x}''\sim p_{s|1}(\cdot\mid\bm{x}_1)$\;
  $\mathcal{K}^{k+1}\leftarrow\lambda_k\|\mathcal{R}_{s,\tau_{\mathrm{o}}}(V^{\theta_{\le k+1}},b^{\theta_{\le k}})(\bm{x}_t)\|_2^2+(1-\lambda_k)\|\mathcal{R}_{s,\tau_{\mathrm{n}}}(V^{\theta_{\le k+1}},b_{\cdot|1}(\cdot\mid\bm{x}_1))(\bm{x}'')\|_2^2$\;
  $\theta^{V}\leftarrow\textsc{OptStep}(\theta^{V},\widetilde{G}_{\theta^{V}}[\mathcal{K}^{k+1}])$ \tcp*{semigradient, \cref{sec: semigradient}}
}
\Return $v^{\theta_{\le k+1}},\,V^{\theta_{\le k+1}}$\;
\end{algorithm}


\clearpage

\section{Experimental details}\label{sec: experimental_details_general}

\subsection{Controlled mixture recovery}\label{app:results-mixture-recovery}

\subsubsection{Experimental details}\label{app:results-mixture-recovery-exp}

\textbf{Data.} Each task is an isotropic Gaussian in $\mathbb{R}^d$ with $\sigma=0.35$, with mode means on the first coordinate at $\pm1.5$ and all remaining coordinates pure noise axes. Task~0 ($-1.5$) is the historical mode with density $q$, and task~1 ($+1.5$) the incoming one with density $q'$. We draw $10{,}000$ points per task at $d=2$ and $100{,}000$ at $d\in\{128,256\}$, with data seed $0$. At $d=2$ we sweep $\lambda\in\{0.1,0.3,0.5,0.7,0.9\}$ and, for each value, sample an independent reference set of the same size from the task mixture $\tilde{q}=\lambda q+(1-\lambda)q'$ of \cref{eq: mixture}, used only for evaluation. The higher-dimensional runs fix $\lambda=0.5$.

\textbf{Model.} The velocity field is an unconditional MLP of width $256$ and depth $4$, with a $64$-dimensional sinusoidal time embedding, trained on the linear interpolant $\bm{x}_t=(1-t)\bm{x}_0+t\bm{x}_1$ with target $\bm{x}_1-\bm{x}_0$ \citep{tong2024improving}. We sample using Euler with step size $0.05$ ($20$ evaluations). Given the task's simplicity, we did not ablate model capacity.

\textbf{Pretraining.} The task-0 model trains for $400$ epochs with batch size $256$ and Adam \citep{kingma2014adam} at learning rate $10^{-3}$, seed $0$. A single checkpoint is shared by every update run, so all differences across methods and seeds come from the update alone.

\textbf{Update.} Each update trains for $400$ epochs with the same optimizer settings, for seeds $\{0,1,2\}$ at every $\lambda$ in the grid. We compare naive fine-tuning, which continues training the historical checkpoint on incoming data alone with no replay term, against generative and trajectory replay, plus a joint-training reference fit directly on samples from $\tilde{q}$. Generative replay integrates the frozen teacher to $t=1$ and treats the samples as data, drawing a fraction $\lambda$ of every batch from the teacher and the remainder from the incoming task, then taking a \textit{single flow-matching loss} over the combined batch. Trajectory replay splits each batch evenly, draws one shared $t\sim\mathcal{U}(10^{-4},1]$ per step, integrates the teacher only to $t$, and matches the student field against the teacher field there, with $\lambda$ weighting the historical and incoming terms of the loss.

\textbf{Evaluation.} We generate $4{,}000$ samples per run with the same sampler and report two quantities. Generated historical mass $\hat{\pi}_0$ (\cref{fig:simulation-mixture}a) is the fraction of samples whose nearest mode mean is mode~0. Energy distance is computed between the generated samples and the held-out reference set, restricted to the active coordinate on which the mode means differ (\cref{fig: app_fig2}). All figures show three seeds.

\textbf{Sampling-budget comparison (\cref{fig: app_fig2}).} The teacher-cost curves use $d\in\{128,256\}$ at $\lambda=0.5$ (equal-task), training for $513$ epochs ($200{,}070$ optimizer steps) with evaluation every $10$ epochs on $2{,}000$ samples, for seeds $\{0,1,2\}$. Cumulative teacher evaluations follow analytically from the fixed Euler protocol: at $100{,}000$ points and batch size $256$, there are $390$ steps per epoch, generative replay spends $\lceil 1/0.05\rceil=20$ evaluations per step, and trajectory replay spends $\mathbb{E}[\lceil t/0.05\rceil]+1=11.5$, since $t$ is uniform on $(0,1]$, integrating only to an intermediate time plus one frozen-field evaluation at the anchor. Over the full run, this is $4.0\times10^{6}$ teacher evaluations for generative replay against $2.3\times10^{6}$ for trajectory replay.
\subsubsection{Additional results}\label{app:results-mixture-recovery-res}
We complement \cref{fig:simulation-mixture} with two additional views. \Cref{fig: app_fig1}a shows the mixture behavior at $\lambda\in\{0.1,0.5,0.9\}$, where generated samples track the prescribed weight across the full range. \Cref{fig: app_fig1}b contrasts this with naive fine-tuning, which collapses onto the incoming mode regardless of $\lambda$ and recovers no historical mass.

\Cref{fig: app_fig2} plots the energy distance against cumulative teacher evaluations rather than training steps, isolating the cost of replay from the cost of optimization. Trajectory replay reaches a given energy at a smaller budget in both dimensions, and the gap widens with $d$: at $d=128$ the two curves converge once the target is essentially matched, whereas at $d=256$ generative replay needs close to twice the budget to reach the level trajectory replay attains at $\approx10^6$ evaluations. 

Part of this follows from the accounting above: $11.5$ evaluations per step against $20$ for full generation. The rest is robustness to the solver, analyzed in \cref{sec: partial_replay_analysis}: discretization error displaces the replayed state but leaves the teacher's velocity there a valid target, whereas the same error corrupts a generated sample that is then fit as data. The matched-backend comparisons in \cref{sec:exp-image-cl} isolate this at equal replay cost.

\begin{figure}[!htbp]
\centering
\includegraphics[width=0.9\textwidth]{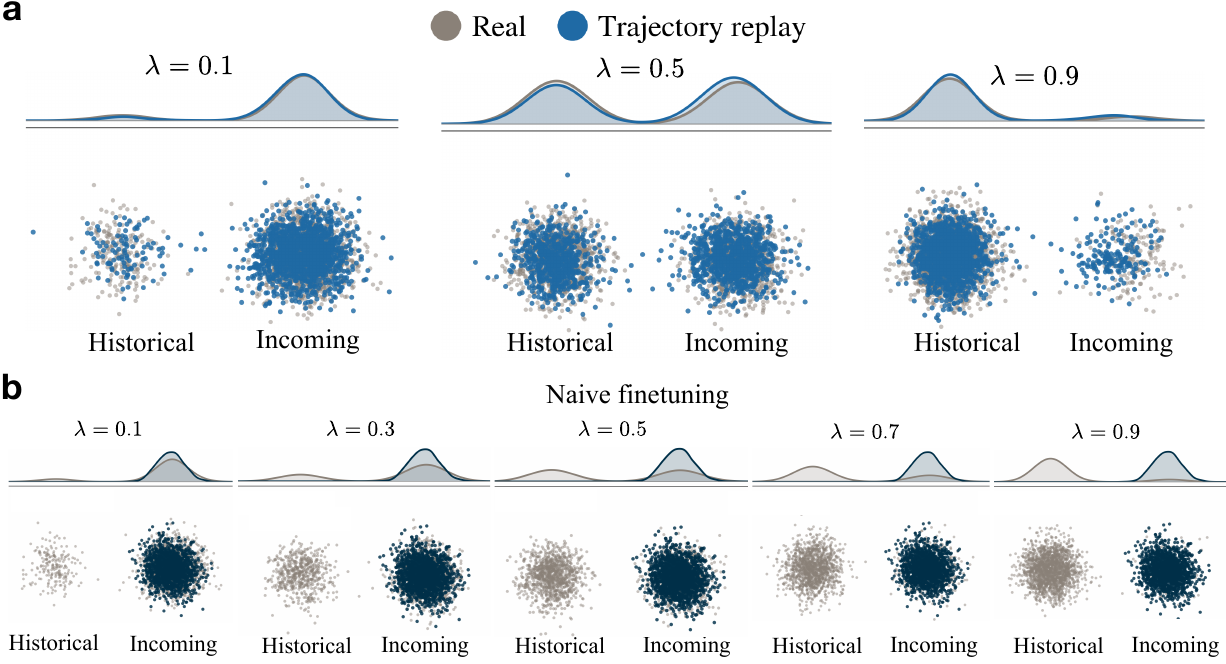}
\caption{\textbf{Two-mode mixture updates.}
(\textbf{a}) Samples and one-dimensional marginals from the real mixture and from trajectory replay after the update, for historical and incoming modes at $\lambda \in \{0.1, 0.5, 0.9\}$.
(\textbf{b}) Samples and marginals from the real mixture and from naive fine-tuning after the update, shown for $\lambda \in \{0.1, 0.3, 0.5, 0.7, 0.9\}$.}
\label{fig: app_fig1}
\end{figure}

\begin{figure}[!htbp]
\centering
\includegraphics[width=0.85\textwidth]{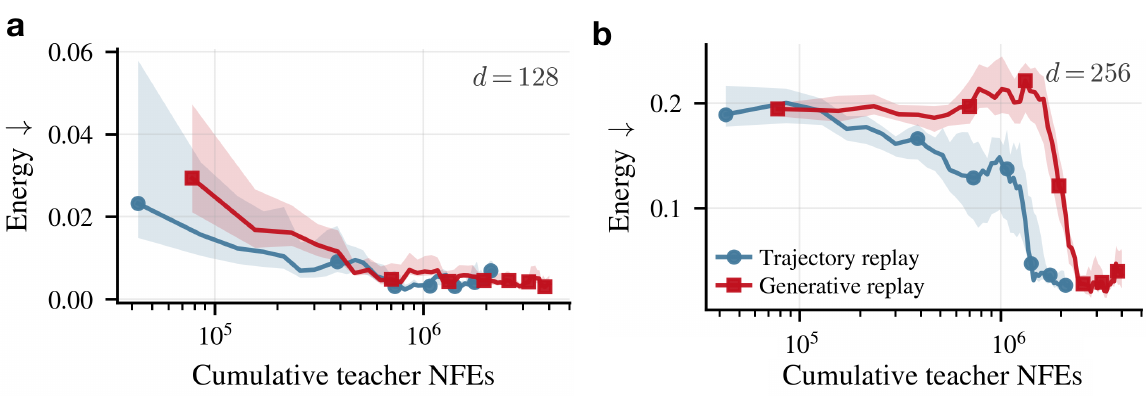}
\caption{\textbf{Energy distance versus teacher sampling budget.}
Energy distance to the target mixture, computed on the active coordinate that separates the modes, as a function of cumulative teacher NFEs for trajectory replay and generative replay, in (\textbf{a}) $d=128$ and (\textbf{b}) $d=256$. Lines show the mean and shaded bands the variability across three seeds. The $x$-axis is logarithmic and lower energy is better.}
\label{fig: app_fig2}
\end{figure}

\FloatBarrier
\clearpage
\subsection{Sequential updates on a ring-structured Gaussian mixture}\label{app:results-sequential-rings}

\subsubsection{Experimental details}\label{app:results-sequential-rings-details}

\textbf{Data.} Five isotropic Gaussians in $\mathbb{R}^2$ with $\sigma=0.35$ are placed on a ring of radius $3.0$, starting at the top and proceeding clockwise, giving a nearest-neighbor separation of $3.53$ between consecutive task means. Tasks arrive in that order, $T_0$ through $T_4$. Each task has independent train, validation, and test splits of $10{,}000$, $2{,}000$, and $5{,}000$ points, with data seed $0$. The balanced final target places $20\%$ of the mass on each task.

\textbf{Model.} The velocity field is the same unconditional MLP as \cref{app:results-mixture-recovery-exp}, of width $256$ and depth $4$ with a $64$-dimensional sinusoidal time embedding, trained on the linear interpolant with target $\bm{x}_1-\bm{x}_0$. We sample using Euler with step size $0.05$.

\textbf{Training.} Every task, including $T_0$, receives a budget of $20{,}000$ optimizer steps, rounded up to whole epochs at batch size $256$. Optimization uses Adam at $10^{-3}$ with gradient-norm clipping at $1.0$. The whole stream is repeated for seeds $\{0,1,2\}$, which vary initialization, minibatch order, and sampling noise.

\textbf{Replay policies.} We compare naive fine-tuning against generative and trajectory replay under three ways of setting the replay weight.
\begin{itemize}[leftmargin=*, itemindent=0pt]
  \item \emph{Constant.} $\lambda\in\{0.25,0.5,0.75,0.9\}$ held fixed across the stream.
  \item \emph{Equal-task.} $\lambda_k=k/(k+1)$ at update $k$, the parameter-free choice that keeps all seen tasks equally weighted.
  \item \emph{Adaptive.} The full update is run once for each candidate in $\{0.25,0.5,2/3,0.75,0.8,0.9\}$, and we keep the model whose $4{,}000$ generated samples minimize energy distance to a balanced validation mixture over all tasks seen so far, so the selection uses only validation data available at that point in the stream and no test data. The chosen model becomes the historical model for the next update, and the sweep is repeated there, which costs one full training run per candidate per update.
\end{itemize}
Running a policy to the end of the stream therefore fixes a target composition over tasks: equal-task yields the \textit{uniform target}, whereas a constant $\lambda$ discounts older tasks geometrically (\cref{sec: selection_of_lambda}). The generated task mass is scored against that target.

\textbf{Evaluation.} After every update, we generate $10{,}000$ samples and report two quantities against the test splits of the seen tasks. Energy distance is computed against the balanced mixture formed by concatenating equal-size test sets across seen modes. Generated task mass assigns each sample to its nearest task mean, adds a pseudocount of $0.5$ to keep the divergence finite when a task receives no samples, and is compared with the uniform target by forward categorical KL. All figures show three seeds.\looseness=-1

\FloatBarrier
\subsubsection{Additional results}
\textbf{Stream layout and final samples.} \Cref{fig: app_fig3}a shows the five-task ring layout and the order in which tasks arrive, and \cref{fig: app_fig3}b shows the samples drawn after the final update: naive fine-tuning retains only $T_4$, generative replay recovers all five modes with visibly uneven mass, and trajectory replay comes closest to the balanced target.

\textbf{Weighting policies.} \Cref{fig: app_fig4}a places every final model on the trade-off between retaining earlier tasks and fitting the newest one. The constant-$\lambda$ runs trace a clear frontier, from $\lambda=0.9$ retaining history at the cost of $T_4$ to $\lambda=0.25$ doing the reverse, with $\lambda=0.5$ and $0.75$ in between. Both scheduled policies sit near the uniform target, well inside the frontier that any single constant weight can reach, which is the practical case for scheduling $\lambda_k$ rather than tuning it once. \Cref{fig: app_fig4}b shows that the adaptive sweep selects weights closely tracking the equal-task rule $\lambda_k=k/(k+1)$ at every update, for both replay methods. This is expected, since the five modes carry equal mass and the balanced criterion is optimized by exactly the composition equal-task produces in closed form: the agreement validates the adaptive procedure rather than separating the two schedules, and explains why they perform comparably in \cref{fig: sequential_gaussian_simulation}b.\looseness=-1

\textbf{Stream evolution.} \Cref{fig: app_fig5} follows both metrics as the stream progresses rather than only at the end. Naive fine-tuning degrades within the first update and stays flat thereafter, since each update discards whatever the previous one retained. Both replay methods remain two to three orders of magnitude below it, and their curves separate progressively: generative replay drifts upward as tasks accumulate, on both energy distance and task-mass divergence, whereas trajectory replay stays roughly level. This is the compounding effect that follows from \cref{sec: partial_replay_analysis}: simulation error in the replayed samples enters the next historical model and is inherited by every subsequent update.

\begin{figure}[!htbp]
\centering
\includegraphics[width=0.8\textwidth]{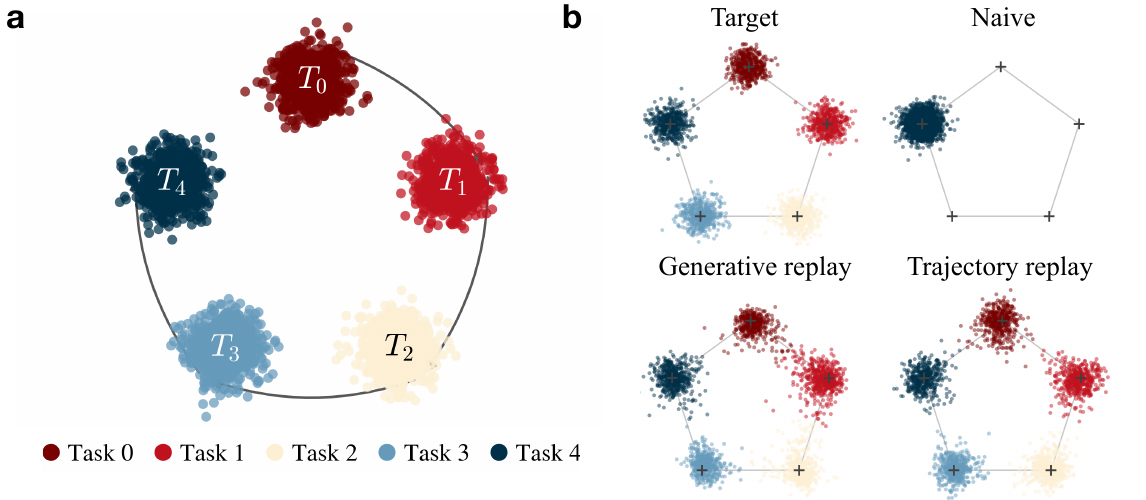}
\caption{\textbf{Five-task Gaussian ring stream.}
(\textbf{a}) The five task distributions $T_0,\dots,T_4$ placed on a ring, with the arrow indicating the order in which they are introduced.
(\textbf{b}) Samples drawn after the final update for the balanced target and for naive fine-tuning, generative replay, and trajectory replay. Crosses mark the task centers, and colors follow the task legend in (\textbf{a}).}
\label{fig: app_fig3}
\end{figure}

\begin{figure}[!htbp]
\centering
\includegraphics[width=\textwidth]{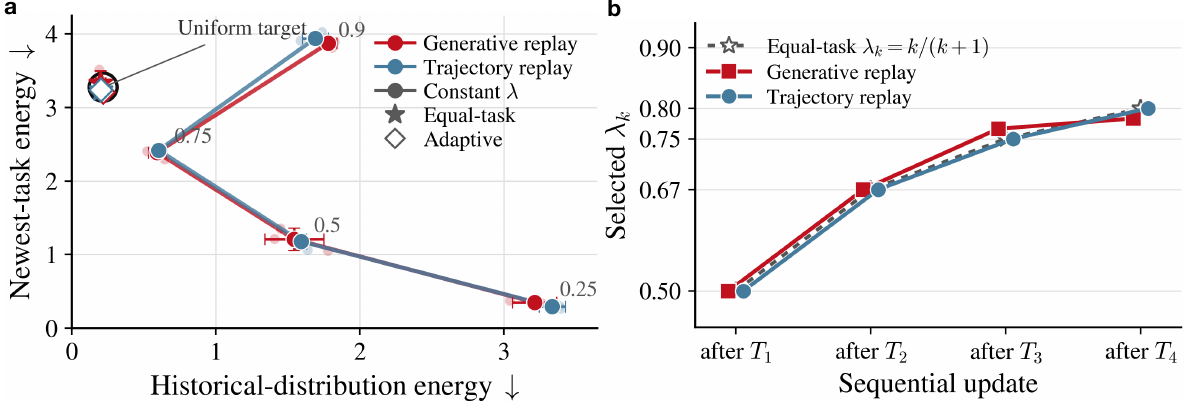}
\caption{\textbf{Replay-weight policies on the five-task stream.}
(\textbf{a}) Trade-off between retaining earlier tasks and fitting the newest one after the final update: each marker is one final model, positioned by its energy distance to the historical task distributions ($x$) and to the newest task ($y$), so points closer to the origin are better on both axes. Colors distinguish generative from trajectory replay, marker shapes distinguish constant-$\lambda$, equal-task, and adaptive policies, and the number next to each constant-$\lambda$ marker is its $\lambda$ value. The circled point near the origin marks the uniform target.
(\textbf{b}) Replay weight $\lambda_k$ selected at each sequential update for adaptive generative and trajectory replay, with the equal-task rule $\lambda_k=k/(k+1)$ shown as a dashed reference. Error bars show variability across seeds.}
\label{fig: app_fig4}
\end{figure}

\begin{figure}[!t]
\centering
\includegraphics[width=0.80\textwidth]{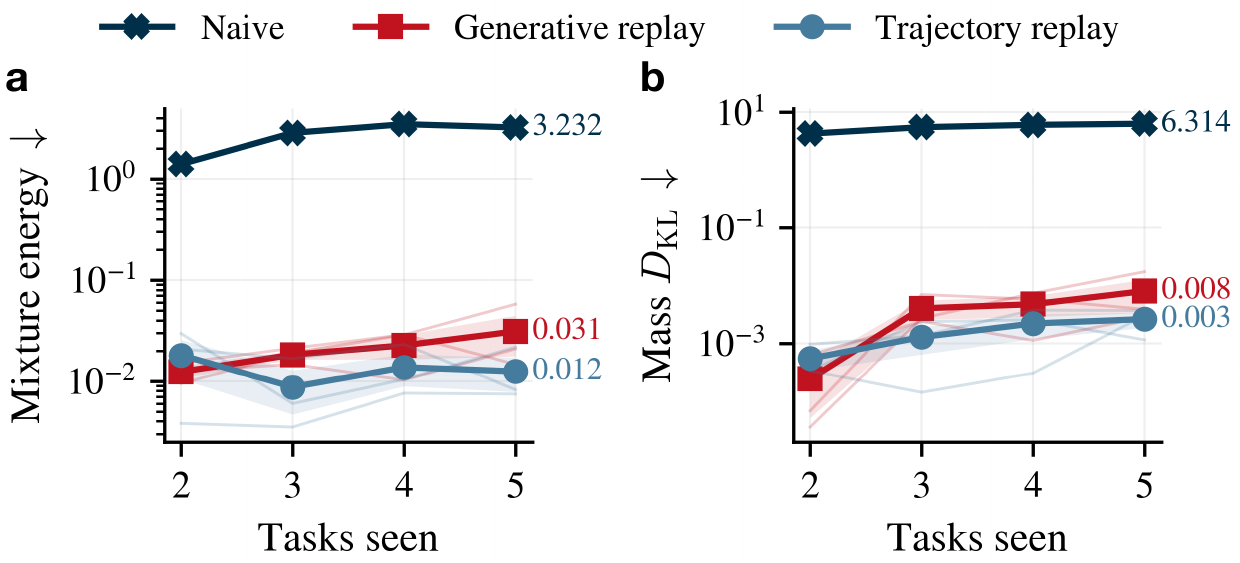}
\caption{\textbf{Stream evolution on the five-task ring benchmark.}
(\textbf{a}) Energy distance to the balanced mixture and (\textbf{b}) KL divergence of the generated task-mass composition from the uniform target, both as a function of the number of tasks seen, for naive fine-tuning, generative replay, and trajectory replay. Bold lines show the mean and faint lines three individual seeds. Numbers at the right give the final values, both axes are logarithmic, and lower is better.}
\label{fig: app_fig5}
\end{figure}

\FloatBarrier
\subsection{Class-incremental image generation}
\label{sec:exp-image-cl}

\subsubsection{Experimental details}\label{sec:exp-image-cl-details}

\textbf{Streams.} Each dataset is split into disjoint class groups arriving sequentially, with class orders drawn per seed via the Avalanche package \citep{carta2023avalanche} so that the stream is identical across methods. We run seeds $\{42,69,1714\}$ (same as \citet{masip2025continual}), each fixing the class order and training randomness. Split Fashion-MNIST and Split CIFAR-10 give five tasks of two classes each and share the same partition at a given seed. Split ImageNet-100 gives three tasks of $34/33/33$ classes. Within every task, a deterministic $10\%$ of the training split is held out for validation, and the official test split is reserved for reporting.

\textbf{Inputs.} Fashion-MNIST is resized to $32{\times}32$ single-channel and CIFAR-10 to $32{\times}32$ RGB, both normalized to $[-1,1]$, with train-time horizontal flips on CIFAR-10 only. ImageNet-100 is modeled in latent space: images are resized to $256{\times}256$ and encoded once with a frozen \texttt{stabilityai/sd-vae-ft-mse} autoencoder\footnote{\url{https://huggingface.co/stabilityai/sd-vae-ft-mse}} into $32{\times}32{\times}4$ scaled latents, cached to disk, with train-time flips. Latents are decoded to pixels only for evaluation and sample grids.

\textbf{Model.} The velocity field is a Diffusers UNet2D \citep{von-platen-etal-2022-diffusers} on the $32{\times}32$ state, trained on the linear noise-to-data interpolant with target $\bm{x}_1-\bm{x}_0$. Fashion-MNIST and CIFAR-10 reuse the medium configuration of \citet{masip2025continual}, with block widths $(128,256,256,256)$ and attention at one resolution ($35.8$M parameters), so our numbers are directly comparable to their diffusion baselines.

ImageNet-100 was not considered by \citet{masip2025continual}, so we selected its backbone and budget with a task-0 pilot crossing the medium configuration against a large latent-diffusion-style one, of widths $(128,256,512,512)$ and attention at the two lowest resolutions ($106.4$M parameters), over optimizer budgets of $\{10,20,50,100,200\}$k steps, picking the lowest FID on the task-0 classes (\cref{fig: app_fig10}). This selects the large backbone at $100$k steps. The same pilot crossed conditional against unconditional variants, and we model ImageNet-100 \textit{class-conditionally} since the unconditional runs were clearly worse at this scale, with label dropout $0.1$ and classifier-free guidance \citep{zheng2023guided} $1.5$ at sampling time.\looseness=-1

\begin{figure}[!htbp]
\centering
\includegraphics[width=0.55\textwidth]{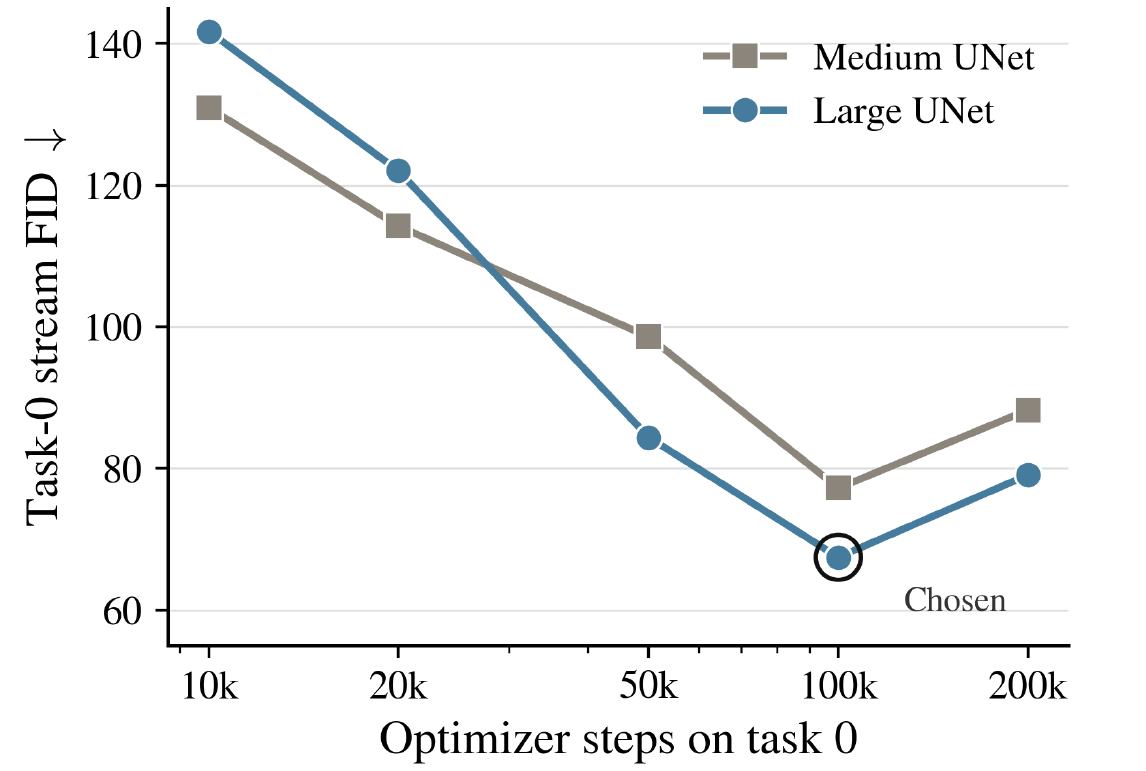}
\caption{\textbf{Task-0 budget and backbone selection on ImageNet-100.}
FID on the task-0 classes as a function of the number of optimizer steps on task~0, for two latent UNet backbones: the medium model with block widths $(128,256,256,256)$ and attention at one resolution ($35.8$M parameters), and the large LDM-style model with widths $(128,256,512,512)$ and attention at the two lowest resolutions ($106.4$M parameters). The circled point marks the configuration chosen for the continual-learning experiments; lower is better.}
\label{fig: app_fig10}
\end{figure}

\textbf{Training.} Every task is trained with Adam under the per-dataset budget in \cref{tab:image_cl_hparams}. Joint training reuses real data from all tasks seen so far: Fashion-MNIST and CIFAR-10 hold the epoch count fixed as the union grows, whereas ImageNet-100 holds the optimizer-step count fixed, so per-task compute does not grow with the stream.

\textbf{Replay.} Both families share the stream and budget and differ only in what the historical term regresses onto.
\begin{itemize}[leftmargin=*, itemindent=0pt]
  \item \emph{Generative replay} draws a fraction $\lambda$ of each batch from the frozen previous-task model and trains a single flow matching loss on the combined batch.
  \item \emph{Trajectory replay} splits each batch evenly, samples one shared time per step, integrates the teacher only to that time, and matches the student field against the teacher field there, with $\lambda$ weighting the historical and incoming terms.
\end{itemize}
On ImageNet-100, both condition the historical branch on labels cycled \textit{uniformly} over all classes seen so far, and generate without guidance, unlike evaluation.

The equal-task policy sets $\lambda_k=k/(k+1)$, and the adaptive policy selects $\lambda$ from the per-dataset grid by validation FID on $2{,}048$ samples. The ODE backend integrates the teacher with Euler at the replay step size in \cref{tab:image_cl_hparams}. The MeanFlow backend replaces it with a distilled mean-velocity model using $N\in\{1,4\}$ jumps, distilling task~0 post hoc at learning rate $10^{-4}$ with gradient clipping $1.0$ for the per-dataset epochs in the table, then co-training the field and the mean velocity on later tasks with logit-normal time pairs, as in \citet{geng2026mean} (see \cref{alg:trajectory_replay_meanflow}). 

\textbf{Adaptive loss weighting.} Task~0 distillation also uses adaptive loss weighting with the exponent in \cref{tab:image_cl_hparams}, but sequential updates do not: rescaling each branch by its own error would change their relative weight away from $\lambda$ and $1-\lambda$, and the fixed point of \cref{prop: mean_flow_minimizer} would no longer be the mixture mean velocity.

\textbf{Baselines.} We report a series of reference points, all on the same stream, per-task budget, backbone, and evaluation protocol as our methods, following the scheme of \citet{masip2025continual}.
\begin{itemize}[leftmargin=*, itemindent=0pt]
  \item \emph{Joint training} retrains on the union of all real data seen so far, storing every past sample, and is therefore an upper target rather than a competing method.
  \item \emph{Naive fine-tuning} continues training on the incoming task alone with no retention mechanism, giving the lower target. Since the baselines below use a different generative objective, we report it for both, as \emph{diffusion naive fine-tuning} and \emph{flow matching naive fine-tuning}, separating the gain attributable to replay from that attributable to the objective.
  \item \emph{Generative replay (diffusion)} is the diffusion counterpart of generative replay above.
  \item \emph{Generative distillation (diffusion)} \citep{masip2025continual} also replays from the frozen previous-task model, but supervises with the teacher's own noise prediction at a randomly drawn timestep rather than with the injected noise, so the student matches the teacher's denoising behavior along the whole trajectory instead of only reproducing its final samples. This is the strongest published no-storage continuous-time baseline on these streams.
 \item \emph{LoRA} \citep{hu2022lora} and \emph{gradual fine-tuning} (GFT) \citep{thorkelsdottir2026gradual} are adapter and regularization baselines, reported separately in \cref{tab:image_cl_lora_gft}. LoRA adapts the velocity field with rank $r\in\{8,16,32\}$ and $\alpha=2r$, and GFT self-distills from the frozen previous model with a distillation weight $\beta$ annealed from $1$ to $0$, under sigmoid and linear schedules.
\end{itemize}

\textbf{Evaluation.} After the final task we generate $10{,}000$ samples and report stream FID in Inception feature space \citep{szegedy2016rethinking} against pooled held-out images from all seen classes. Fashion-MNIST and CIFAR-10 are modeled unconditionally, so class composition is emergent and forgetting shows up as mass collapsing onto recent classes. Class-mass recovery (KLD) compares the classifier-predicted histogram against the target composition, using the pretrained convolutional neural network (CNN) classifiers released with the diffusion baselines \citep{masip2025continual}, and isolates retention from the sample quality captured by FID. For ImageNet-100, labels are drawn from the target composition at sampling time, so the histogram matches by construction, and we report FID alone. All numbers are means with standard errors over seeds.

\begin{table}[!ht]
\centering
\definecolor{grouprow}{HTML}{F0F0F0}
\setlength{\tabcolsep}{4.5pt}
\small
\caption{Per-dataset configuration for the class-incremental image experiments. Sampler and replay steps are Euler step sizes, with the resulting number of function evaluations in parentheses.}
\label{tab:image_cl_hparams}
\resizebox{\textwidth}{!}{%
\begin{tabular}{@{}lccc@{}}
\toprule
& Fashion-MNIST & CIFAR-10 & ImageNet-100 \\
\midrule
\rowcolor{grouprow} \multicolumn{4}{@{}l}{\textit{Data}} \\
State & $1{\times}32{\times}32$ & $3{\times}32{\times}32$ & $4{\times}32{\times}32$ latent \\
Tasks $\times$ classes & $5\times2$ & $5\times2$ & $3\times(34/33/33)$ \\
Horizontal flips & no & yes & yes \\
\midrule
\rowcolor{grouprow} \multicolumn{4}{@{}l}{\textit{Model}} \\
UNet & medium & medium & large \\
Parameters & $35.8$M & $35.8$M & $106.4$M \\
Class-conditional & no & no & yes (CFG $1.5$) \\
\midrule
\rowcolor{grouprow} \multicolumn{4}{@{}l}{\textit{Optimization}} \\
Epochs / task & $100$ & $200$ & $667$ \\
Batch size & $128$ & $64$ & $256$ \\
Learning rate & $2\times10^{-4}$ & $2\times10^{-4}$ & $1\times10^{-4}$ \\
Gradient clip & none & none & $1.0$ \\
Joint schedule & fixed epochs & fixed epochs & fixed steps \\
\midrule
\rowcolor{grouprow} \multicolumn{4}{@{}l}{\textit{Replay and sampling}} \\
Sampler step (NFE) & $0.1$ ($10$) & $0.05$ ($20$) & $0.02$ ($50$) \\
Replay step (NFE) & $0.1$ ($10$) & $0.05$ ($20$) & $0.033$ ($30$) \\
$\lambda$ grid & $\{0.1,0.25,0.5,0.75,0.9\}$ & $\{0.25,0.5,0.75\}$ & $\{0.25,0.5,0.75\}$ \\
MeanFlow $(\mu,\sigma)$ \citep{geng2026mean} & $(2.0,2.0)$ & $(2.0,2.0)$ & $(0.4,1.0)$ \\
MeanFlow diagonal rate \citep{geng2026mean} & $0.25$ & $0.25$ & $0.75$ \\
MeanFlow adaptive loss weight ($T_0$ only) \citep{geng2026mean} & $0.75$ & $0.75$ & $1.0$ \\
MeanFlow distillation epochs & $100$ & $200$ & $150$ \\
\midrule
\rowcolor{grouprow} \multicolumn{4}{@{}l}{\textit{Evaluation}} \\
Validation samples & $2{,}048$ & $2{,}048$ & $2{,}048$ \\
Test samples & $10{,}000$ & $10{,}000$ & $10{,}000$ \\
Reported metrics & FID, KLD & FID, KLD & FID \\
\bottomrule
\end{tabular}}
\end{table}

\FloatBarrier
\subsubsection{Additional results}

We report wall-time gains, degradation across the stream, the weights the adaptive policy selects, the quality-budget trade-off, comparisons with adapter and regularization baselines, when the MeanFlow backend beats the ODE teacher, and qualitative samples.

\textbf{Runtime comparison.} \Cref{fig: app_fig24} compares the two replay families over the full CL pipeline under equal-task $\lambda$. Trajectory replay is faster on every stream, by $1.22\times$ on Fashion-MNIST, $1.51\times$ on CIFAR-10, and $1.65\times$ on ImageNet-100. The margin grows with the cost of the teacher: integrating to an intermediate time saves a fraction of the trajectory fixed by $\mathbb{E}[t]$, so the absolute saving scales with how expensive a full generation is, from $0.2$h on Fashion-MNIST to $25.5$h on ImageNet-100. The ratio is capped at a factor of two in \cref{sec: partial_replay_analysis}, but the hours saved keep growing.

\begin{figure}[!htbp]
\centering
\includegraphics[width=0.8\textwidth]{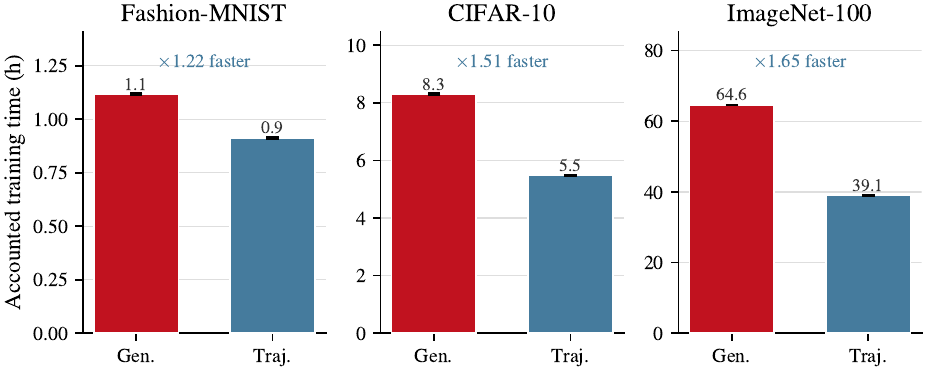}
\caption{\textbf{Training wall-time comparison}. Full-CL-pipeline comparison of trajectory and generative replay in terms of accounted wall time across datasets. The two methods are compared in the equal-task $\lambda$ setting.}
\label{fig: app_fig24}
\end{figure}

\FloatBarrier
\textbf{Stream evolution.} \Cref{fig: app_fig6} tracks both FID and KLD metrics across the stream rather than only at the final task. Naive fine-tuning worsens monotonically everywhere. The replay families diverge in how they respond to the weighting policy. Generative replay under equal-task weighting degrades with each update, ending far above its adaptive counterpart on Fashion-MNIST and CIFAR-10. Trajectory replay is nearly flat under both. It therefore reaches comparable quality with the parameter-free schedule, which matters where a validation sweep at every update is the dominant cost. 

\Cref{fig: app_fig7} shows the resulting compositions: naive fine-tuning retains almost no mass on early classes, and trajectory replay is closest to the target histogram on both datasets, though its KLD remains higher than joint training (\cref{tab:image_cl_all}) and every method still favors recent classes. Unbiased retention across a long stream is not solved by reweighting alone and remains for \textit{future work}.

\begin{figure}[!htbp]
\centering
\includegraphics[width=0.90\textwidth]{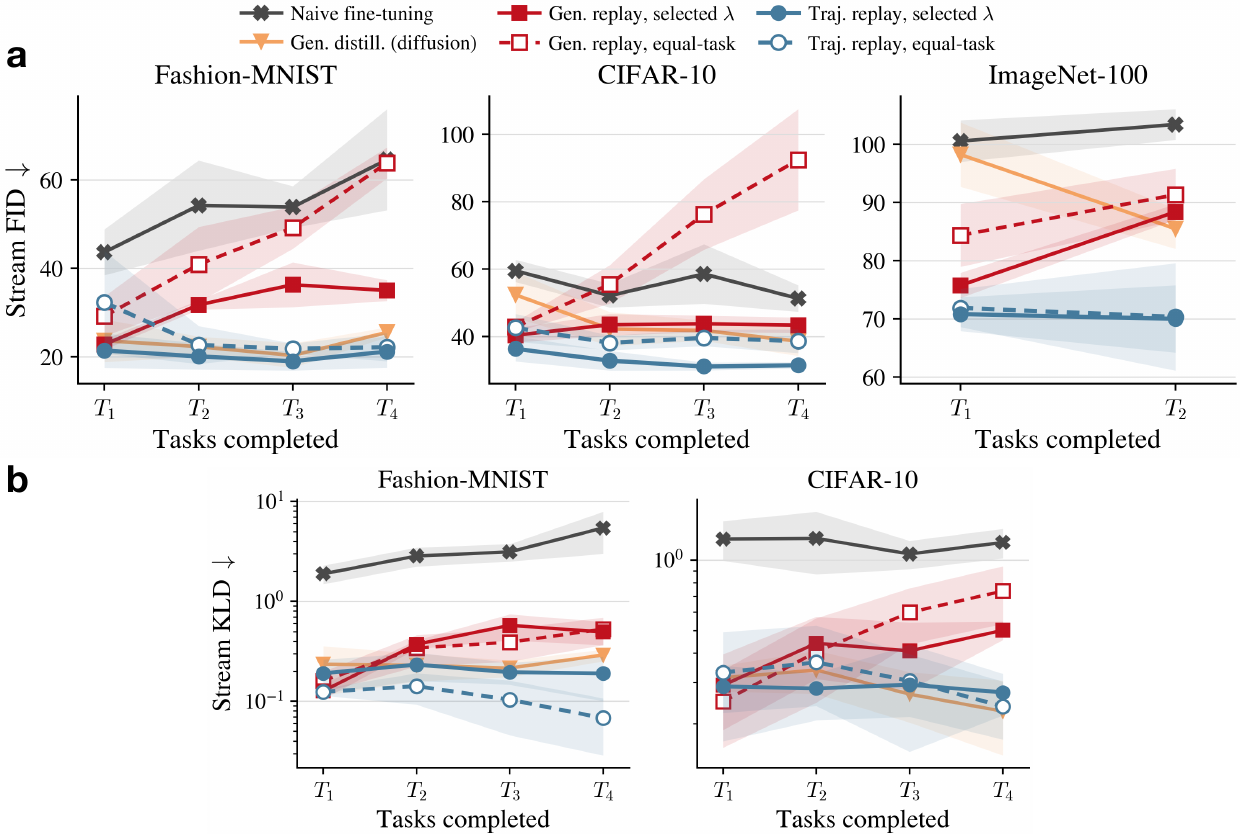}
\caption{\textbf{Class-incremental image streams.}
(\textbf{a}) Stream FID and (\textbf{b}) stream KLD as a function of tasks completed, for naive fine-tuning, diffusion distillation, and generative and trajectory replay under selected and equal-task $\lambda$. FID is shown for Fashion-MNIST, CIFAR-10, and ImageNet-100. KLD for Fashion-MNIST and CIFAR-10. Lines show the mean and shaded bands the variability across three seeds. Filled markers with solid lines denote selected $\lambda$ and open markers with dashed lines equal-task $\lambda$, and lower is better on both metrics.}
\label{fig: app_fig6}
\end{figure}

\begin{figure}[!htbp]
\centering
\includegraphics[width=0.9\textwidth]{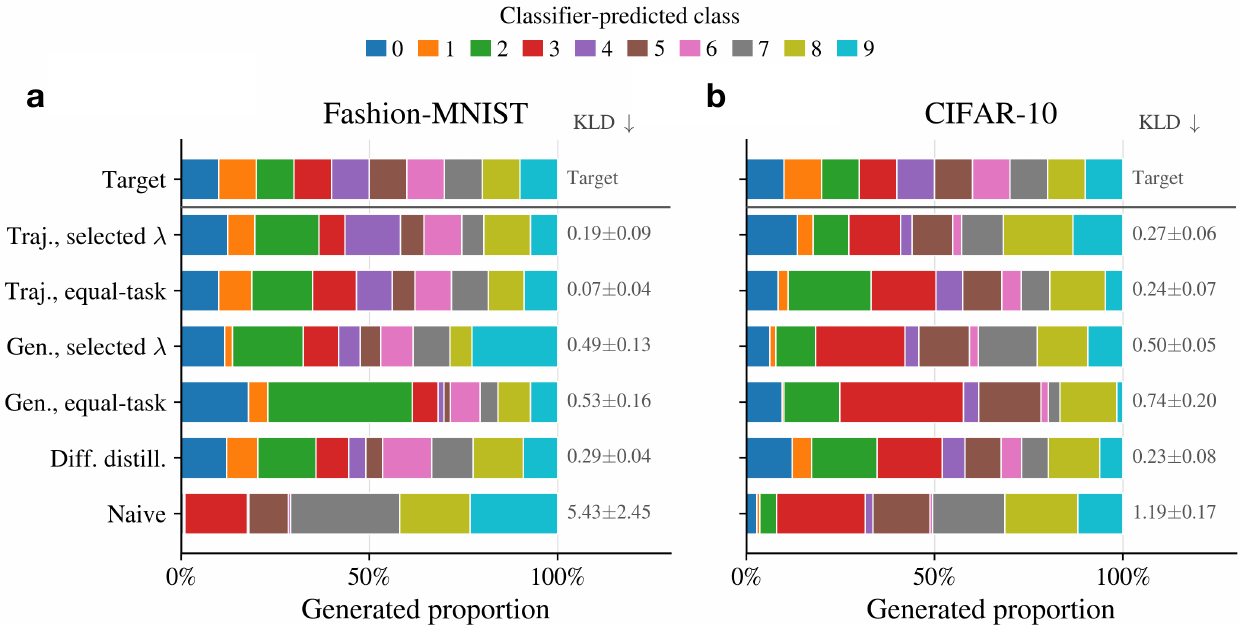}
\caption{\textbf{Class composition of generated samples after the final task.}
Stacked bars give the proportion of generated samples assigned to each class by the evaluation classifier, for (\textbf{a}) Fashion-MNIST and (\textbf{b}) CIFAR-10. Rows list the uniform target and each method: trajectory and generative replay under selected and equal-task $\lambda$, diffusion distillation, and naive fine-tuning. Numbers at the right report stream KLD from the target composition (mean $\pm$ SE over three seeds). Lower is better.}
\label{fig: app_fig7}
\end{figure}
\FloatBarrier

\textbf{Selected replay weights.} \Cref{fig: app_fig8} reports which weights the adaptive $\lambda$-selection policy chooses. The criterion scores against a balanced reference, so the uniform composition that equal-task produces is what it rewards. On Fashion-MNIST and CIFAR-10, trajectory replay tracks that rule, rising with the number of tasks seen, whereas generative replay settles well below it and non-monotonically. Since $\lambda$ sets how much mass the update assigns to history, the criterion appears to discount generative replay's own samples, whose fidelity is low enough that weighting them as heavily as real tasks costs more than the retention it buys. Trajectory replay faces no such trade-off, as its historical term regresses onto the frozen field rather than onto generated data. The ImageNet-100 panel runs the other way, though with only two updates and wide error bars.

\begin{figure}[!htbp]
\centering
\includegraphics[width=1\textwidth]{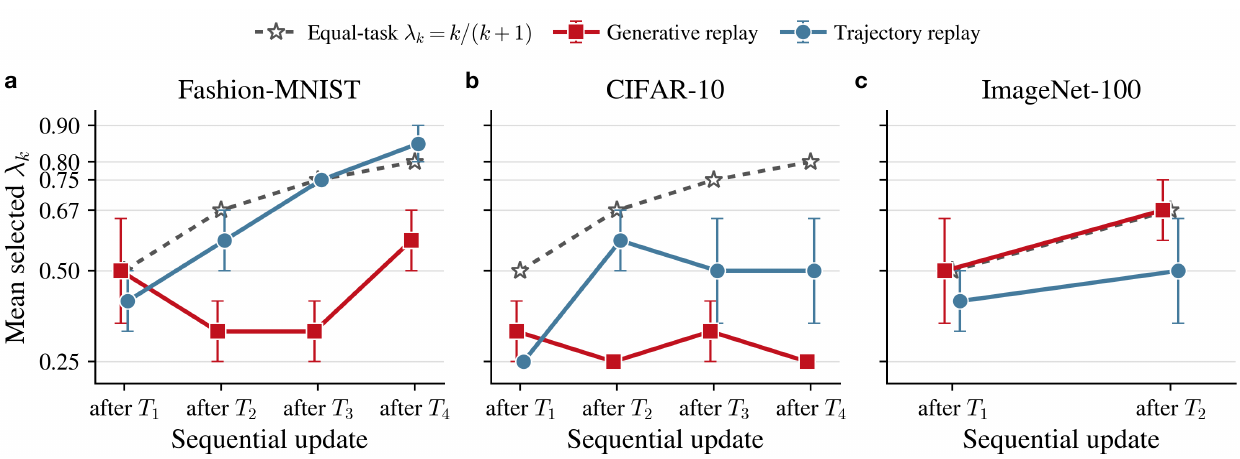}
\caption{\textbf{Selected replay weights across the image streams.}
Mean replay weight $\lambda_k$ chosen at each sequential update for generative and trajectory replay on (\textbf{a}) Fashion-MNIST, (\textbf{b}) CIFAR-10, and (\textbf{c}) ImageNet-100, with the equal-task rule $\lambda_k=k/(k+1)$ shown as a dashed reference. Error bars show variability across three seeds.}
\label{fig: app_fig8}
\end{figure}
\FloatBarrier

\textbf{Does the MeanFlow backend cost quality?} \Cref{fig: app_fig9} asks what replacing the ODE teacher with a distilled mean velocity gives up and what it saves. Panel a normalizes final FID by the matched ODE run at each backend's replay-field evaluations per update. Trajectory replay sits at parity across the range, so MF-1 and MF-4 \textit{lose essentially nothing} relative to full integration and one or two evaluations suffice. Generative replay instead starts well above parity at the smallest budget and only recovers as evaluations accumulate, as expected: a replay sample must be integrated to $t=1$ before it can serve as data, whereas trajectory replay supervises at an intermediate time where a coarse trajectory is still informative. Panel b prices the same endpoint in accounted training time: the wall-clock cost of an update including replay sampling and, for the MeanFlow backends, the amortized cost of distilling the teacher. The gain is dataset-dependent: on Fashion-MNIST and CIFAR-10 the MeanFlow backends save little because the ODE teacher is already cheap, whereas on ImageNet-100 they reach comparable FID at a fraction of the time.

\begin{figure}[!htbp]
\centering
\includegraphics[width=1\textwidth]{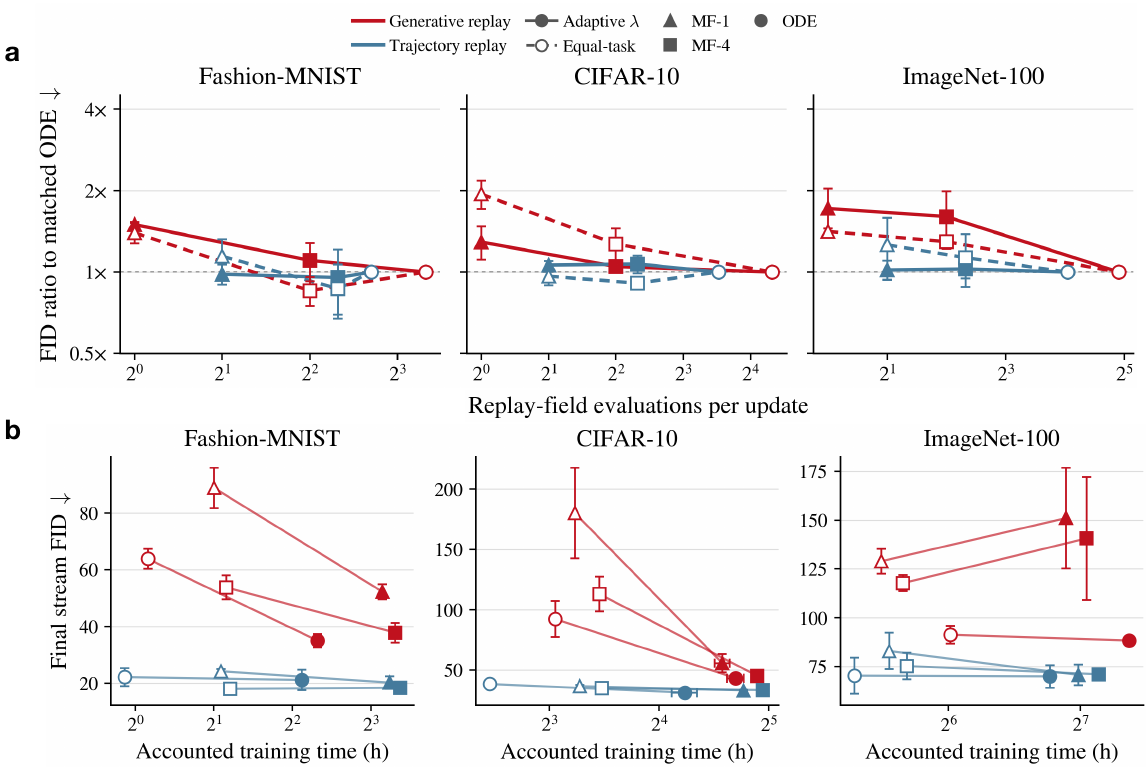}
\caption{\textbf{Replay backend cost and quality on the image streams.}
(\textbf{a}) Final stream FID relative to the matched ODE run as a function of replay-field evaluations per update, and (\textbf{b}) final stream FID against accounted training time, for Fashion-MNIST, CIFAR-10, and ImageNet-100. Colors distinguish generative from trajectory replay, marker shapes distinguish the MF-1, MF-4, and ODE backends, and filled markers with solid lines denote adaptive $\lambda$ while open markers with dashed lines denote equal-task $\lambda$. The dotted line in (\textbf{a}) marks parity with the matched ODE run. Error bars show variability across three seeds, $x$-axes are logarithmic, and lower is better.}
\label{fig: app_fig9}
\end{figure}
\FloatBarrier

\textbf{Parameter-efficient and regularization baselines.} \Cref{tab:image_cl_lora_gft} reports LoRA and gradual fine-tuning (GFT) under the same protocol. Neither is competitive: the best configuration reaches $65.1$, $50.0$, and $104.6$ FID against $18.0$, $31.4$, and $70.0$ for trajectory replay, with class-mass divergences up to two orders of magnitude larger. Capacity does not help, as raising the LoRA rank from $8$ to $32$ moves FID by at most $5.2$ points, and the two $\beta$ schedules are indistinguishable. LoRA bounds how far the velocity field may move, and gradual fine-tuning anchors it to the previous model only transiently, annealing that anchor away, so neither retains a target for the historical tasks by the end of the update, mirroring the transcriptomics streams (\cref{tab:replogle_cl_full,tab:pbmc_cl_full}).

\textbf{When the mean velocity pays off.} \Cref{fig: app_fig14} sweeps the number of Euler integration steps used by the ODE teacher on Fashion-MNIST, from $5$ to $100$. Final FID is flat beyond roughly ten steps (panel a), so a coarse teacher loses nothing here, while the teacher's amortized cost per update grows with the step count and the MeanFlow teachers cost a \textit{fixed amount set by co-training} rather than by the number of replay steps (panel b).

The two therefore cross at a moderate teacher budget: on Fashion-MNIST, where ten steps already saturate quality, the crossing sits above the operating regime, and the ODE backend is cheaper, whereas on ImageNet-100 (\cref{fig: app_fig9}b) the teacher is expensive enough that the MeanFlow backends match FID at a fraction of the cost. The mean velocity is thus not uniformly cheaper: it removes the dependence of replay cost on teacher fidelity, and pays off where that fidelity is expensive.

\begin{figure}[!htbp]
\centering
\includegraphics[width=0.8\textwidth]{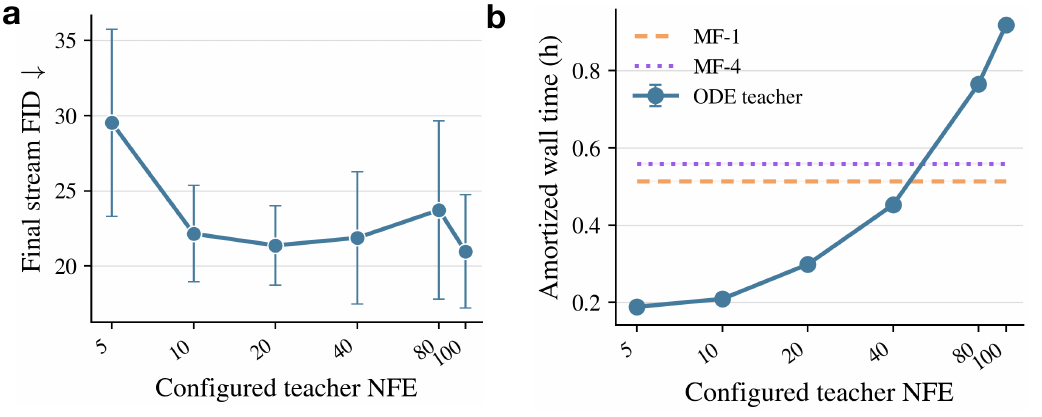}
\caption{\textbf{Teacher NFE sweep on Fashion-MNIST.}
(\textbf{a}) Final stream FID and (\textbf{b}) amortized wall time per update as a function of the configured teacher NFE for the ODE teacher. Dashed and dotted horizontal lines in (\textbf{b}) mark the corresponding cost of the MF-1 and MF-4 teachers. Error bars show variability across three seeds and the $x$-axis is logarithmic.}
\label{fig: app_fig14}
\end{figure}

\begin{table*}[t]
\centering
\definecolor{trajectoryrow}{HTML}{E5EEF3}
\setlength{\tabcolsep}{4.5pt}
\small
\caption{Memory-free baselines for class-incremental image generation.
Final-stream means with standard errors over three seeds after five tasks
(Fashion-MNIST and CIFAR-10) or three tasks (ImageNet-100). LoRA uses
$\alpha=2r$; GFT decays the teacher-regularization weight $\beta$ from 1 to 0.
Blue rows show trajectory-replay references (adaptive or equal-task $\lambda$;
MF-$N$ uses $N$ replay steps). Lower is better.}
\label{tab:image_cl_lora_gft}
\resizebox{1\textwidth}{!}{%
\begin{tabular}{@{}lccccc@{}}
\toprule
Method & \multicolumn{2}{c}{Fashion-MNIST} & \multicolumn{2}{c}{CIFAR-10} & ImageNet-100 \\
\cmidrule(lr){2-3}\cmidrule(lr){4-5}\cmidrule(l){6-6}
& $\downarrow$ FID & $\downarrow$ KLD & $\downarrow$ FID & $\downarrow$ KLD & $\downarrow$ FID \\
\midrule
LoRA ($r{=}8$, $\alpha{=}16$) & $68.4$ {\scriptsize $(\pm\,11.1)$} & $4.508$ {\scriptsize $(\pm\,1.956)$} & $50.9$ {\scriptsize $(\pm\,3.4)$} & $1.006$ {\scriptsize $(\pm\,0.114)$} & $109.8$ {\scriptsize $(\pm\,0.5)$} \\
LoRA ($r{=}16$, $\alpha{=}32$) & $67.3$ {\scriptsize $(\pm\,10.9)$} & $5.129$ {\scriptsize $(\pm\,2.154)$} & $50.2$ {\scriptsize $(\pm\,3.2)$} & $1.055$ {\scriptsize $(\pm\,0.123)$} & $107.6$ {\scriptsize $(\pm\,0.9)$} \\
LoRA ($r{=}32$, $\alpha{=}64$) & $66.2$ {\scriptsize $(\pm\,11.2)$} & $5.107$ {\scriptsize $(\pm\,2.146)$} & $50.0$ {\scriptsize $(\pm\,3.0)$} & $1.069$ {\scriptsize $(\pm\,0.115)$} & $104.6$ {\scriptsize $(\pm\,2.3)$} \\
\midrule
GFT (sigmoid $\beta$ schedule) & $65.1$ {\scriptsize $(\pm\,12.0)$} & $4.264$ {\scriptsize $(\pm\,1.309)$} & $50.7$ {\scriptsize $(\pm\,4.0)$} & $1.138$ {\scriptsize $(\pm\,0.171)$} & $119.5$ {\scriptsize $(\pm\,2.7)$} \\
GFT (linear $\beta$ schedule) & $65.8$ {\scriptsize $(\pm\,12.8)$} & $4.873$ {\scriptsize $(\pm\,2.188)$} & $51.0$ {\scriptsize $(\pm\,4.3)$} & $1.101$ {\scriptsize $(\pm\,0.176)$} & $119.1$ {\scriptsize $(\pm\,3.4)$} \\
\midrule
\rowcolor{trajectoryrow} Trajectory replay (ODE, adaptive $\lambda$) & $21.1$ {\scriptsize $(\pm\,3.6)$} & $0.190$ {\scriptsize $(\pm\,0.090)$} & $31.4$ {\scriptsize $(\pm\,0.9)$} & $0.272$ {\scriptsize $(\pm\,0.055)$} & $70.0$ {\scriptsize $(\pm\,5.8)$} \\
\rowcolor{trajectoryrow} Trajectory replay (ODE, equal-task $\lambda$) & $22.1$ {\scriptsize $(\pm\,3.2)$} & $0.068$ {\scriptsize $(\pm\,0.039)$} & $38.6$ {\scriptsize $(\pm\,3.5)$} & $0.237$ {\scriptsize $(\pm\,0.066)$} & $70.4$ {\scriptsize $(\pm\,9.2)$} \\
\rowcolor{trajectoryrow} Trajectory replay (MF-1, adaptive $\lambda$) & $20.2$ {\scriptsize $(\pm\,2.2)$} & $0.140$ {\scriptsize $(\pm\,0.008)$} & $33.3$ {\scriptsize $(\pm\,0.5)$} & $0.332$ {\scriptsize $(\pm\,0.066)$} & $70.7$ {\scriptsize $(\pm\,5.3)$} \\
\rowcolor{trajectoryrow} Trajectory replay (MF-1, equal-task $\lambda$) & $24.2$ {\scriptsize $(\pm\,0.8)$} & $0.065$ {\scriptsize $(\pm\,0.027)$} & $36.8$ {\scriptsize $(\pm\,1.2)$} & $0.237$ {\scriptsize $(\pm\,0.089)$} & $83.0$ {\scriptsize $(\pm\,9.3)$} \\
\rowcolor{trajectoryrow} Trajectory replay (MF-4, adaptive $\lambda$) & $18.4$ {\scriptsize $(\pm\,1.5)$} & $0.120$ {\scriptsize $(\pm\,0.023)$} & $33.5$ {\scriptsize $(\pm\,1.7)$} & $0.349$ {\scriptsize $(\pm\,0.056)$} & $71.1$ {\scriptsize $(\pm\,2.7)$} \\
\rowcolor{trajectoryrow} Trajectory replay (MF-4, equal-task $\lambda$) & $18.0$ {\scriptsize $(\pm\,1.7)$} & $0.090$ {\scriptsize $(\pm\,0.030)$} & $35.1$ {\scriptsize $(\pm\,3.1)$} & $0.211$ {\scriptsize $(\pm\,0.023)$} & $75.3$ {\scriptsize $(\pm\,6.8)$} \\
\bottomrule
\end{tabular}%
}
\end{table*}

\FloatBarrier

\textbf{Qualitative samples.} \Cref{fig:app_fig11,fig: app_fig12,fig: app_fig13} show held-out and generated samples for all three streams, where trajectory replay retains recognizable class structure on early tasks that generative replay renders with visible degradation, especially on ImageNet-100.
\FloatBarrier

\begin{figure}[!htbp]
\centering
\includegraphics[width=0.7\textwidth]{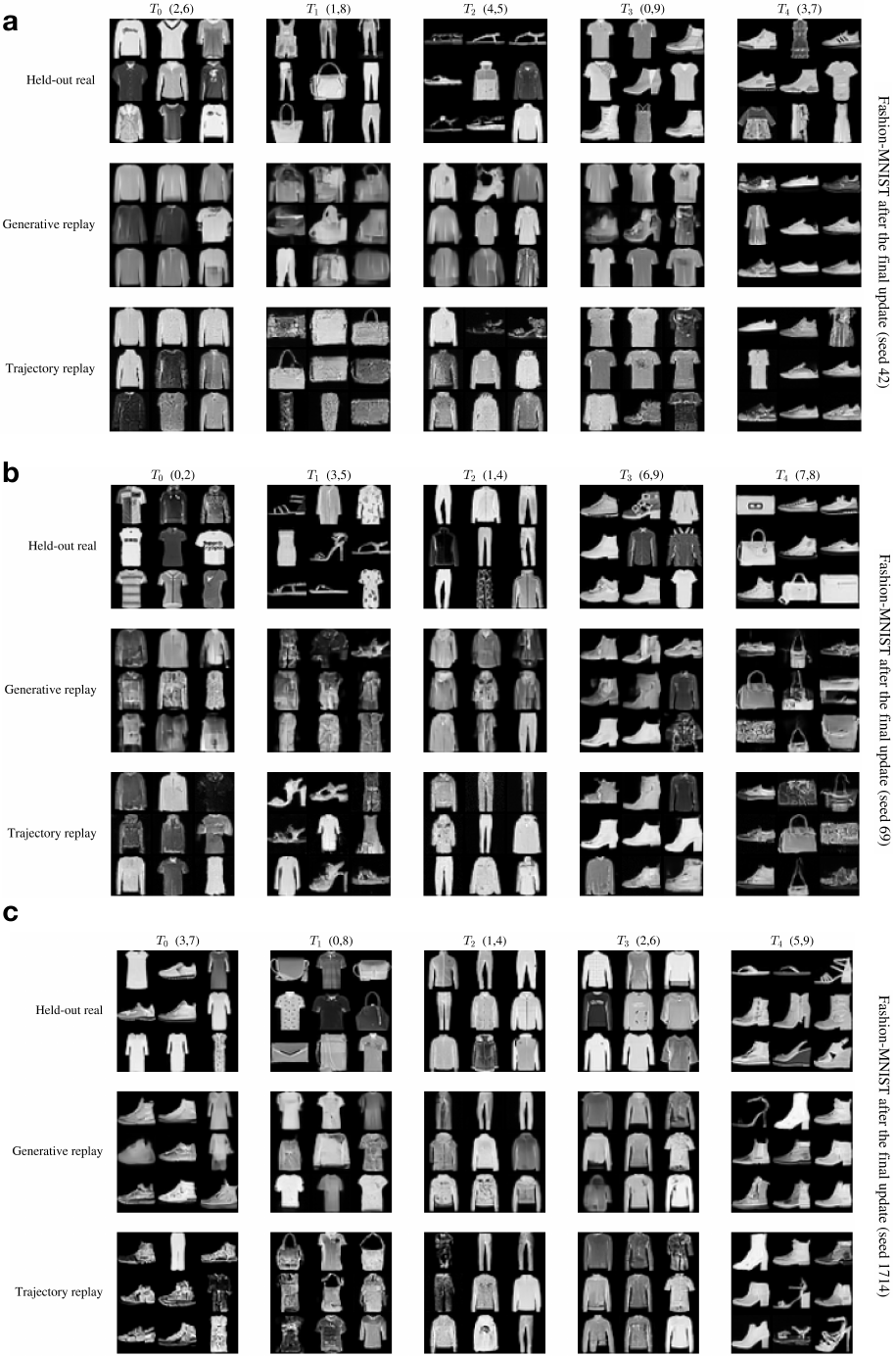}
\caption{\textbf{Fashion-MNIST samples after the final update.}
Held-out real images and samples from generative and trajectory replay, shown for seeds (\textbf{a}) 42, (\textbf{b}) 69, and (\textbf{c}) 1714. Columns correspond to the tasks of that seed's class stream, with the task index and its two classes given above each block. Rows within a block share the same task.}
\label{fig:app_fig11}
\end{figure}

\begin{figure}[!htbp]
\centering
\includegraphics[width=0.85\textwidth]{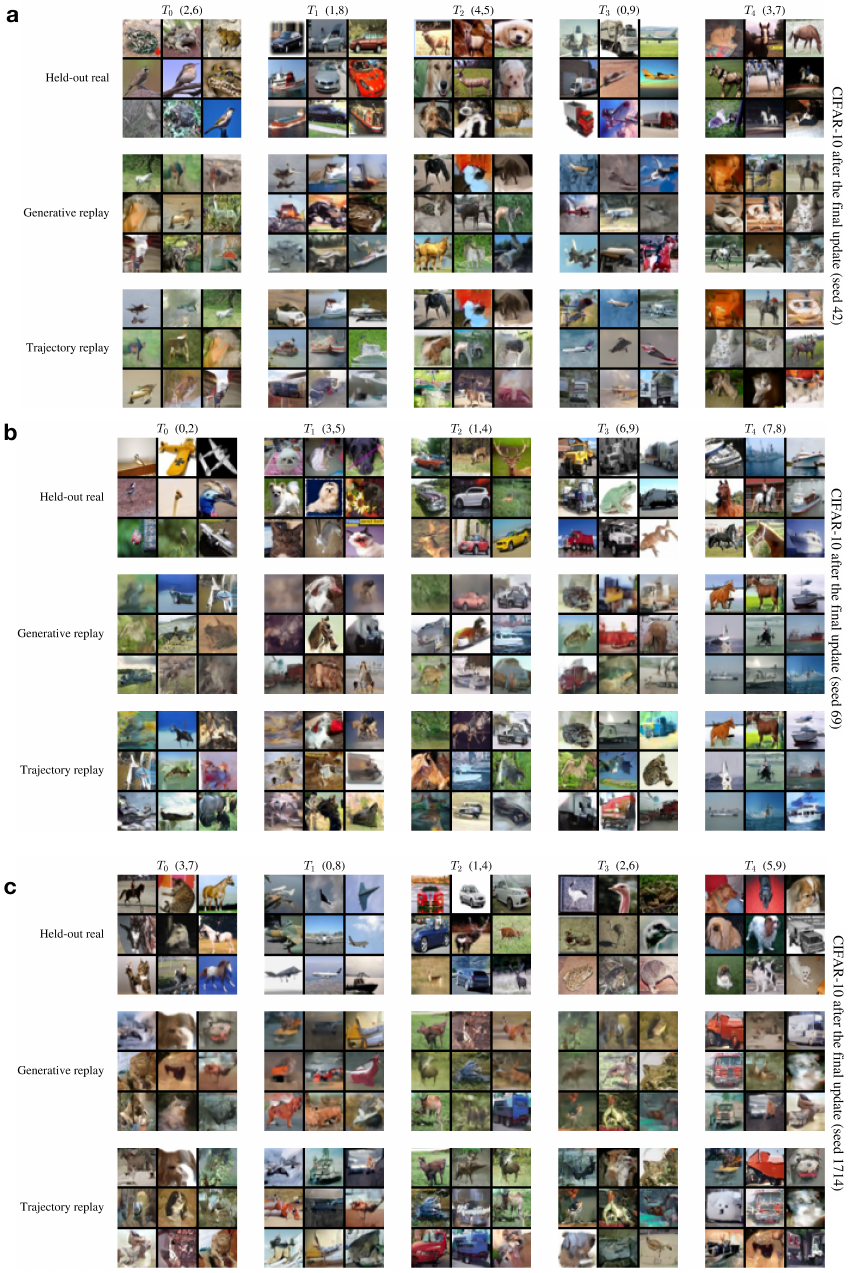}
\caption{\textbf{CIFAR-10 samples after the final update.}
Held-out real images and samples from generative and trajectory replay, shown for seeds (\textbf{a}) 42, (\textbf{b}) 69, and (\textbf{c}) 1714. Columns correspond to the tasks of that seed's class stream, with the task index and its two classes given above each block. Rows within a block share the same task.}
\label{fig: app_fig12}
\end{figure}

\begin{figure}[!htbp]
\centering
\includegraphics[width=0.5\textwidth]{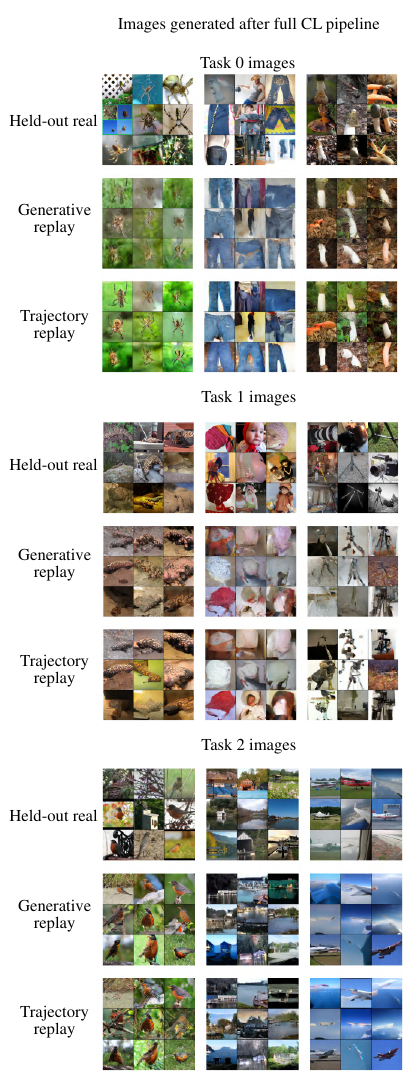}
\caption{\textbf{ImageNet-100 samples after the final update.}
Held-out real images and samples from generative and trajectory replay, grouped by the three tasks of the stream. Within each group, rows share the same source and columns show three classes drawn from that task.}
\label{fig: app_fig13}
\end{figure}

\clearpage
\subsection{Synthetic perturbation shift generalization}
\label{app:synthetic-shift-gen}
\subsubsection{Experimental details}

\textbf{Task.} Data arrive as pairs $(c,y)$ of a context $c$ and a perturbation $y$, and a generative model approximates $p(\bm{x}\mid c,y)$, mirroring \cref{sec:results-transcriptomics} in a controlled setting where the ground truth is known. Contexts shift the Gaussian mean, and perturbations add a fixed offset shared across contexts, so perturbation effects are transferable by construction. The update introduces a new context but withholds one perturbation within it, leaving a single combination $(c^\star,y^\star)$ never observed in either phase. Recovering it requires (\textbf{I}) retaining the perturbation direction learned on historical contexts and (\textbf{II}) fitting the basal position of the new one, so held-out error reads out the balance between the two rather than either alone.

\textbf{Data.} We build a $6\times4$ grid of conditions in $\mathbb{R}^2$: six contexts $c$ placed on a hexagon of radius $4$, and four perturbations $y$ applied as an additive offset of scale $1.5$ shared across contexts, with $y=0$ the unperturbed \textit{control}. Each cell $(c,y)$ is an isotropic Gaussian with $\sigma=0.25$, providing $800$ training and $200$ validation points, and the single held-out cell provides $1{,}000$ test points, with data seed $0$.

\textbf{Stream.} Pretraining uses contexts $0$--$4$ with all four perturbations. The update introduces context $5$ with perturbations $0$--$2$ only, so the held-out combination is $(c^\star,y^\star)=(5,3)$.

\textbf{Model.} The velocity field is a conditional MLP of width $256$ and depth $4$, with a $64$-dimensional sinusoidal time embedding and $32$-dimensional learned embeddings for context and perturbation, trained on the linear noise-to-data interpolant. Sampling uses Euler at step size $0.05$ ($20$ evaluations).

\textbf{Pretraining.} The context-$0$--$4$ model trains for $1{,}000$ epochs with batch size $1024$ and Adam at $3\times10^{-4}$, with one checkpoint per seed.

\textbf{Update.} Each update trains for $2{,}000$ epochs with the same optimizer settings, for seeds $\{10,20,30,40,50\}$ at every $\lambda$ in the grid, each loading the pretraining checkpoint of its own seed. We compare naive fine-tuning against generative and trajectory replay, sweeping $\lambda\in\{0.05,0.1,0.25,0.5,0.75,0.9,0.95\}$. Replay conditions are drawn from the pretraining support, so no data from the held-out combination enters training at any point.

\textbf{Evaluation.} After the update, we generate $500$ samples conditioned on $(c^\star,y^\star)$ and report energy distance to that cell's held-out test set, with the naive fine-tuning value indicated for reference. Lower is better.

\subsubsection{Results}

\Cref{fig: app_fig20}a shows the layout: contexts separate cleanly, and the shared perturbation offsets give each context the same internal geometry, so the held-out combination has a well-defined target location.

\Cref{fig: app_fig20}b traces held-out energy distance as a function of $\lambda$. Both replay methods trace a U, with a minimum near $\lambda=0.5$ that is more than an order of magnitude below naive fine-tuning, which has no mechanism for retaining the perturbation direction. The shape is the point: held-out error rises at both extremes, so the unseen combination is recovered neither by staying close to the historical model nor by fitting the incoming context alone, but only where the update balances the two. The two replay methods are indistinguishable at the optimum and separate only at high $\lambda$, where trajectory replay degrades faster.\looseness=-1

\Cref{fig: app_fig20}c shows the balance directly: the generated population coincides with the held-out target at $\lambda=0.5$ and is displaced from it at either extreme. Recovering an unseen combination therefore requires fitting the incoming context without discarding the transferable structure carried from the historical ones.\looseness=-1

\begin{figure}[!htbp]
\centering
\includegraphics[width=0.8\textwidth]{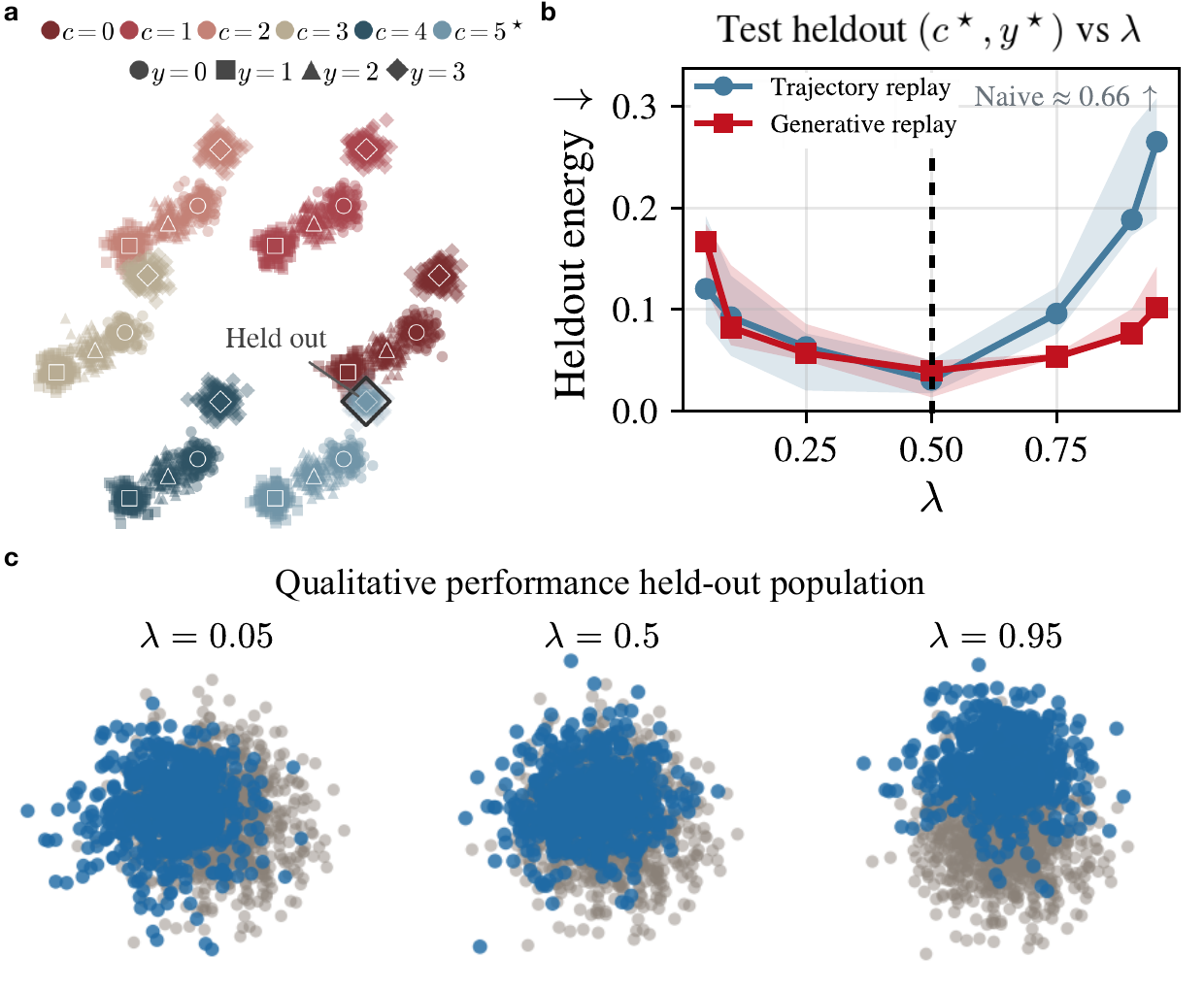}
\caption{\textbf{Synthetic perturbation-shift generalization.}
(\textbf{a}) Layout of the synthetic data, where color encodes the context $c$ and marker shape the perturbation $y$. The held-out combination $(c^\star, y^\star)$ is marked separately.
(\textbf{b}) Energy distance on the held-out combination as a function of the replay weight $\lambda$ for generative and trajectory replay, with the dashed vertical line at $\lambda=0.5$ and the naive fine-tuning value indicated at the top right. Lines show the mean and shaded bands the variability across seeds, and lower is better.
(\textbf{c}) Samples generated by trajectory replay for the held-out combination (blue) against the observed held-out cells (grey), at $\lambda\in\{0.05,0.5,0.95\}$. The two clouds coincide at $\lambda=0.5$, whereas both extremes displace the generated population away from its target, toward the incoming context at low $\lambda$ and toward the historical one at high $\lambda$.}
\label{fig: app_fig20}
\end{figure}

\FloatBarrier
\subsection{Application: perturbation modeling in single-cell transcriptomics}\label{app: perturbation_tx}
\FloatBarrier
\subsubsection{Experimental details}\label{app: perturbation_tx_details}
\textbf{Datasets.} We use two perturbation datasets whose context variable is a biological or technical batch rather than a semantic class. Throughout, a \textit{stratum} is one context-by-perturbation combination, the unit at which conditional distributions are modeled and at which every metric below is computed. Exact values for every setting are in \cref{tab:tx_hparams}.

\begin{itemize}[leftmargin=*, itemindent=0pt]
  \item \emph{Replogle-Nadig} \citep{Replogle2022, Nadig2025}. Genome-scale CRISPRi screens in which the context is the cell line (Jurkat, K562, RPE1, HepG2) and the perturbation is the targeted gene, $2{,}023$ of them plus a non-targeting control.
  \item \emph{PBMC cytokine screen} \citep{Oesinghaus2025.12.12.693897}. Cytokine stimulations in which the context is the donor, twelve in total, and the perturbation is the applied cytokine, $90$ of them plus a PBS (phosphate-buffered saline) control. Cells additionally carry one of $18$ annotated cell types, used as a covariate rather than as a context.
  \item \emph{Preprocessing.} Following \citet{Adduri2026}, both are restricted to the $2{,}000$ highly variable genes, library-size normalized to $10^4$ counts and $\log(1{+}x)$ transformed, with no further scaling.
  \item \emph{Splits.} Split labels follow the few-shot definitions of \citet{Adduri2026}: cells in non-few-shot contexts are training data, while in few-shot contexts the perturbations on that context's test list become held-out test cells and the remainder fine-tuning data. Controls stay trainable everywhere.
  \item \emph{Latent space.} To avoid leakage, flow matching operates on a zero-centered PCA basis fit only on task-0 training cells, so later contexts never influence the coordinate system. We use $50$ components on Replogle and $100$ on PBMC, retaining $34.2\%$ and $58.4\%$ of the variance (see \cref{tab:tx_architecture_selection}).
  \item \emph{Decoding.} Latents return to gene space through a per-task decoder $D_k$ composing the tied task-0 PCA inverse with a per-gene affine calibration $(a_{k,g},b_{k,g})$ fit by scalar least squares on that task's training cells, so that $D_k(\bm{z})=\mathbf{A}_k(\mathbf{W}^{\!\top}\bm{z}+\bm{\mu})+\bm{b}_k$ with $\mathbf{A}_k$ diagonal. A single tied inverse cannot absorb the per-task shift in gene scale, which the calibration corrects, but being diagonal, it cannot recover variance discarded by the truncation to $d$ components: it fixes gene scale, not the subspace. The bank is fit once per dataset, frozen, and shared by every method and seed, so no part of the reported gap can come from a better decoder. A qualitative reconstruction display on the pooled later-task datasets is in \cref{fig: app_fig19}.
\end{itemize}

\textbf{Streams.} Replogle is a two-task stream: task~0 trains jointly on Jurkat, K562, and RPE1, and task~1 introduces HepG2. PBMC is a five-task stream over eight pretraining donors followed by Donor1, Donor4, Donor9, and Donor12. The perturbation vocabulary is shared across contexts, so each task adds a context row rather than a new label space. The model conditions on context and perturbation embeddings, plus a cell-type covariate on PBMC. Task~0 is trained once per seed and shared by every method.

\textbf{Model and training.} The velocity field is an MLP over PCA latents with $128$-dimensional sinusoidal time and condition embeddings concatenated to the input, trained on the linear noise-to-data interpolant with target $\bm{x}_1-\bm{x}_0$. Replogle uses a flat SiLU MLP of width $2048$ and depth $4$, and PBMC uses a pre-normalized residual MLP of width $2048$ and depth $6$. Both were fixed by the two-stage selection study in \cref{tab:tx_architecture_selection}.
\begin{itemize}[leftmargin=*, itemindent=0pt]
  \item \emph{PCA width} (panel a), from single-seed pilots at a fixed $1024\times4$ backbone. The fit falls monotonically with width on both datasets, giving $50$ components on Replogle, while on PBMC $50$ and $100$ are within a thousandth of each other, and we take $100$ following \citet{Klein2025}.
  \item \emph{Backbone} (panel b), sweeping width, depth, and residual connections under the final sequential protocol and selecting on task-0 held-out $\Delta$. The flat MLP wins on Replogle and the smaller residual model on PBMC.
\end{itemize}
Both selection studies use seed $42$ alone, and no method is tuned against them afterwards. Every method uses Adam at $10^{-4}$ with batch size $1024$, under the budgets in \cref{tab:tx_hparams}, for seeds $\{42,69,1714\}$.

\textbf{Replay.} Replay follows the same formulation as in the image experiments, with the teacher integrated by Euler at step size $0.05$ and equal-task weighting setting $\lambda_k=k/(k+1)$. Two things differ here. Replay conditions are \textit{sampled proportionally} to the stratum frequencies of the old training pool, since the conditioning space is a context-by-perturbation grid rather than a class label. The adaptive policy selects from $\lambda\in\{0.25,0.5,0.75\}$ by a balanced Pearson $\Delta$, the unweighted mean of the old-task and new-task stratum macros, scored on the training strata of every task seen so far with $K=200$ generated cells per stratum. The MeanFlow backends distill the task-0 teacher with logit-normal time sampling ($\mu=0.4$, $\sigma=1.0$) and diagonal fraction $0.25$, under the same step budget as task-0 training itself (\cref{tab:tx_hparams}), then replay with $N\in\{1,4\}$ jumps.

\begin{table}[!ht]
\centering
\definecolor{grouprow}{HTML}{F0F0F0}
\setlength{\tabcolsep}{4.5pt}
\small
\caption{Configuration for the continual transcriptomics experiments.}
\label{tab:tx_hparams}
\resizebox{0.8\textwidth}{!}{%
\begin{tabular}{@{}lcc@{}}
\toprule
Setting & Replogle-Nadig & PBMC \\
\midrule
\rowcolor{grouprow} \multicolumn{3}{@{}l}{\textit{Data}} \\
Cells & $643{,}413$ & $9{,}697{,}974$ \\
Genes (highly variable) & $2{,}000$ & $2{,}000$ \\
Contexts & $4$ cell lines & $12$ donors \\
Perturbations & $2{,}023$ genes & $90$ cytokines \\
Control & non-targeting & PBS \\
Covariate conditioning & none & cell type ($18$) \\
Normalization & \multicolumn{2}{c}{$10^4$ library size, $\log(1{+}x)$} \\
PCA width (task-0 basis) & $50$ & $100$ \\
Variance retained & $0.342$ & $0.584$ \\
\midrule
\rowcolor{grouprow} \multicolumn{3}{@{}l}{\textit{Stream}} \\
Tasks & $2$ & $5$ \\
Task 0 & Jurkat, K562, RPE1 & $8$ donors \\
Sequential tasks & HepG2 & Donor1, 4, 9, 12 \\
\midrule
\rowcolor{grouprow} \multicolumn{3}{@{}l}{\textit{Model}} \\
Backbone & MLP & residual MLP \\
Width / depth & $2048$ / $4$ & $2048$ / $6$ \\
Time embedding & $128$ & $128$ \\
Condition embedding & $128$ & $128$ \\
\midrule
\rowcolor{grouprow} \multicolumn{3}{@{}l}{\textit{Training}} \\
Optimizer & Adam, $10^{-4}$ & Adam, $10^{-4}$ \\
Batch size & $1024$ & $1024$ \\
Gradient clip / EMA & none & none \\
Steps, task 0 & $100{,}000$ & $300{,}000$ \\
Steps, later tasks & $50{,}000$ & $150{,}000$ \\
Seeds & \multicolumn{2}{c}{$42$, $69$, $1714$} \\
\midrule
\rowcolor{grouprow} \multicolumn{3}{@{}l}{\textit{Replay}} \\
Teacher sampler & \multicolumn{2}{c}{Euler, step $0.05$ ($20$ evaluations)} \\
Condition sampling & \multicolumn{2}{c}{proportional to old stratum frequency} \\
Equal-task $\lambda$ & $1/2$ & $1/2, 2/3, 3/4, 4/5$ \\
Adaptive $\lambda$ grid & \multicolumn{2}{c}{$\{0.25,0.5,0.75\}$} \\
Selection metric & \multicolumn{2}{c}{old/new balanced Pearson $\Delta$, $K=200$} \\
Selection data & \multicolumn{2}{c}{seen-task training strata} \\
MeanFlow jumps & \multicolumn{2}{c}{$N\in\{1,4\}$} \\
MeanFlow $(\mu,\sigma)$ & \multicolumn{2}{c}{$(0.4,1.0)$, logit-normal} \\
MeanFlow diagonal rate  & \multicolumn{2}{c}{$0.25$} \\
MeanFlow adaptive loss weight ($T_0$ only) & \multicolumn{2}{c}{$0.75$} \\
MeanFlow optimizer & \multicolumn{2}{c}{AdamW, $10^{-4}$, $\beta_2=0.95$, clip $1.0$} \\
MeanFlow distillation budget & \multicolumn{2}{c}{same as task-0 training} \\
\midrule
\rowcolor{grouprow} \multicolumn{3}{@{}l}{\textit{Evaluation}} \\
Generated cells per stratum & $500$ & $500$ \\
Control rows & \multicolumn{2}{c}{decoded observed control latents} \\
Tracks & \multicolumn{2}{c}{fit, retention, stream, heldout} \\
Metrics & \multicolumn{2}{c}{Pearson $\Delta$, DE Overlap@$N$, residual $\Delta$, L1 discr.} \\
\bottomrule
\end{tabular}}
\end{table}
\FloatBarrier

\textbf{Baselines.} All baselines share the stream, backbone, budget, decoder bank, and evaluation protocol, differing only in what is updated or replayed.
\begin{itemize}[leftmargin=*, itemindent=0pt]
  \item \emph{Naive fine-tuning} trains on the incoming context alone, sweeping three independent freeze switches: whether the velocity field trains, whether perturbation rows seen earlier train, and whether earlier context rows train. Rows first appearing on the new task always train.
  \item \emph{LoRA} adapts the MLP with rank $r\in\{4,8,16\}$ and $\alpha=2r$, leaving the output layer frozen and merging the adapter once the task finishes.
  \item \emph{Gradual fine-tuning} self-distills from the frozen previous model with a distillation weight annealed from $1$ to $0$, under sigmoid and linear schedules.
  \item \emph{Two-stage fine-tuning} spends the first $10\%$ or $50\%$ of the task training only on the new context embedding row, then unfreezes the full field.
  \item \emph{Real replay} mixes stored real cells from earlier tasks into each batch, a privileged reference rather than a competing method because it retains raw data.
\end{itemize}
Generative and trajectory replay are run with all three teacher backends and both $\lambda$ policies.

\textbf{Evaluation.} All numbers come from a single evaluation of the final checkpoint, reported on four tracks.
\begin{itemize}[leftmargin=*, itemindent=0pt]
  \item \emph{Retention}, the macro over the task-0 training strata: Jurkat, K562 and RPE1 on Replogle, the eight pretraining donors on PBMC.
  \item \emph{Fit}, the unweighted mean of the per-context macros for every context introduced sequentially thereafter: HepG2 on Replogle, and Donor1, Donor4, Donor9 and Donor12 on PBMC.
  \item \emph{Stream}, the union of the two above (reported, for example, in \cref{fig: app_fig18}).
  \item \emph{Held-out}, the test perturbations of every context seen so far.
\end{itemize}
We generate $K=500$ cells per non-control stratum and take the basal from decoded observed control latents rather than from generated controls, so only the treated side is synthesized. Following the \textit{CellEval} metric suite in \citet{Adduri2026}, the headline metric is Pearson $\Delta$, the correlation between predicted and observed perturbation-minus-control profiles, macro-averaged over strata. We additionally report DE Overlap@$N$, the fraction of the $N$ observed significant differentially expressed genes recovered among the top-$N$ predicted, with $N$ set per perturbation. Results are means with standard errors over seeds.

\subsubsection{Additional results}\label{app:results-transcriptomics}

\textbf{Velocity model selection and full results.} \Cref{tab:tx_architecture_selection} reports the PCA dimension and backbone sweeps, both run on task~0 alone to avoid leakage. Fit falls monotonically with PCA width on both datasets (panel a), since wider bases spend capacity on directions that carry little perturbation signal. On Replogle, the $50$-component basis is clearly best, whereas on PBMC $50$ and $100$ are within a thousandth of each other, and we take $100$ following \citet{Klein2025}. The backbone sweep (panel b) separates the two datasets: on Replogle, the flat MLP wins on fit and task-0 held-out with retention tied within a thousandth, while on PBMC the residual backbone is ahead on both retention and held-out, and among residual configurations the smaller $2048{\times}6$ model gives up little on retention while scoring higher on held-out, so we keep it.

\Cref{tab:replogle_cl_full,tab:pbmc_cl_full} give every configuration behind \cref{fig: tx_figure}. On the single Replogle update, the replay families are close, with trajectory replay retaining slightly more than generative replay ($0.604$--$0.620$ against $0.573$--$0.594$ Pearson $\Delta$) and held-out tied at $0.39$--$0.41$. Over the four sequential updates of PBMC, the gap widens on every axis, to $0.430$--$0.451$ retention against $0.306$--$0.403$ and $0.332$--$0.350$ held-out against $0.289$--$0.330$, placing trajectory replay level with real replay on fit and held-out despite storing no cells. The MeanFlow backends leave it unchanged on both streams but cost generative replay substantially on PBMC, where MF-1 retention drops to $0.306$.

The freeze grid explains why a memory-free non-replay baseline can still top the retention column. Retention depends almost entirely on the velocity field, and adaptation almost entirely on it too, so the three switches are not three comparable knobs.
\begin{itemize}[leftmargin=*, itemindent=0pt]
    \item \emph{Field frozen, old embeddings frozen} ($\mathrm{vf}{=}0$, $\mathrm{pert}{=}0$, $\mathrm{ctx}{=}0$). Only the new context row is learned, so task~0 is preserved almost exactly, and PBMC retention is the highest of any memory-free method at $0.491$. A single embedding vector cannot express the new context's perturbation responses, so fit stalls at $0.233$ against $0.515$ for trajectory replay.
  \item \emph{Field frozen, old embeddings trainable} ($\mathrm{vf}{=}0$, $\mathrm{pert}{=}1$, $\mathrm{ctx}{=}1$). The old rows drift to compensate for a field that cannot move, costing retention ($0.400$) for no fit at all ($0.232$) on PBMC. On Replogle, the same relaxation does buy fit, $0.359$ against $0.193$, at a smaller retention cost.
    \item \emph{Field trainable} ($\mathrm{vf}{=}1$, any embedding setting). Fit recovers to $0.289$ and retention collapses to $0.056$--$0.064$, while the embedding switches change neither, confirming that the field rather than the conditioning rows carries the historical task.
\end{itemize}
Freezing therefore buys retention only by refusing to adapt, which is why the two axes have to be read jointly rather than ranked one column at a time.

\FloatBarrier
\begin{table*}[!htbp]
\centering
\definecolor{grouprow}{HTML}{F0F0F0}
\definecolor{selectedrow}{HTML}{E5EEF3}
\setlength{\tabcolsep}{4.5pt}
\small
\caption{\textbf{Transcriptomics architecture selection.}
(\textbf{a}) PCA width sweep with a fixed MLP $1024{\times}4$ backbone, reporting fit Pearson $\Delta$ from single-seed pilots.
(\textbf{b}) Backbone width and depth under the final sequential protocol with $500$ generated cells per non-control stratum, reporting fit, retention, and task-0 held-out Pearson $\Delta$. All rows use seed $42$, and blue rows mark the selected configuration.}
\label{tab:tx_architecture_selection}
\vspace{4pt}

\begin{tabular}{@{}llc@{}}
\toprule
Dataset & PCA width & $\uparrow$ Pilot fit $\Delta$ \\
\midrule
\rowcolor{grouprow} \multicolumn{3}{@{}l}{\textit{(a) PCA width sweep}} \\
\rowcolor{selectedrow} Replogle & $50$ & $\mathbf{0.486}$ \\
Replogle & $100$ & $0.458$ \\
Replogle & $256$ & $0.423$ \\
Replogle & $512$ & $0.389$ \\
\addlinespace
PBMC & $50$ & $0.353$ \\
\rowcolor{selectedrow} PBMC & $100$ & $0.352$ \\
PBMC & $256$ & $0.282$ \\
PBMC & $512$ & $0.263$ \\
\bottomrule
\end{tabular}

\vspace{12pt}

\begin{tabular}{@{}llccccc@{}}
\toprule
Dataset & Backbone & Width & Depth & $\uparrow$ Fit $\Delta$ & $\uparrow$ Retention $\Delta$ & $\uparrow$ $T_0$ held-out $\Delta$ \\
\midrule
\rowcolor{grouprow} \multicolumn{7}{@{}l}{\textit{(b) Final-protocol model-capacity check ($K=500$)}} \\
\rowcolor{selectedrow} Replogle & MLP & $2048$ & $4$ & $\mathbf{0.611}$ & $0.608$ & $\mathbf{0.410}$ \\
Replogle & Residual MLP & $2048$ & $6$ & $0.599$ & $\mathbf{0.609}$ & $0.392$ \\
\addlinespace
PBMC & MLP & $4096$ & $4$ & $\mathbf{0.546}$ & $0.441$ & $0.297$ \\
PBMC & MLP & $3072$ & $4$ & $0.528$ & $0.443$ & $0.308$ \\
PBMC & Residual MLP & $3072$ & $6$ & $0.532$ & $\mathbf{0.455}$ & $0.352$ \\
\rowcolor{selectedrow} PBMC & Residual MLP & $2048$ & $6$ & $0.514$ & $0.445$ & $\mathbf{0.355}$ \\
\bottomrule
\end{tabular}
\end{table*}

\FloatBarrier
\begin{table*}[!htbp]
\centering
\definecolor{trajectoryrow}{HTML}{E5EEF3}
\definecolor{referencerow}{HTML}{F0F0F0}
\setlength{\tabcolsep}{4.2pt}
\small
\caption{\textbf{Continual perturbation-response modeling on Replogle-Nadig.} Final-checkpoint Pearson $\Delta$ and DE Overlap@N after adaptation from three task-0 cell lines to HepG2, reported as fit on the later task, retention on task 0, and held-out on perturbations excluded from training. Means with standard errors over seeds 42, 69, and 1714. Adaptive $\lambda$ is selected on validation strata and equal-task uses $\lambda_k=k/(k+1)$. Gray marks the privileged real-replay reference and blue marks trajectory replay. Higher is better.}
\label{tab:replogle_cl_full}
\resizebox{\textwidth}{!}{%
\begin{tabular}{@{}lcccccc@{}}
\toprule
Method & \multicolumn{3}{c}{Pearson $\Delta$} & \multicolumn{3}{c}{DE Overlap@N} \\
\cmidrule(lr){2-4}\cmidrule(l){5-7}
& $\uparrow$ Fit & $\uparrow$ Retention & $\uparrow$ Held-out & $\uparrow$ Fit & $\uparrow$ Retention & $\uparrow$ Held-out \\
\midrule
Naive (vf${=}0$, pert${=}0$, ctx${=}0$) & $0.193$ {\scriptsize $(\pm\,0.004)$} & $0.629$ {\scriptsize $(\pm\,0.001)$} & $0.163$ {\scriptsize $(\pm\,0.004)$} & $0.172$ {\scriptsize $(\pm\,0.001)$} & $0.287$ {\scriptsize $(\pm\,0.001)$} & $0.161$ {\scriptsize $(\pm\,0.001)$} \\
Naive (vf${=}0$, pert${=}1$, ctx${=}1$) & $0.359$ {\scriptsize $(\pm\,0.003)$} & $0.595$ {\scriptsize $(\pm\,0.001)$} & $0.169$ {\scriptsize $(\pm\,0.012)$} & $0.213$ {\scriptsize $(\pm\,0.000)$} & $0.275$ {\scriptsize $(\pm\,0.000)$} & $0.165$ {\scriptsize $(\pm\,0.002)$} \\
Naive (vf${=}1$, pert${=}0$, ctx${=}0$) & $0.616$ {\scriptsize $(\pm\,0.001)$} & $0.216$ {\scriptsize $(\pm\,0.006)$} & $0.351$ {\scriptsize $(\pm\,0.002)$} & $0.260$ {\scriptsize $(\pm\,0.001)$} & $0.156$ {\scriptsize $(\pm\,0.001)$} & $0.198$ {\scriptsize $(\pm\,0.001)$} \\
Naive (vf${=}1$, pert${=}0$, ctx${=}1$) & $0.616$ {\scriptsize $(\pm\,0.001)$} & $0.216$ {\scriptsize $(\pm\,0.006)$} & $0.351$ {\scriptsize $(\pm\,0.002)$} & $0.260$ {\scriptsize $(\pm\,0.001)$} & $0.156$ {\scriptsize $(\pm\,0.001)$} & $0.198$ {\scriptsize $(\pm\,0.001)$} \\
Naive (vf${=}1$, pert${=}1$, ctx${=}0$) & $0.616$ {\scriptsize $(\pm\,0.001)$} & $0.214$ {\scriptsize $(\pm\,0.007)$} & $0.353$ {\scriptsize $(\pm\,0.002)$} & $0.260$ {\scriptsize $(\pm\,0.001)$} & $0.156$ {\scriptsize $(\pm\,0.001)$} & $0.198$ {\scriptsize $(\pm\,0.000)$} \\
Naive (vf${=}1$, pert${=}1$, ctx${=}1$) & $0.616$ {\scriptsize $(\pm\,0.001)$} & $0.214$ {\scriptsize $(\pm\,0.007)$} & $0.353$ {\scriptsize $(\pm\,0.002)$} & $0.260$ {\scriptsize $(\pm\,0.001)$} & $0.156$ {\scriptsize $(\pm\,0.001)$} & $0.198$ {\scriptsize $(\pm\,0.000)$} \\
\midrule
LoRA ($r{=}4$, $\alpha{=}8$) & $0.551$ {\scriptsize $(\pm\,0.000)$} & $0.259$ {\scriptsize $(\pm\,0.005)$} & $0.391$ {\scriptsize $(\pm\,0.004)$} & $0.251$ {\scriptsize $(\pm\,0.000)$} & $0.170$ {\scriptsize $(\pm\,0.002)$} & $0.216$ {\scriptsize $(\pm\,0.000)$} \\
LoRA ($r{=}8$, $\alpha{=}16$) & $0.585$ {\scriptsize $(\pm\,0.001)$} & $0.260$ {\scriptsize $(\pm\,0.002)$} & $0.386$ {\scriptsize $(\pm\,0.004)$} & $0.258$ {\scriptsize $(\pm\,0.002)$} & $0.169$ {\scriptsize $(\pm\,0.000)$} & $0.214$ {\scriptsize $(\pm\,0.001)$} \\
LoRA ($r{=}16$, $\alpha{=}32$) & $0.608$ {\scriptsize $(\pm\,0.001)$} & $0.264$ {\scriptsize $(\pm\,0.011)$} & $0.382$ {\scriptsize $(\pm\,0.002)$} & $0.261$ {\scriptsize $(\pm\,0.001)$} & $0.169$ {\scriptsize $(\pm\,0.002)$} & $0.212$ {\scriptsize $(\pm\,0.001)$} \\
GFT (sigmoid, $vf{=}1$) & $0.616$ {\scriptsize $(\pm\,0.001)$} & $0.226$ {\scriptsize $(\pm\,0.006)$} & $0.361$ {\scriptsize $(\pm\,0.005)$} & $0.259$ {\scriptsize $(\pm\,0.001)$} & $0.157$ {\scriptsize $(\pm\,0.001)$} & $0.202$ {\scriptsize $(\pm\,0.001)$} \\
GFT (linear, $vf{=}1$) & $0.615$ {\scriptsize $(\pm\,0.001)$} & $0.226$ {\scriptsize $(\pm\,0.006)$} & $0.363$ {\scriptsize $(\pm\,0.005)$} & $0.259$ {\scriptsize $(\pm\,0.000)$} & $0.157$ {\scriptsize $(\pm\,0.001)$} & $0.203$ {\scriptsize $(\pm\,0.001)$} \\
Two-stage (10\% new ctx $\to$ $vf{=}1$) & $0.618$ {\scriptsize $(\pm\,0.000)$} & $0.227$ {\scriptsize $(\pm\,0.004)$} & $0.361$ {\scriptsize $(\pm\,0.003)$} & $0.257$ {\scriptsize $(\pm\,0.001)$} & $0.159$ {\scriptsize $(\pm\,0.001)$} & $0.201$ {\scriptsize $(\pm\,0.000)$} \\
Two-stage (50\% new ctx $\to$ $vf{=}1$) & $0.617$ {\scriptsize $(\pm\,0.001)$} & $0.266$ {\scriptsize $(\pm\,0.004)$} & $0.377$ {\scriptsize $(\pm\,0.001)$} & $0.260$ {\scriptsize $(\pm\,0.001)$} & $0.166$ {\scriptsize $(\pm\,0.000)$} & $0.208$ {\scriptsize $(\pm\,0.001)$} \\
\midrule
\rowcolor{referencerow} Real replay (adaptive $\lambda$) & $0.610$ {\scriptsize $(\pm\,0.001)$} & $0.632$ {\scriptsize $(\pm\,0.001)$} & $0.406$ {\scriptsize $(\pm\,0.004)$} & $0.260$ {\scriptsize $(\pm\,0.000)$} & $0.288$ {\scriptsize $(\pm\,0.001)$} & $0.218$ {\scriptsize $(\pm\,0.001)$} \\
\rowcolor{referencerow} Real replay (equal-task $\lambda$) & $0.613$ {\scriptsize $(\pm\,0.001)$} & $0.620$ {\scriptsize $(\pm\,0.001)$} & $0.401$ {\scriptsize $(\pm\,0.001)$} & $0.257$ {\scriptsize $(\pm\,0.001)$} & $0.285$ {\scriptsize $(\pm\,0.002)$} & $0.217$ {\scriptsize $(\pm\,0.001)$} \\
\midrule
Generative replay (ODE, adaptive $\lambda$) & $0.611$ {\scriptsize $(\pm\,0.001)$} & $0.594$ {\scriptsize $(\pm\,0.005)$} & $0.403$ {\scriptsize $(\pm\,0.005)$} & $0.257$ {\scriptsize $(\pm\,0.001)$} & $0.276$ {\scriptsize $(\pm\,0.002)$} & $0.217$ {\scriptsize $(\pm\,0.001)$} \\
Generative replay (ODE, equal-task $\lambda$) & $0.613$ {\scriptsize $(\pm\,0.000)$} & $0.588$ {\scriptsize $(\pm\,0.002)$} & $0.400$ {\scriptsize $(\pm\,0.003)$} & $0.257$ {\scriptsize $(\pm\,0.003)$} & $0.275$ {\scriptsize $(\pm\,0.003)$} & $0.217$ {\scriptsize $(\pm\,0.001)$} \\
Generative replay (MF-1, adaptive $\lambda$) & $0.611$ {\scriptsize $(\pm\,0.001)$} & $0.580$ {\scriptsize $(\pm\,0.005)$} & $0.403$ {\scriptsize $(\pm\,0.005)$} & $0.257$ {\scriptsize $(\pm\,0.001)$} & $0.280$ {\scriptsize $(\pm\,0.003)$} & $0.217$ {\scriptsize $(\pm\,0.000)$} \\
Generative replay (MF-1, equal-task $\lambda$) & $0.613$ {\scriptsize $(\pm\,0.000)$} & $0.573$ {\scriptsize $(\pm\,0.002)$} & $0.400$ {\scriptsize $(\pm\,0.002)$} & $0.256$ {\scriptsize $(\pm\,0.002)$} & $0.277$ {\scriptsize $(\pm\,0.002)$} & $0.217$ {\scriptsize $(\pm\,0.000)$} \\
Generative replay (MF-4, adaptive $\lambda$) & $0.609$ {\scriptsize $(\pm\,0.001)$} & $0.588$ {\scriptsize $(\pm\,0.003)$} & $0.407$ {\scriptsize $(\pm\,0.005)$} & $0.256$ {\scriptsize $(\pm\,0.001)$} & $0.279$ {\scriptsize $(\pm\,0.003)$} & $0.219$ {\scriptsize $(\pm\,0.001)$} \\
Generative replay (MF-4, equal-task $\lambda$) & $0.613$ {\scriptsize $(\pm\,0.000)$} & $0.581$ {\scriptsize $(\pm\,0.002)$} & $0.400$ {\scriptsize $(\pm\,0.002)$} & $0.256$ {\scriptsize $(\pm\,0.003)$} & $0.278$ {\scriptsize $(\pm\,0.002)$} & $0.218$ {\scriptsize $(\pm\,0.001)$} \\
\midrule
\rowcolor{trajectoryrow} Trajectory replay (ODE, adaptive $\lambda$) & $0.612$ {\scriptsize $(\pm\,0.001)$} & $0.619$ {\scriptsize $(\pm\,0.002)$} & $0.399$ {\scriptsize $(\pm\,0.002)$} & $0.256$ {\scriptsize $(\pm\,0.002)$} & $0.285$ {\scriptsize $(\pm\,0.001)$} & $0.217$ {\scriptsize $(\pm\,0.000)$} \\
\rowcolor{trajectoryrow} Trajectory replay (ODE, equal-task $\lambda$) & $0.610$ {\scriptsize $(\pm\,0.001)$} & $0.606$ {\scriptsize $(\pm\,0.002)$} & $0.398$ {\scriptsize $(\pm\,0.001)$} & $0.257$ {\scriptsize $(\pm\,0.001)$} & $0.282$ {\scriptsize $(\pm\,0.001)$} & $0.217$ {\scriptsize $(\pm\,0.001)$} \\
\rowcolor{trajectoryrow} Trajectory replay (MF-1, adaptive $\lambda$) & $0.613$ {\scriptsize $(\pm\,0.001)$} & $0.620$ {\scriptsize $(\pm\,0.001)$} & $0.404$ {\scriptsize $(\pm\,0.001)$} & $0.256$ {\scriptsize $(\pm\,0.002)$} & $0.285$ {\scriptsize $(\pm\,0.001)$} & $0.218$ {\scriptsize $(\pm\,0.001)$} \\
\rowcolor{trajectoryrow} Trajectory replay (MF-1, equal-task $\lambda$) & $0.608$ {\scriptsize $(\pm\,0.002)$} & $0.604$ {\scriptsize $(\pm\,0.001)$} & $0.392$ {\scriptsize $(\pm\,0.002)$} & $0.257$ {\scriptsize $(\pm\,0.001)$} & $0.282$ {\scriptsize $(\pm\,0.001)$} & $0.217$ {\scriptsize $(\pm\,0.001)$} \\
\rowcolor{trajectoryrow} Trajectory replay (MF-4, adaptive $\lambda$) & $0.613$ {\scriptsize $(\pm\,0.001)$} & $0.620$ {\scriptsize $(\pm\,0.001)$} & $0.404$ {\scriptsize $(\pm\,0.001)$} & $0.256$ {\scriptsize $(\pm\,0.002)$} & $0.285$ {\scriptsize $(\pm\,0.001)$} & $0.218$ {\scriptsize $(\pm\,0.001)$} \\
\rowcolor{trajectoryrow} Trajectory replay (MF-4, equal-task $\lambda$) & $0.608$ {\scriptsize $(\pm\,0.002)$} & $0.604$ {\scriptsize $(\pm\,0.001)$} & $0.392$ {\scriptsize $(\pm\,0.002)$} & $0.258$ {\scriptsize $(\pm\,0.000)$} & $0.282$ {\scriptsize $(\pm\,0.001)$} & $0.217$ {\scriptsize $(\pm\,0.001)$} \\
\bottomrule
\end{tabular}%
}
\end{table*}

\begin{table*}[!htbp]
\centering
\definecolor{trajectoryrow}{HTML}{E5EEF3}
\definecolor{referencerow}{HTML}{F0F0F0}
\setlength{\tabcolsep}{4.2pt}
\small
\caption{\textbf{Continual perturbation-response modeling on Parse-PBMC.} Final-checkpoint Pearson $\Delta$ and DE Overlap@N after four sequential donor updates, reported as fit on the later tasks, retention on task 0, and held-out on perturbations excluded from training. Means with standard errors over seeds 42, 69, and 1714. Adaptive $\lambda$ is selected on validation strata and equal-task uses $\lambda_k=k/(k+1)$. Gray marks the privileged real-replay reference and blue marks trajectory replay. Higher is better.}
\label{tab:pbmc_cl_full}
\resizebox{\textwidth}{!}{%
\begin{tabular}{@{}lcccccc@{}}
\toprule
Method & \multicolumn{3}{c}{Pearson $\Delta$} & \multicolumn{3}{c}{DE Overlap@N} \\
\cmidrule(lr){2-4}\cmidrule(l){5-7}
& $\uparrow$ Fit & $\uparrow$ Retention & $\uparrow$ Held-out & $\uparrow$ Fit & $\uparrow$ Retention & $\uparrow$ Held-out \\
\midrule
Naive (vf${=}0$, pert${=}0$, ctx${=}0$) & $0.233$ {\scriptsize $(\pm\,0.003)$} & $0.491$ {\scriptsize $(\pm\,0.001)$} & $0.235$ {\scriptsize $(\pm\,0.004)$} & $0.100$ {\scriptsize $(\pm\,0.004)$} & $0.144$ {\scriptsize $(\pm\,0.001)$} & $0.104$ {\scriptsize $(\pm\,0.003)$} \\
Naive (vf${=}0$, pert${=}1$, ctx${=}1$) & $0.232$ {\scriptsize $(\pm\,0.005)$} & $0.400$ {\scriptsize $(\pm\,0.003)$} & $0.207$ {\scriptsize $(\pm\,0.002)$} & $0.112$ {\scriptsize $(\pm\,0.002)$} & $0.133$ {\scriptsize $(\pm\,0.001)$} & $0.096$ {\scriptsize $(\pm\,0.002)$} \\
Naive (vf${=}1$, pert${=}0$, ctx${=}0$) & $0.290$ {\scriptsize $(\pm\,0.000)$} & $0.056$ {\scriptsize $(\pm\,0.000)$} & $0.151$ {\scriptsize $(\pm\,0.001)$} & $0.114$ {\scriptsize $(\pm\,0.001)$} & $0.074$ {\scriptsize $(\pm\,0.001)$} & $0.087$ {\scriptsize $(\pm\,0.002)$} \\
Naive (vf${=}1$, pert${=}0$, ctx${=}1$) & $0.290$ {\scriptsize $(\pm\,0.000)$} & $0.056$ {\scriptsize $(\pm\,0.000)$} & $0.151$ {\scriptsize $(\pm\,0.001)$} & $0.114$ {\scriptsize $(\pm\,0.001)$} & $0.074$ {\scriptsize $(\pm\,0.001)$} & $0.087$ {\scriptsize $(\pm\,0.002)$} \\
Naive (vf${=}1$, pert${=}1$, ctx${=}0$) & $0.289$ {\scriptsize $(\pm\,0.001)$} & $0.064$ {\scriptsize $(\pm\,0.001)$} & $0.154$ {\scriptsize $(\pm\,0.002)$} & $0.114$ {\scriptsize $(\pm\,0.002)$} & $0.078$ {\scriptsize $(\pm\,0.001)$} & $0.092$ {\scriptsize $(\pm\,0.001)$} \\
Naive (vf${=}1$, pert${=}1$, ctx${=}1$) & $0.289$ {\scriptsize $(\pm\,0.001)$} & $0.064$ {\scriptsize $(\pm\,0.001)$} & $0.154$ {\scriptsize $(\pm\,0.002)$} & $0.114$ {\scriptsize $(\pm\,0.002)$} & $0.078$ {\scriptsize $(\pm\,0.001)$} & $0.092$ {\scriptsize $(\pm\,0.001)$} \\
\midrule
LoRA ($r{=}4$, $\alpha{=}8$) & $0.273$ {\scriptsize $(\pm\,0.004)$} & $0.154$ {\scriptsize $(\pm\,0.007)$} & $0.214$ {\scriptsize $(\pm\,0.002)$} & $0.123$ {\scriptsize $(\pm\,0.002)$} & $0.083$ {\scriptsize $(\pm\,0.003)$} & $0.111$ {\scriptsize $(\pm\,0.001)$} \\
LoRA ($r{=}8$, $\alpha{=}16$) & $0.279$ {\scriptsize $(\pm\,0.003)$} & $0.155$ {\scriptsize $(\pm\,0.008)$} & $0.217$ {\scriptsize $(\pm\,0.001)$} & $0.122$ {\scriptsize $(\pm\,0.002)$} & $0.085$ {\scriptsize $(\pm\,0.003)$} & $0.110$ {\scriptsize $(\pm\,0.002)$} \\
LoRA ($r{=}16$, $\alpha{=}32$) & $0.287$ {\scriptsize $(\pm\,0.003)$} & $0.166$ {\scriptsize $(\pm\,0.006)$} & $0.225$ {\scriptsize $(\pm\,0.003)$} & $0.122$ {\scriptsize $(\pm\,0.003)$} & $0.086$ {\scriptsize $(\pm\,0.004)$} & $0.109$ {\scriptsize $(\pm\,0.002)$} \\
GFT (sigmoid, $vf{=}1$) & $0.291$ {\scriptsize $(\pm\,0.001)$} & $0.061$ {\scriptsize $(\pm\,0.001)$} & $0.156$ {\scriptsize $(\pm\,0.001)$} & $0.114$ {\scriptsize $(\pm\,0.002)$} & $0.075$ {\scriptsize $(\pm\,0.001)$} & $0.087$ {\scriptsize $(\pm\,0.002)$} \\
GFT (linear, $vf{=}1$) & $0.291$ {\scriptsize $(\pm\,0.001)$} & $0.060$ {\scriptsize $(\pm\,0.001)$} & $0.157$ {\scriptsize $(\pm\,0.001)$} & $0.113$ {\scriptsize $(\pm\,0.001)$} & $0.074$ {\scriptsize $(\pm\,0.001)$} & $0.087$ {\scriptsize $(\pm\,0.002)$} \\
Two-stage (10\% new ctx $\to$ $vf{=}1$) & $0.289$ {\scriptsize $(\pm\,0.001)$} & $0.059$ {\scriptsize $(\pm\,0.002)$} & $0.155$ {\scriptsize $(\pm\,0.002)$} & $0.113$ {\scriptsize $(\pm\,0.001)$} & $0.076$ {\scriptsize $(\pm\,0.000)$} & $0.086$ {\scriptsize $(\pm\,0.002)$} \\
Two-stage (50\% new ctx $\to$ $vf{=}1$) & $0.298$ {\scriptsize $(\pm\,0.002)$} & $0.091$ {\scriptsize $(\pm\,0.002)$} & $0.171$ {\scriptsize $(\pm\,0.002)$} & $0.121$ {\scriptsize $(\pm\,0.001)$} & $0.077$ {\scriptsize $(\pm\,0.001)$} & $0.092$ {\scriptsize $(\pm\,0.002)$} \\
\midrule
\rowcolor{referencerow} Real replay (adaptive $\lambda$) & $0.498$ {\scriptsize $(\pm\,0.002)$} & $0.483$ {\scriptsize $(\pm\,0.000)$} & $0.350$ {\scriptsize $(\pm\,0.003)$} & $0.165$ {\scriptsize $(\pm\,0.001)$} & $0.143$ {\scriptsize $(\pm\,0.001)$} & $0.137$ {\scriptsize $(\pm\,0.001)$} \\
\rowcolor{referencerow} Real replay (equal-task $\lambda$) & $0.496$ {\scriptsize $(\pm\,0.001)$} & $0.500$ {\scriptsize $(\pm\,0.002)$} & $0.355$ {\scriptsize $(\pm\,0.005)$} & $0.167$ {\scriptsize $(\pm\,0.002)$} & $0.147$ {\scriptsize $(\pm\,0.001)$} & $0.141$ {\scriptsize $(\pm\,0.002)$} \\
\midrule
Generative replay (ODE, adaptive $\lambda$) & $0.475$ {\scriptsize $(\pm\,0.001)$} & $0.383$ {\scriptsize $(\pm\,0.002)$} & $0.320$ {\scriptsize $(\pm\,0.003)$} & $0.122$ {\scriptsize $(\pm\,0.001)$} & $0.059$ {\scriptsize $(\pm\,0.002)$} & $0.100$ {\scriptsize $(\pm\,0.002)$} \\
Generative replay (ODE, equal-task $\lambda$) & $0.471$ {\scriptsize $(\pm\,0.003)$} & $0.403$ {\scriptsize $(\pm\,0.004)$} & $0.330$ {\scriptsize $(\pm\,0.003)$} & $0.125$ {\scriptsize $(\pm\,0.001)$} & $0.061$ {\scriptsize $(\pm\,0.001)$} & $0.107$ {\scriptsize $(\pm\,0.001)$} \\
Generative replay (MF-1, adaptive $\lambda$) & $0.437$ {\scriptsize $(\pm\,0.001)$} & $0.306$ {\scriptsize $(\pm\,0.007)$} & $0.289$ {\scriptsize $(\pm\,0.002)$} & $0.141$ {\scriptsize $(\pm\,0.005)$} & $0.088$ {\scriptsize $(\pm\,0.005)$} & $0.112$ {\scriptsize $(\pm\,0.006)$} \\
Generative replay (MF-1, equal-task $\lambda$) & $0.440$ {\scriptsize $(\pm\,0.004)$} & $0.340$ {\scriptsize $(\pm\,0.001)$} & $0.314$ {\scriptsize $(\pm\,0.004)$} & $0.148$ {\scriptsize $(\pm\,0.001)$} & $0.101$ {\scriptsize $(\pm\,0.002)$} & $0.124$ {\scriptsize $(\pm\,0.002)$} \\
Generative replay (MF-4, adaptive $\lambda$) & $0.446$ {\scriptsize $(\pm\,0.002)$} & $0.316$ {\scriptsize $(\pm\,0.007)$} & $0.305$ {\scriptsize $(\pm\,0.001)$} & $0.150$ {\scriptsize $(\pm\,0.005)$} & $0.101$ {\scriptsize $(\pm\,0.004)$} & $0.122$ {\scriptsize $(\pm\,0.003)$} \\
Generative replay (MF-4, equal-task $\lambda$) & $0.441$ {\scriptsize $(\pm\,0.004)$} & $0.345$ {\scriptsize $(\pm\,0.007)$} & $0.323$ {\scriptsize $(\pm\,0.007)$} & $0.152$ {\scriptsize $(\pm\,0.001)$} & $0.108$ {\scriptsize $(\pm\,0.002)$} & $0.130$ {\scriptsize $(\pm\,0.000)$} \\
\midrule
\rowcolor{trajectoryrow} Trajectory replay (ODE, adaptive $\lambda$) & $0.496$ {\scriptsize $(\pm\,0.004)$} & $0.435$ {\scriptsize $(\pm\,0.003)$} & $0.332$ {\scriptsize $(\pm\,0.006)$} & $0.162$ {\scriptsize $(\pm\,0.002)$} & $0.128$ {\scriptsize $(\pm\,0.004)$} & $0.129$ {\scriptsize $(\pm\,0.002)$} \\
\rowcolor{trajectoryrow} Trajectory replay (ODE, equal-task $\lambda$) & $0.515$ {\scriptsize $(\pm\,0.002)$} & $0.450$ {\scriptsize $(\pm\,0.003)$} & $0.348$ {\scriptsize $(\pm\,0.007)$} & $0.161$ {\scriptsize $(\pm\,0.003)$} & $0.129$ {\scriptsize $(\pm\,0.003)$} & $0.131$ {\scriptsize $(\pm\,0.003)$} \\
\rowcolor{trajectoryrow} Trajectory replay (MF-1, adaptive $\lambda$) & $0.493$ {\scriptsize $(\pm\,0.003)$} & $0.430$ {\scriptsize $(\pm\,0.003)$} & $0.333$ {\scriptsize $(\pm\,0.003)$} & $0.158$ {\scriptsize $(\pm\,0.001)$} & $0.127$ {\scriptsize $(\pm\,0.001)$} & $0.127$ {\scriptsize $(\pm\,0.001)$} \\
\rowcolor{trajectoryrow} Trajectory replay (MF-1, equal-task $\lambda$) & $0.514$ {\scriptsize $(\pm\,0.003)$} & $0.449$ {\scriptsize $(\pm\,0.007)$} & $0.349$ {\scriptsize $(\pm\,0.006)$} & $0.161$ {\scriptsize $(\pm\,0.004)$} & $0.129$ {\scriptsize $(\pm\,0.005)$} & $0.132$ {\scriptsize $(\pm\,0.004)$} \\
\rowcolor{trajectoryrow} Trajectory replay (MF-4, adaptive $\lambda$) & $0.493$ {\scriptsize $(\pm\,0.003)$} & $0.431$ {\scriptsize $(\pm\,0.003)$} & $0.333$ {\scriptsize $(\pm\,0.003)$} & $0.158$ {\scriptsize $(\pm\,0.001)$} & $0.127$ {\scriptsize $(\pm\,0.001)$} & $0.127$ {\scriptsize $(\pm\,0.001)$} \\
\rowcolor{trajectoryrow} Trajectory replay (MF-4, equal-task $\lambda$) & $0.515$ {\scriptsize $(\pm\,0.002)$} & $0.451$ {\scriptsize $(\pm\,0.005)$} & $0.350$ {\scriptsize $(\pm\,0.007)$} & $0.161$ {\scriptsize $(\pm\,0.005)$} & $0.130$ {\scriptsize $(\pm\,0.005)$} & $0.132$ {\scriptsize $(\pm\,0.004)$} \\
\bottomrule
\end{tabular}%
}
\end{table*}

\FloatBarrier

\textbf{Alternative metrics.} \Cref{fig: app_fig15} repeats the fit-retention view under two metrics from \citet{Adduri2026} that stress different failure modes. L1 discrimination (top) ranks each predicted profile against the observed profiles of all strata by $L_1$ distance and reports how far its own target sits in that ranking, so it rewards predictions that are specific to their condition rather than merely close to the average response. Residual Pearson $\Delta$ (bottom) subtracts the mean perturbation effect across strata before correlating, leaving only the context-specific deviation, so a method predicting the shared response alone scores near zero. The layout is unchanged under both: freezing sits top-left, the adapters and fine-tuning variants low on retention, and the replay family in the upper-right corner with trajectory replay adjacent to real replay. Residual $\Delta$ matters most, since it rules out the reading that all methods recover the same global shift and differ only in how much of it they keep.

\begin{figure}[!htbp]
\centering
\includegraphics[width=1\textwidth]{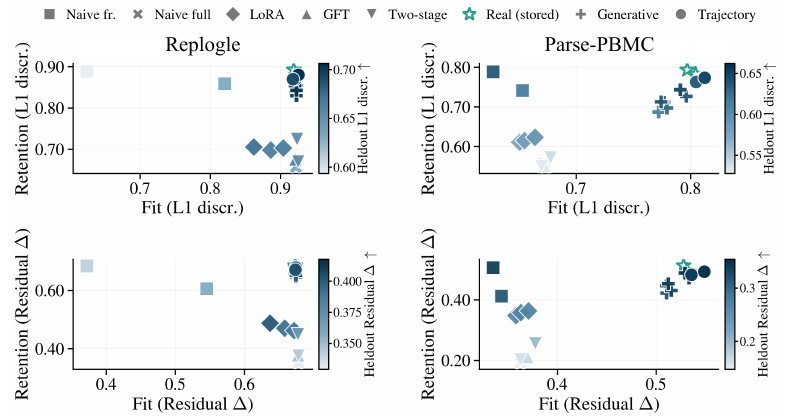}
\caption{\textbf{Fit--retention trade-off under alternative metrics.}
Retention against fit for every configuration on Replogle-Nadig (left) and Parse-PBMC (right), measured by L1 discrimination (top) and residual Pearson $\Delta$ (bottom). Marker shape denotes the method family and marker color the heldout score for the same metric, with a separate color scale per panel. Higher is better on all axes.}
\label{fig:tx_alt_metrics}
\label{fig: app_fig15}
\end{figure}

\textbf{Sequential behavior on PBMC.} \Cref{fig: app_fig16}a tracks held-out Pearson $\Delta$ after each donor update. Naive fine-tuning collapses at the first update, while the three replay variants decline in parallel from roughly $0.43$ to $0.33$--$0.35$, with trajectory replay tracking real replay. That the privileged reference declines equally points to the difficulty of covering more donors at fixed capacity rather than to forgetting specific to generated replay. Panel b shows all three selecting weights well below the equal-task schedule, real replay included, so the criterion favors the incoming donor rather than discounting generated samples, consistent with a shared perturbation vocabulary where each new context mainly needs its own basal position fitted.

\begin{figure}[!htbp]
\centering
\includegraphics[width=0.9\textwidth]{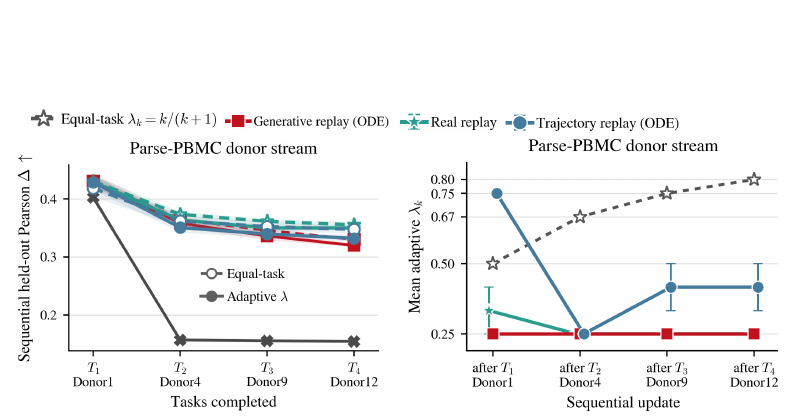}
\caption{\textbf{Parse-PBMC donor stream.} (\textbf{a}) Held-out Pearson $\Delta$ after each donor update for naive fine-tuning, real replay, and generative and trajectory replay with an ODE backend. Filled markers with solid lines denote adaptive $\lambda$ and open markers with dashed lines equal-task $\lambda$, and shaded bands show variability across seeds. (\textbf{b}) Mean replay weight $\lambda_k$ selected at each update by the adaptive policy for the same replay methods, with the equal-task rule $\lambda_k=k/(k+1)$ as a dashed reference. Error bars show variability over seeds $42$, $69$, and $1714$. Higher is better in (\textbf{a}).}
\label{fig: app_fig16}
\end{figure}
\FloatBarrier

\textbf{Replay budget.} \Cref{fig: app_fig17} asks what the MeanFlow backends give up. Each point is a backend, MF-1, MF-4, or the ODE teacher, plotted at its replay-field evaluations per update and normalized by the same method's ODE run, so parity means the distilled teacher matches full integration. On Replogle, both families sit at parity across the whole range, so a single update is short enough that teacher fidelity does not matter. On PBMC, the two separate: trajectory replay stays at parity under both weighting policies, whereas generative replay loses about ten percent at one evaluation and only recovers as the budget grows. The asymmetry is the same as on the image streams, since a replay sample must be integrated to $t=1$ before it can serve as data, while trajectory replay supervises at an intermediate time where a coarse trajectory is still informative.

\begin{figure}[!htbp]
\centering
\includegraphics[width=0.8\textwidth]{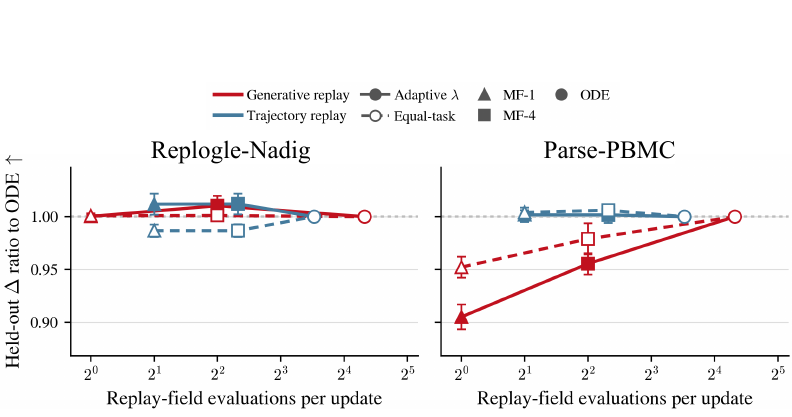}
\caption{\textbf{Replay budget on the transcriptomics streams.} Final held-out Pearson $\Delta$ relative to the matched ODE run as a function of replay-field evaluations per update, for Replogle-Nadig (left) and Parse-PBMC (right), with the MF-1, MF-4, and ODE backends labeled by their evaluation count. Colors distinguish generative from trajectory replay, filled markers with solid lines denote adaptive $\lambda$ and open markers with dashed lines equal-task $\lambda$, and the dotted line marks parity with the matched ODE run. Error bars show variability over seeds $42$, $69$, and $1714$. Higher is better.}
\label{fig: app_fig17}
\end{figure}
\FloatBarrier

\begin{figure}[!htbp]
\centering
\includegraphics[width=0.55\textwidth]{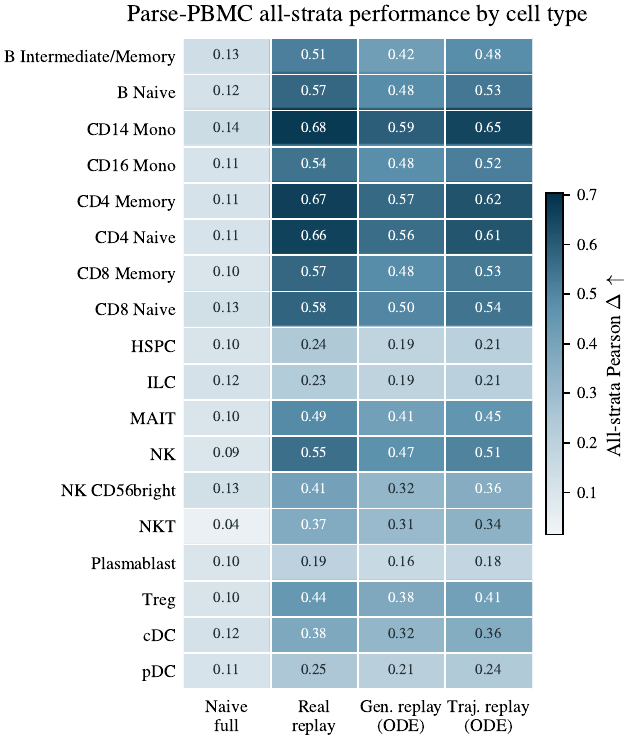}
\caption{\textbf{Parse-PBMC performance by cell type.}
Pearson $\Delta$ for each annotated cell type under naive full fine-tuning, real replay, and generative and trajectory replay with an ODE backend, all using equal-task $\lambda$. Each cell averages over the donor-cytokine strata of that cell type and over seeds 42, 69, and 1714. Higher is better.}
\label{fig: app_fig18}
\end{figure}
\FloatBarrier

\begin{figure}[!htbp]
\centering
\includegraphics[width=1\textwidth]{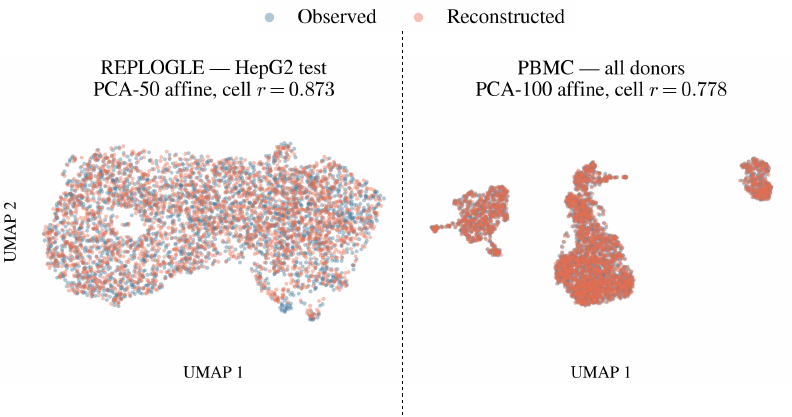}
\caption{\textbf{Decoder reconstruction quality.}
Joint UMAP of observed cells and their reconstructions from the observed PCA latents through the frozen affine decoder, for Replogle-Nadig HepG2 test cells with PCA-50 (left) and Parse-PBMC cells from all donors with PCA-100 (right). Panel titles report the mean per-cell reconstruction correlation. Each panel embeds observed and reconstructed cells together, so no flow-matching generation is involved.}
\label{fig: app_fig19}
\end{figure}
\FloatBarrier

\textbf{Performance by cell type.} \Cref{fig: app_fig18} breaks PBMC down by annotated cell type. The ordering is identical in all eighteen: naive fine-tuning near zero, then generative replay, then trajectory replay, then real replay, with trajectory replay closer to the stored-data reference than to generative replay throughout. Absolute values track cell-type abundance rather than method, from $0.65$ on CD14 monocytes to $0.18$ on plasmablasts, so the residual difficulty is in rare populations for which few cells are available per donor-cytokine stratum rather than in any method-specific failure.

\textbf{Decoder quality.} \Cref{fig: app_fig19} embeds observed cells together with their reconstructions through the frozen affine decoder, with no generation involved. The two clouds overlap throughout on both datasets, so the PCA basis and decoder preserve the population structure that the metrics are computed on, and the decoder is not a confound in the method comparison. It does, however, bound what any method can reach: per-cell correlation is $0.873$ on Replogle and $0.778$ on PBMC, so reported scores sit below a ceiling set by the task-0 PCA basis, which later tasks must also be expressed in.

\clearpage
\subsection{Application: Continual phenotypic modeling in microscopy}\label{app: perturbation_px}
\subsubsection{Experimental details}\label{app: perturbation_px_details}

\textbf{Data.} We use RxRx1 \citep{Sypetkowski_2023_CVPR}, a fluorescence-microscopy screen of gene knockdowns acquired in separate experimental batches, following the dataset choice, preprocessing pipeline, and modeling approach of \citet{zhang2025cellflux} and \citet{jones2026elucidating}. Exact values for every setting below are in \cref{tab:px_hparams}.
\begin{itemize}[leftmargin=*, itemindent=0pt]
  \item \emph{Content.} We restrict to the U2OS cell line as \citet{zhang2025cellflux}, spanning three acquisition experiments and $1{,}071$ perturbation identities, of which $100$ form the official evaluation subset. Since every crop is U2OS, cell-type conditioning carries no information and is disabled.
  \item \emph{Crops.} Each sample is a $96{\times}96$ crop with all six fluorescence channels retained, modeled natively rather than as an RGB projection or a latent code.
  \item \emph{Normalization.} Pixels are mapped to $[-1,1]$. At training time, we add uniform dequantization noise to avoid fitting the quantization lattice, and apply random flips. Neither is applied at evaluation.
  \item \emph{Conditioning.} The model receives perturbation and experiment identity as independent learned embeddings. Experiment identity is an observed acquisition covariate, not a phenotype.
  \item \emph{Channel projection.} Generated and real tensors are mapped to RGB through the same fixed linear projection and re-quantized before any metric, so the comparison never mixes projections.
\end{itemize}

\textbf{Stream.} Tasks arrive in the fixed order U2OS-01 $\rightarrow$ U2OS-02 $\rightarrow$ U2OS-03, so the shift is technical rather than semantic: the perturbation vocabulary is \textit{identical} in every task and only the acquisition batch changes. A deterministic $10\%$ of each experiment's training rows is held out for validation and $\lambda$ selection, and the official test rows are never trained on, replayed from, or used for selection. Task~0 is trained once per seed and shared by every method within each objective family, the flow matching and the diffusion one, so all differences arise from the updates.

\textbf{Model and training.} The velocity field is the ADM U-Net \citep{dhariwal2021diffusion} adapted to $96{\times}96{\times}6$ inputs, trained on the linear noise-to-data interpolant with target $\bm{x}_1-\bm{x}_0$. Every task receives the same budget of $25{,}000$ optimizer steps with Adam and a power-function EMA \citep{karras2024analyzing}, and all reported samples come from the EMA weights.

\textbf{Replay.} Generative and trajectory replay follow the same formulation as in the image experiments (\cref{sec:exp-image-cl}), with the teacher integrated by Euler at step size $0.05$. Equal-task weighting gives $\lambda=1/2$ at task~1 and $2/3$ at task~2, and the adaptive policy selects from $\{0.25,0.5,0.75\}$ by macro validation FID across the experiments seen so far. As in the image streams, we also ask whether a distilled mean velocity can replace the ODE teacher: we distill a MeanFlow model at task~0 for $25{,}000$ steps, matching the per-task budget but at batch size $16$ rather than $64$, and replay with $N\in\{1,4\}$ jumps. Here this is restricted to trajectory replay under \textit{equal-task} weighting, since the MeanFlow residual requires a Jacobian-vector product whose activation footprint scales with the $96{\times}96{\times}6$ state, making co-training substantially more expensive than at $32{\times}32$. The reduced batch is a direct consequence of the same cost.

\textbf{Baselines.} We report four reference points, all on the same stream, per-task budget, backbone, and evaluation protocol as our methods (\cref{tab:px_hparams}).
\begin{itemize}[leftmargin=*, itemindent=0pt]
  \item \emph{Joint training} retrains on the union of all real data seen so far, storing every past sample, and is therefore an upper target rather than a competing method.
  \item \emph{Real replay} mixes stored historical images into each batch. It is likewise privileged but respects the sequential budget, isolating the cost of replaying generated rather than stored samples.
  \item \emph{Naive fine-tuning} continues training on the incoming experiment alone with no retention mechanism, giving the lower target.
  \item \emph{Generative distillation (diffusion)} \citep{masip2025continual} replays samples from the frozen previous-task model but supervises them with the teacher's own noise prediction at a randomly drawn timestep rather than with the injected noise, so the student matches the teacher's denoising behavior along the whole trajectory instead of only reproducing its final samples. We run it on the identical backbone, stream, budget, and EMA as trajectory replay, with a full current batch plus a full generated batch ($32+32$), task-proportional weighting at replay coefficient $0.75$, the post-task online model as teacher, and zero-terminal-SNR DDIM with $20$ teacher and $50$ evaluation steps.
\end{itemize}
Because \citet{masip2025continual} evaluated only natural-image streams, every diffusion entry here is our \textit{own matched run} of their method on RxRx1.

\textbf{Evaluation.} At the final checkpoint, we evaluate each sequential experiment separately and then pool \citep{zhang2025cellflux}. Following prior work, with only a few dozen test images per perturbation, Fréchet distances cannot be estimated conditionally, so we report them marginally over perturbations and rely on composition matching to make the comparison meaningful: we generate exactly one image per test condition, so the generated and reference sets share the same perturbation-by-experiment composition and no perturbation is over-represented. The benchmark comprises $2{,}942$ treated test images, with the $120$ untreated controls excluded. We report Inception FID alongside a Fréchet distance in DINOv2 \citep{oquab2024dinov} feature space (FDD) at a pinned revision, sampling with an adaptive Dopri5 solver. All numbers are means with standard errors over seeds.
\begin{table}[t]
\centering
\definecolor{grouprow}{HTML}{F0F0F0}
\setlength{\tabcolsep}{4.5pt}
\small
\caption{Configuration for the continual microscopy experiments on RxRx1.}
\label{tab:px_hparams}
\resizebox{0.7\textwidth}{!}{%
\begin{tabular}{@{}ll@{}}
\toprule
Setting & Value \\
\midrule
\rowcolor{grouprow} \multicolumn{2}{@{}l}{\textit{Data}} \\
Cell line / experiments & U2OS, $3$ (U2OS-01/02/03) \\
Input & $6\times96\times96$, \texttt{uint8} $\rightarrow[-1,1]$ \\
Train augmentation & uniform dequantization $1/128$, random H/V flips \\
Perturbations & $1{,}071$ total, $100$ evaluated \\
Validation split & $10\%$ of each experiment's train rows \\
\midrule
\rowcolor{grouprow} \multicolumn{2}{@{}l}{\textit{Model}} \\
Backbone & Config-B ADM U-Net, $56.2$M parameters \\
Base channels / blocks & $128$ / $4$ per resolution \\
Channel multipliers & $(2,2,2)$, attention at one resolution \\
Dropout / label dropout & $0.3$ / $0.15$ \\
Conditioning & perturbation $+$ experiment (cell type disabled) \\
\midrule
\rowcolor{grouprow} \multicolumn{2}{@{}l}{\textit{Training}} \\
Steps per task & $25{,}000$ \\
Batch size & $64$ (one H100) \\
Optimizer & Adam, $10^{-4}$ \\
Gradient clip & $0.5$ (flow matching), $1.0$ (diffusion) \\
EMA & power-function, $\sigma_{\mathrm{rel}}=0.01$ \\
Seeds & $42$, $69$, $1714$ \\
\midrule
\rowcolor{grouprow} \multicolumn{2}{@{}l}{\textit{Replay}} \\
Teacher sampler & Euler, step $0.05$ ($20$ evaluations) \\
Equal-task $\lambda$ & $k/(k+1)$, i.e.\ $1/2$ then $2/3$ \\
Adaptive $\lambda$ grid & $\{0.25,0.5,0.75\}$, macro validation FID \\
MeanFlow jumps & $N\in\{1,4\}$, trajectory replay only \\
MeanFlow $(\mu,\sigma)$ & $(2.0,2.0)$, logit-normal \\
MeanFlow diagonal rate & $0.25$ \\
MeanFlow optimizer & AdamW, $10^{-4}$ \\
MeanFlow batch size & $16$ \\
MeanFlow distillation steps & $25{,}000$ (same as task-0 training) \\
\midrule
\rowcolor{grouprow} \multicolumn{2}{@{}l}{\textit{Evaluation}} \\
Generated samples & one per test condition (composition-matched) \\
Benchmark size & $2{,}942$ treated test images ($961/1000/981$) \\
Sampler & Dopri5, EMA weights, eval seed $2027$ \\
Metrics & Inception FID, DINOv2 FDD \\
Pooling & features concatenated across experiments \\
\bottomrule
\end{tabular}}
\end{table}

\FloatBarrier
\subsubsection{Additional results}\label{app:results-microscopy}

\textbf{Final-checkpoint performance.} \Cref{tab:px_cl_final} reports every configuration at the end of the stream. Trajectory replay is the best memory-free method on both pooled metrics, at $16.8$ FID and $96.1$ FDD against $22.0$ and $119.0$ for generative replay, narrowing the gap to the privileged references. The per-experiment columns locate the difference: the families are comparable on the two later experiments, but on U2OS-01 generative replay degrades to $70.9$--$71.4$ while trajectory replay stays at $39.6$--$40.5$. Since the perturbation vocabulary is identical across tasks, this is the oldest acquisition batch being overwritten rather than any loss of semantic coverage. The MeanFlow backends cost a little, $19.4$ FID at one jump against $17.2$ for the matched ODE run, but remain ahead of generative replay under either weighting. Generative distillation (diffusion), the strongest published no-storage baseline on natural-image streams, is the weakest replay method here on every column at $46.4$ FID and $218.2$ FDD.

\begin{table*}[!htbp]
\centering
\definecolor{trajectoryrow}{HTML}{E5EEF3}
\definecolor{referencerow}{HTML}{F0F0F0}
\setlength{\tabcolsep}{4.5pt}
\small
\caption{\textbf{Continual phenomic generation on RxRx1.}
Final-checkpoint performance after the U2OS-01$\rightarrow$U2OS-02$\rightarrow$U2OS-03 experiment stream. Per-experiment FID compares composition-matched generated and test images within each biological experiment and stream metrics pool all three experiments, with FDD denoting Fr\'echet distance in DINOv2 \citep{oquab2024dinov} feature space. Means with standard errors over seeds 42, 69, and 1714. Adaptive $\lambda$ is selected using validation FID and equal-task uses $\lambda_k=k/(k+1)$. Gray marks the privileged references that retain real historical data and blue marks trajectory replay. Bold marks the best memory-free mean in each column. Lower is better.}
\label{tab:px_cl_final}
\resizebox{\textwidth}{!}{%
\begin{tabular}{@{}lccccc@{}}
\toprule
Method & \multicolumn{3}{c}{Per-experiment FID $\downarrow$} & \multicolumn{2}{c}{Pooled stream $\downarrow$} \\
\cmidrule(lr){2-4}\cmidrule(l){5-6}
& U2OS-01 & U2OS-02 & U2OS-03 & FID & DINOv2 FDD \\
\midrule
\rowcolor{referencerow} \textbf{Joint training} (all data; upper bound) & $29.7$ {\scriptsize $(\pm\,0.2)$} & $19.9$ {\scriptsize $(\pm\,0.2)$} & $26.6$ {\scriptsize $(\pm\,0.2)$} & $12.3$ {\scriptsize $(\pm\,0.2)$} & $66.4$ {\scriptsize $(\pm\,1.4)$} \\
\rowcolor{referencerow} Real replay (stored data) & $29.7$ {\scriptsize $(\pm\,0.4)$} & $19.6$ {\scriptsize $(\pm\,0.2)$} & $26.3$ {\scriptsize $(\pm\,0.3)$} & $12.2$ {\scriptsize $(\pm\,0.3)$} & $66.1$ {\scriptsize $(\pm\,1.0)$} \\
\midrule
Naive fine-tuning & $\mathbf{36.0}$ {\scriptsize $(\pm\,0.6)$} & $83.0$ {\scriptsize $(\pm\,1.6)$} & $\mathbf{26.0}$ {\scriptsize $(\pm\,0.1)$} & $24.5$ {\scriptsize $(\pm\,0.3)$} & $110.1$ {\scriptsize $(\pm\,1.3)$} \\
Generative distillation (diffusion) \citep{masip2025continual} & $80.9$ {\scriptsize $(\pm\,6.9)$} & $82.1$ {\scriptsize $(\pm\,2.6)$} & $42.0$ {\scriptsize $(\pm\,0.3)$} & $46.4$ {\scriptsize $(\pm\,3.6)$} & $218.2$ {\scriptsize $(\pm\,13.4)$} \\
\midrule
Generative replay (ODE, adaptive $\lambda$) & $71.4$ {\scriptsize $(\pm\,0.3)$} & $27.1$ {\scriptsize $(\pm\,0.2)$} & $27.0$ {\scriptsize $(\pm\,0.2)$} & $22.0$ {\scriptsize $(\pm\,0.1)$} & $119.0$ {\scriptsize $(\pm\,1.6)$} \\
Generative replay (ODE, equal-task $\lambda$) & $70.9$ {\scriptsize $(\pm\,1.1)$} & $26.8$ {\scriptsize $(\pm\,0.1)$} & $27.7$ {\scriptsize $(\pm\,0.1)$} & $22.2$ {\scriptsize $(\pm\,0.2)$} & $121.4$ {\scriptsize $(\pm\,1.7)$} \\
\midrule
\rowcolor{trajectoryrow} Trajectory replay (ODE, adaptive $\lambda$) & $39.6$ {\scriptsize $(\pm\,0.3)$} & $\mathbf{22.5}$ {\scriptsize $(\pm\,0.6)$} & $28.8$ {\scriptsize $(\pm\,0.1)$} & $\mathbf{16.8}$ {\scriptsize $(\pm\,0.2)$} & $\mathbf{96.1}$ {\scriptsize $(\pm\,1.9)$} \\
\rowcolor{trajectoryrow} Trajectory replay (ODE, equal-task $\lambda$) & $40.5$ {\scriptsize $(\pm\,0.4)$} & $23.1$ {\scriptsize $(\pm\,0.6)$} & $28.6$ {\scriptsize $(\pm\,0.2)$} & $17.2$ {\scriptsize $(\pm\,0.4)$} & $99.5$ {\scriptsize $(\pm\,2.9)$} \\
\rowcolor{trajectoryrow} Trajectory replay (MF-1, equal-task $\lambda$) & $48.9$ {\scriptsize $(\pm\,4.5)$} & $25.4$ {\scriptsize $(\pm\,0.9)$} & $28.8$ {\scriptsize $(\pm\,0.2)$} & $19.4$ {\scriptsize $(\pm\,1.2)$} & $102.2$ {\scriptsize $(\pm\,3.2)$} \\
\rowcolor{trajectoryrow} Trajectory replay (MF-4, equal-task $\lambda$) & $46.4$ {\scriptsize $(\pm\,3.4)$} & $24.7$ {\scriptsize $(\pm\,0.8)$} & $28.9$ {\scriptsize $(\pm\,0.1)$} & $18.9$ {\scriptsize $(\pm\,0.7)$} & $101.7$ {\scriptsize $(\pm\,2.4)$} \\
\bottomrule
\end{tabular}%
}
\end{table*}

\textbf{Stream evolution.} \Cref{fig: app_fig21} tracks the same quantities after each update, and panel a shows the final checkpoint hiding a non-monotonic effect. Naive fine-tuning loses U2OS-01 at the first update, rising to roughly $97$ FID, then falls back to $36.0$, nominally the best memory-free value in that column. This is not retention but a \textit{similarity artifact}: U2OS-01 and U2OS-03 are visually close in illumination and staining intensity, whereas U2OS-02 is the outlier (\cref{fig:px_representative_samples}), so training on the third experiment partially restores the first. The excursion in between is the real measurement, and its size, a threefold FID increase in a single update, is what makes the point about forgetting under a purely technical shift. Both replay families avoid it, with trajectory replay flat across the stream and generative replay drifting upward from the second update. Panel b shows the pooled metric over a growing reference set, where trajectory replay separates at the first update and holds the margin.

\begin{figure}[!htbp]
\centering
\includegraphics[width=1\textwidth]{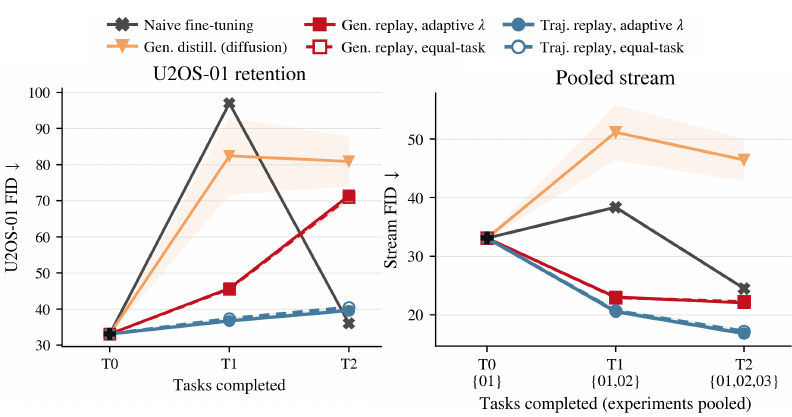}
\caption{\textbf{RxRx1 experiment stream.}
(\textbf{a}) FID on U2OS-01 and (\textbf{b}) pooled stream FID as a function of tasks completed, for naive fine-tuning, diffusion distillation, and generative and trajectory replay under adaptive and equal-task $\lambda$. The pooled reference set in (\textbf{b}) grows with each task, as indicated below the tick labels, whereas (\textbf{a}) keeps U2OS-01 fixed. Lines show the mean and shaded bands the variability over seeds 42, 69, and 1714. Lower is better.}
\label{fig: app_fig21}
\end{figure}
\FloatBarrier

\textbf{Weighting and backend.} \Cref{fig: app_fig22}a separates the two knobs. On the left, the choice of weighting policy barely moves either family, $16.8$ against $17.2$ FID for trajectory replay and $22.0$ against $22.2$ for generative replay, so the gap between them is not an artifact of how $\lambda$ was set. On the right, the selected weights nonetheless differ: trajectory replay tracks the equal-task rule across both updates, while generative replay settles below it and declines, the same pattern as on the image streams and consistent with the criterion discounting replayed samples whose fidelity is low.

Panel b sweeps the backend. The ODE teacher integrates to an intermediate time at step size $0.05$, averaging $11.5$ evaluations per update against the $20$ a full generation would cost, and two MeanFlow evaluations reach $19.4$ FID against its $17.2$, with the FDD ordering the same at a proportionally smaller spread. A near-sixfold reduction in replay cost therefore costs about two FID points, which is the regime in which the mean velocity is worth using.

\textbf{Qualitative samples.} \Cref{fig:px_representative_samples} shows real and trajectory-replay crops for three perturbations across the three experiments. The batch structure is visible directly: U2OS-01 and U2OS-03 share a warm, brightly stained appearance while U2OS-02 is darker and shifted toward green, which is the similarity that lets naive fine-tuning recover U2OS-01 after training on U2OS-03. Generated crops reproduce the staining balance and nuclear morphology of their own experiment rather than defaulting to the appearance of the most recent one.

\FloatBarrier
\begin{figure}[!htbp]
\centering
\includegraphics[width=0.9\textwidth]{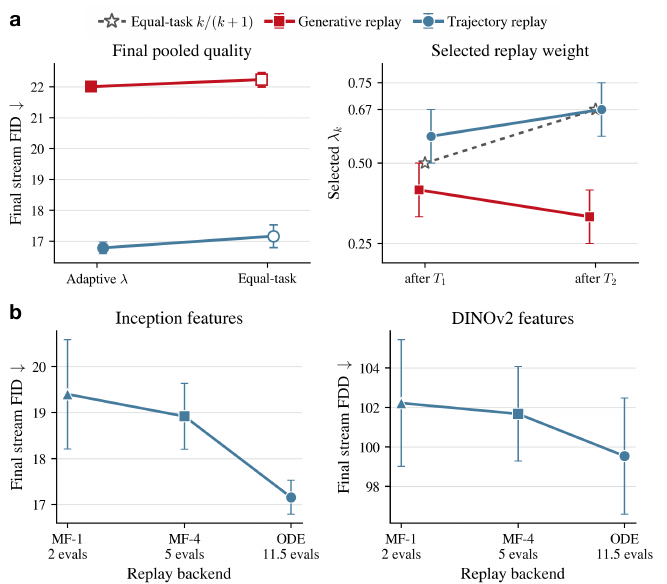}
\caption{\textbf{Replay-weight and backend ablations on RxRx1.}
(\textbf{a}) Final pooled stream FID under adaptive and equal-task $\lambda$ for generative and trajectory replay (left), with filled markers denoting adaptive $\lambda$ and open markers equal-task, and the mean replay weight $\lambda_k$ selected at each update (right), with the equal-task rule $\lambda_k=k/(k+1)$ as a dashed reference.
(\textbf{b}) Final pooled stream FID in Inception features (left) and final pooled stream Fr\'echet DINOv2 distance (FDD) in DINOv2 features (right) for trajectory replay with the MF-1, MF-4, and ODE backends, labeled by their replay-field evaluations per update. Error bars show variability over seeds 42, 69, and 1714. Lower is better for both metrics.}
\label{fig: app_fig22}
\end{figure}

\begin{figure}[!htbp]
\centering
\includegraphics[width=0.6\textwidth]{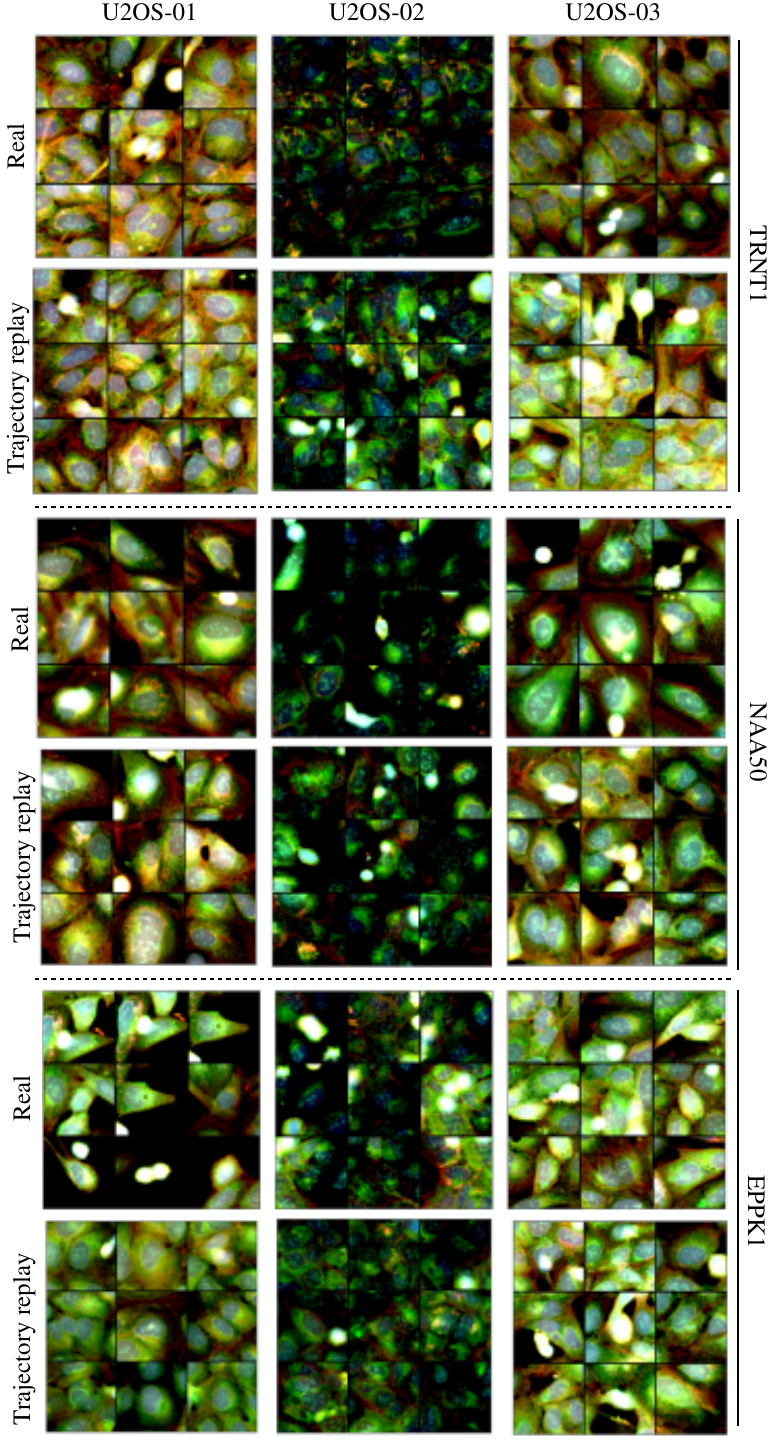}
\caption{\textbf{RxRx1 samples after the final update.}
Real test images and trajectory-replay samples for the perturbations TRNT1, NAA50, and EPPK1, shown for each biological experiment of the stream. Each block contains nine images from the corresponding perturbation and experiment, rendered through the fixed six-channel-to-RGB projection.}
\label{fig:px_representative_samples}
\label{fig: app_fig23}
\end{figure}

\end{document}